\documentclass[12pt,a4paper]{article}
\usepackage{mathrsfs,amsthm,graphicx,xcolor,verbatim,bbm,amsmath,amsfonts,amssymb,newclude,nicefrac,graphicx,enumerate,hyperref,bm,geometry,mathtools,xparse,etoolbox}
\usepackage[capitalise]{cleveref} % clever reference

\crefname{enumi}{}{}
\crefname{equation}{}{}
\crefname{subsection}{Subsection}{Subsections}

\hypersetup{
    colorlinks,
    linkcolor={red!50!black},
    citecolor={green},
    urlcolor={blue!80!black}
}

\theoremstyle{plain}
\newtheorem{theorem}{Theorem}[section]
\newtheorem{lemma}[theorem]{Lemma}
\newtheorem{prop}[theorem]{Proposition}
\newtheorem{cor}[theorem]{Corollary}

\theoremstyle{remark}

\theoremstyle{definition}
\newtheorem{definition}[theorem]{Definition}

\DeclareMathAlphabet{\mathpzc}{OT1}{pzc}{m}{it}

\DeclareFontEncoding{LS1}{}{}
\DeclareFontSubstitution{LS1}{stix}{m}{n}
\DeclareMathAlphabet{\mathscr}{LS1}{stixscr}{m}{n}

\newcommand{\E}{\mathbb{E}}
\renewcommand{\P}{\mathbb{P}}

\newcommand{\R}{\mathbb{R}}
\newcommand{\N}{\mathbb{N}}

\newcommand{\Z}{\mathbb{Z}}
\newcommand{\bbL}{\mathbb{L}}
\newcommand{\bbM}{\mathbb{M}}
\newcommand{\bbA}{\mathbb{A}}

\newcommand{\smallsum}{\textstyle\sum}

\newcommand{\with}{\curvearrowleft}

\newcommand{\cA}{\mathcal{A}}
\newcommand{\cB}{\mathcal{B}}
\newcommand{\cC}{\mathcal{C}}
\newcommand{\cD}{\mathcal{D}}
\newcommand{\cE}{\mathcal{E}}
\newcommand{\cF}{\mathcal{F}}
\newcommand{\cG}{\mathcal{G}}
\newcommand{\cH}{\mathcal{H}}
\newcommand{\cI}{\mathcal{I}}

\newcommand{\cL}{\mathcal{L}}
\newcommand{\cM}{\mathcal{M}}
\newcommand{\cN}{\mathcal{N}}
\newcommand{\cO}{\mathcal{O}}
\newcommand{\cP}{\mathcal{P}}

\newcommand{\cR}{\mathcal{R}}

\newcommand{\cT}{\mathcal{T}}

\newcommand{\cW}{\mathcal{W}}

\newcommand{\cY}{\mathcal{Y}}

\newcommand{\bfa}{\mathbf{a}}

\newcommand{\bfd}{\mathbf{d}}

\newcommand{\bfk}{\mathbf{k}}
\newcommand{\bfl}{\mathbf{l}}

\newcommand{\bfL}{\mathbf{L}}

\newcommand{\bfN}{\mathbf{N}}

\newcommand{\bfP}{\mathbf{P}}

\newcommand{\bfR}{\mathbf{R}}

\newcommand{\scrA}{\mathscr{A}}

\newcommand{\scrD}{\mathscr{D}}

\newcommand{\scrN}{\mathscr{N} \cfadd{def:clippedDNN}}

\newcommand{\fB}{\mathfrak{B}}
\newcommand{\fC}{\mathfrak{C}}

\newcommand{\fL}{\mathfrak{L}}
\newcommand{\fM}{\mathfrak{M}}
\newcommand{\fN}{\mathfrak{N}}

\newcommand{\fR}{\mathfrak{R}}
\newcommand{\fS}{\mathfrak{S}}
\newcommand{\fT}{\mathfrak{T}}

\newcommand{\fW}{\mathfrak{W}}
\newcommand{\fX}{\mathfrak{X}}

\newcommand{\fZ}{\mathfrak{Z}}

\newcommand{\fc}{\mathfrak{c}}
\newcommand{\fd}{\mathfrak{d}}

\newcommand{\fr}{\mathfrak{r} \cfadd{def:rect}}

\newcommand{\fx}{\mathscr{x}}
\newcommand{\fy}{\mathscr{y}}

\renewcommand{\emptyset}{\varnothing}

\newcommand{\normmm}[1]{{\left\vert\kern-0.25ex\left\vert\kern-0.25ex\left\vert #1 
    \right\vert\kern-0.25ex\right\vert\kern-0.25ex\right\vert}} %norm with tripple line
\newcommand{\norm}[1]{\left\lVert #1 \right\rVert}
\newcommand{\abs}[1]{\left\lvert #1 \right\rvert}

\newcommand{\dx}{\mathrm{d}}

\newcommand{\ceil}[1]{ \left\lceil #1 \right\rceil \cfadd{def:ceiling}}
\newcommand{\floor}[1]{ \left\lfloor #1 \right\rfloor \cfadd{def:ceiling}}
\newcommand{\qandq}{\qquad\text{and}\qquad}

\newcommand{\id}{\mathbf{I} \cfadd{def:idmatrix}}
\newcommand{\NN}{\mathbf{N} \cfadd{def:DNN}}

\newcommand{\indicator}[1]{\mathbbm{1}_{\smash{#1}}}

\newcommand{\ReLUidANN}[1]{\mathbb{I}_{#1}}

\ExplSyntaxOn
\seq_new:N \g_cflist_loaded
\seq_new:N \g_cflist_pending

\NewDocumentCommand{\cfadd}{ m }
{
  \seq_if_in:NnF \g_cflist_loaded { #1 } {
    \seq_if_in:NnF \g_cflist_pending { #1 } {
      \seq_gput_right:Nn \g_cflist_pending { #1 }
    }
  }
}

\NewDocumentCommand{\cfconsiderloaded}{ m }{
  \seq_gput_right:Nn \g_cflist_loaded {#1}
}

\NewDocumentCommand{\cfremove}{ m }
{
  \seq_gremove_all:Nn \g_cflist_pending { #1 }
}

\NewDocumentCommand{\cfload}{ o }
{
  \seq_if_empty:NTF \g_cflist_pending {\unskip} {
    (cf.\ \cref{\seq_use:Nn \g_cflist_pending {,}})\IfValueTF{#1}{#1~}{\unskip}
    \seq_gconcat:NNN \g_cflist_loaded \g_cflist_loaded \g_cflist_pending
    \seq_gclear:N \g_cflist_pending
  }
}

\NewDocumentCommand{\cfclear} {} {
  \seq_gclear:N \g_cflist_loaded
  \seq_gclear:N \g_cflist_pending
}

\NewDocumentCommand{\cfout}{ o }
{
  \seq_if_empty:NTF \g_cflist_pending {\unskip} {
    (cf.\ \cref{\seq_use:Nn \g_cflist_pending {,}})\IfValueTF{#1}{#1~}{\unskip}
    \seq_gclear:N \g_cflist_pending
  }
}

\NewDocumentCommand{\ifnocf} { m } {
  \seq_if_empty:NT \g_cflist_pending { #1 }
}

\ExplSyntaxOff

\begin{document}

 \title{Strong overall error analysis for the training of \\artificial neural networks via random initializations}

\author{Arnulf Jentzen$^{1}$ and Adrian Riekert$^{2}$\bigskip\\
\small{$^1$ Faculty of Mathematics and Computer Science, University of M\"unster,}\\
\small{M\"unster, Germany; e-mail: \texttt{ajentzen}\textcircled{\texttt{a}}\texttt{uni-muenster.de}}\\
\small{$^1$ Faculty of Mathematics and Computer Science, University of M\"unster,}\\
\small{M\"unster, Germany; e-mail: \texttt{ariekert}\textcircled{\texttt{a}}\texttt{uni-muenster.de}}
}
\date{\today}
 \maketitle

\begin{abstract}
Although deep learning based approximation algorithms have been applied very successfully to numerous problems, at the moment the reasons for their performance are not entirely understood from a mathematical point of view. Recently, estimates for the convergence of the overall error have been obtained in the situation of deep supervised learning, but with an extremely slow rate of convergence. In this note we partially improve on these estimates. More specifically, we show that the depth of the neural network only needs to increase much slower in order to obtain the same rate of approximation. The results hold in the case of an arbitrary stochastic optimization algorithm with i.i.d.\ random initializations.
\end{abstract}

\tableofcontents

\section{Introduction}

Deep learning algorithms have been applied very successfully to various problems such as image recognition, language processing, mobile advertising, and autonomous driving. However, at the moment the reasons for their performance are not entirely understood. In particular, there is no full mathematical analysis for deep learning algorithms which explains their success.

Roughly speaking, the field of deep learning can be divided into three subfields, deep \emph{supervised learning}, deep \emph{unsupervised learning}, and deep \emph{reinforcement learning}. In the following we will focus on supervised learning, since algorithms in this subfield seem to be most accessible for a rigorous mathematical analysis. Loosely speaking, a typical situation that arises in deep supervised learning is the following (cf., e.g., \cite{CuckerSmale2002}). Let $d, \bfd \in \N$, let $(\Omega, \cF, \P)$ be a probability space, let $\fX \colon \Omega \to [0,1]^d$ be a random variable, and let $\cE \colon [0,1]^d \to [0,1]$ be a continuous function. The goal is then to find a deep neural network (DNN) $\scrN^\theta$ with parameter vector $\theta \in \R^\bfd$ which is a good approximation for $\cE$ in the sense that the expected $L^2$-error $\E [ |\scrN^\theta( \fX) - \cE( \fX)| ^2 ]$ is as small as possible. However, usually the function $\cE$ and the distribution of $\fX$ are unknown and, instead, one only has access to training samples $(X_j, \cE(X_j))$, where the $X_j \colon \Omega \to [0,1]^d$, $j \in \{1,2, \ldots, M\}$, are i.i.d.\ random variables which have the same distribution as $\fX$. Based on these training samples one can compute the empirical risk $\cR(\theta) = \frac{1}{M} \sum_{j=1}^M |\scrN^\theta(X_j) - \cE(X_j)|^2$. A typical approach in supervised learning is then to minimize the risk with respect to $\theta$ by using a stochastic optimization algorithm such as the stochastic gradient descent (SGD) optimization method. The overall error arising from this procedure can be decomposed into three parts
(cf., e.g., \cite{BernerGrohsJentzen2018arXiv} and \cite{CuckerSmale2002}):
approximating the target function $ \cE $
by the considered class of DNNs induces the \emph{approximation error}
(cf., e.g.,~\cite{Barron1993,
Cybenko1989,
Funahashi1989,
HartmanKeelerKowalski1990,
Hornik1991,
HornikStinchcombeWhite1989,
LuShenYangZhang2020,
ShenYangZhang2020}
 and the references mentioned in \cite{BeckJentzenKuckuck2019arXiv}),
replacing the true risk by the empirical risk based on the training samples leads to the \emph{generalization error}
(cf., e.g.,~\cite{BartlettBousquetMendelson2005,
BernerGrohsJentzen2018arXiv,
CuckerSmale2002,
EMaWu2019,
GyorfiKohlerKrzyzakWalk2002,
Massart2007,
VanDeGeer2000}),
and employing the selected optimization method to compute an approximate minimzer introduces the
\emph{optimization error}
(cf., e.g.,~\cite{BachMoulines2013,
BeckJentzenKuckuck2019arXiv,
BercuFort2013,
DuZhaiPoczosSingh2018arXiv,
EMaWu2020online,
ZouCaoZhouGu2019} and the references mentioned therein).
In \cite{JentzenWelti2020arxiv} convergence rates for all three error types have been established in order to obtain a strong overall error estimate. However, the speed of convergence is rather slow. The purpose of this article is to partially improve on the results of \cite{JentzenWelti2020arxiv}. One of the most challenging problems is to quantify the rate of convergence of the optimization error with respect to the number of gradient steps used in the SGD optimization method. While we do not consider this problem here, we derive partially improved upper estimates for the approximation error in comparison to \cite{JentzenWelti2020arxiv}. More specifically, we show that the depth of the neural network $\scrN^\theta$ only needs to increase much slower compared to \cite{JentzenWelti2020arxiv} in order to obtain the same rate of approximation. 
We now state \cref{theo:intro}, which illustrates the contributions of this article in a special case.

\cfclear
\begin{theorem} \label{theo:intro}
Let $d, \bfd, \bfL, M,K, N \in \N$, $c \in [2,\infty)$, $\bfl = (\bfl_0, \bfl_1, \ldots, \bfl_\bfL) \in \N^{\bfL+1}$, $\gamma \in \R$ satisfy 
$ \bfl_0 = d $, 
$ \bfl_\bfL = 1 $,
and
$ \bfd \geq \sum_{i=1}^{\bfL} \bfl_i( \bfl_{ i - 1 } + 1 ) $,
for every
$ m, n \in \N $,
$ s \in \N_0 $,
$ \theta = ( \theta_1, \theta_2, \ldots, \theta_\bfd ) \in \R^\bfd $
with
$ \bfd \geq s + m n + m $
let
$ \scrA_{ m, n }^{ \theta, s }
\colon \R^n \to \R^m $
satisfy for all
$ x = ( x_1, x_2, \ldots, x_n ) \in \R^n $
that
\begin{equation}
\label{eq:affine_linear}
\scrA_{ m, n }^{ \theta, s }( x )
= 
\begin{pmatrix}
\theta_{ s + 1 }
& \theta_{ s + 2 }
& \cdots
& \theta_{ s + n }
\\
\theta_{ s + n + 1 }
& \theta_{ s + n + 2 }
& \cdots
& \theta_{ s + 2 n }
\\
\vdots
& \vdots
& \ddots
& \vdots
\\
\theta_{ s + ( m - 1 ) n + 1 }
& \theta_{ s + ( m - 1 ) n + 2 }
& \cdots
& \theta_{ s + m n }
\end{pmatrix}
\begin{pmatrix}
x_1
\\
x_2
\\
\vdots
\\
x_n
\end{pmatrix}
+
\begin{pmatrix}
\theta_{ s + m n + 1 }
\\
\theta_{ s + m n + 2 }
\\
\vdots 
\\
\theta_{ s + m n + m }
\end{pmatrix},
\end{equation}
let
$ \bfa_i \colon \R^{\bfl _ i} \to \R^{\bfl _i}$, $i \in \{1,2, \ldots, \bfL \}$,
satisfy for all
$ i \in \{1,2, \ldots, \bfL-1 \}$, 
$ x = ( x_1, x_2, \ldots, x_{\bfl_i} ) \in \R^{ \bfl_i } $
that
$ \bfa _i( x ) =
( \max\{ x_1, 0 \}, \max \{x_2, 0\}, 
\ldots,
\max\{ x_{ \bfl_i }, 0 \} ) $, assume for all $x \in \R$ that
$ \bfa_\bfL ( x ) = \max\{ \min\{ x, 1 \}, 0 \} $,
for every $ \theta \in \R^\bfd $
let $ \scrN_\theta \colon \R^d \to \R $
satisfy
\begin{equation} \label{intro:defdnn}
\scrN_\theta =
\bfa_\bfL \circ \scrA_{ \bfl_\bfL, \bfl_{ \bfL - 1 } }^{ \theta, \smash{ \sum_{i = 1}^{\bfL-1} \bfl_i ( \bfl_{ i - 1 } + 1 ) } }
\circ \bfa_{\bfL-1} \circ \smash{ \scrA_{ \bfl_{ \bfL - 1 },  \bfl_{ \bfL - 2 } }^{ \theta, \smash{ \sum_{i = 1}^{\bfL-2} \bfl_i ( \bfl_{ i - 1 } + 1 ) } } }
\circ \cdots 
\circ \bfa_1 \circ \scrA_{  \bfl_1, \bfl_0 }^{ \theta, 0 } ,    
\end{equation}
 let $(\Omega, \cF, \P)$ be a probability space, let $X_j \colon \Omega \to [0,1]^d $, $j \in \{1, 2, \ldots, M\}$, be i.i.d.\ random variables, let $\cE \colon [0,1]^d \to [0,1]$ satisfy for all $x,y \in [0,1]^d$ that $|\cE (x)- \cE (y)| \leq c \norm{x-y}_1$, let $\Theta_{k, n} \colon \Omega \to \R^\bfd$, $k, n \in \N_0$, and $\bfk \colon \Omega \to (\N _0)^2$ be random variables, assume that $\Theta_{k, 0}$, $k \in \{1,2, \ldots, K\}$, are i.i.d., assume that $\Theta_{1, 0}$ is continuous uniformly distributed on $[-c,c]^\bfd$,  let $\cR \colon \R^\bfd \times \Omega \to [0, \infty)$ satisfy for all $\theta \in \R^\bfd$, $\omega \in \Omega$ that \cfadd{def:p-norm}
 \begin{equation} \label{eq:intro:emprisk}
    \cR(\theta, \omega) = \frac{1}{M} \left[ \sum_{j=1}^{M}  \bigl| \scrN_{\theta} (X_j (\omega))-\cE(X_j (\omega)) \bigr| ^2 \right],
\end{equation}
let $\cG  \colon \R^\bfd \times \Omega \to \R^\bfd$ satisfy for all $\omega \in \Omega$,
$\theta \in \{ \vartheta \in \R^\bfd \colon (\cR(\cdot, \omega) \colon \R^\bfd \to [0, \infty) \text{ is differentiable } \allowbreak\text{at } \vartheta )\}$
that $\cG(\theta, \omega) = (\nabla_\theta \cR)(\theta, \omega)$, assume for all $ k,n \in \N$ that $\Theta_{k,n} = \Theta_{k, n-1} - \gamma \cG(\Theta_{k, n-1})$, and assume for all $\omega \in \Omega$ that
\begin{equation} \label{eq:intro: komega}
    \bfk( \omega) \in \arg \min\nolimits_{(k,n) \in \{1,2, \ldots, K\} \times \{0, 1, \ldots, N\},\, \| \Theta_{k,n} (\omega) \|_\infty \leq c} \cR (\Theta_{k,n}(\omega), \omega).
\end{equation}
Then
\begin{equation} \label{eq:theo:intro}
\begin{split}
          &\E \! \left[    \int_{[0,1]^d} \bigl| \scrN_{\Theta_\bfk} (x)-\cE (x) \bigr| \, \P_{X_1}(\dx x) \right] \\
          &\leq \frac{6 d c}{\left( \min\{2^{{\bfL} }, \bfl_1, \ldots, \bfl_{\bfL-1} \} \right)^{\nicefrac{1}{d}}} 
        +  \frac{ \bfL ( \norm{\bfl}_\infty+1)^\bfL c^{\bfL+1} }
            {K^{[(2\bfL)^{-1}(\norm{\bfl}_\infty+1)^{-2}]}} 
         + \frac{4 c \bfL ( \norm{\bfl}_\infty +1) \ln (e M)} { M^{\nicefrac{1}{4}}}.
\end{split}
\end{equation}
\end{theorem}
\cref{theo:intro} above is a direct consequence of \cref{cor:sgdsimple} (applied with $\fN \with \{0,1, \ldots, N \}$, $(X^{k,n}_j)_{j \in \{1, \ldots, M \}, \, (k,n) \in (\N_0)^2 } \with (X_j)_{j \in \{1, \ldots, M \}}$, $(Y_j^{k,n})_{j \in \{1, \ldots, M \}, \, (k,n ) \in (\N_0) ^2 } \with (\cE(X_j))_{j \in \{1, \ldots, M \}}$, $ (J_n)_{n \in \N} \allowbreak \with (M)_{n \in \N}$, $(\gamma_n)_{n \in \N} \with (\gamma)_{n \in \N}$ in the notation of \cref{cor:sgdsimple}). \cref{cor:sgdsimple} follows from \cref{cor:simple} below, which, in turn, is a consequence of the more general result in \cref{theo:main}, one of the main results of this article. 

In the following we provide additional explanations regarding the mathematical objects in \cref{theo:intro}. The vector $\bfl \in \N^{\bfL + 1}$ determines the architecture of an artificial neural network with input dimension $\bfl_0=d$, output dimension $\bfl_\bfL =1$, and $\bfL-1$ hidden layers of dimensions $\bfl_1, \bfl_2, \ldots, \bfl_{\bfL-1}$, respectively. For every $\theta \in \R^\bfd$ the function $\scrN _{\theta} \colon \R^d \to \R $ (cf.\ \eqref{intro:defdnn}) is the realization of an artificial neural network with parameters (weights and biases) given by the vector $\theta$ where the multidimensional rectifier functions $\bfa_i $, $i \in \{1,2, \ldots, \bfL-1 \}$, are the employed activation functions in front of the hidden layers and where the clipping function $\bfa_\bfL$ is the employed activation function in front of the output layer. We intend to approximate the unknown target function $\cE \colon [0,1]^d \to [0,1]$, which is assumed to be Lipschitz continuous with Lipschitz constant $c$. The training samples $(X_j,  \cE (X_j))$, $j \in \{1,2, \ldots, M\}$, are used to compute the empirical risk $\cR$ in \eqref{eq:intro:emprisk} and the function $\cG$ is defined as the gradient of the empirical risk with respect to the parameter vector $\theta\in \R ^\bfd$. The function $\cG$ is needed in order to compute the random parameter vectors $\Theta_{k,n} \colon \Omega \to \R^\bfd$, $k \in \N$, $n \in \N_0$, via the stochastic gradient descent algorithm with constant learning rate $\gamma$. Note that the index $n \in \N_0$ indicates the current gradient step and the index $k \in \N$ counts the number of random initializations. The starting values $\Theta_{k,0}$, $k \in \{1,2, \ldots, K\}$, are assumed to be independent and uniformly distributed on the hypercube $[-c,c]^\bfd$. After the SGD procedure has been started $K$ times, performing $N$ gradient steps in each case, the random double index $\bfk(\omega) \in (\N_0)^2$ represents the final choice of the parameter vector and is selected as follows. We consider those pairs of indices $(k, n) \in \{1,2, \ldots, K\} \times \{0, 1, \ldots, N\}$ which satisfy that the vector $\Theta_{k,n}$ is inside the hypercube $[-c,c]^\bfd$ (cf.\ \eqref{eq:intro: komega}). Among these parameter vectors, $\Theta_\bfk$ is the one which minimizes the empirical risk $\cR(\Theta_{k,n})$. 

The conclusion of \cref{theo:intro}, inequality \eqref{eq:theo:intro}, provides an upper estimate for the expected $L^1$-distance between the target function $\cE$ and the selected neural network $\scrN_{\Theta_\bfk}$ with respect to the distribution of the input data $X_1$. The right-hand side of \eqref{eq:theo:intro} consists of three terms: The first summand is an upper estimate for the approximation error and converges to zero as the number of hidden layers (the depth of the DNN) and their dimensions (the width of the DNN) increase to infinity. The second term corresponds to the optimization error and converges to zero as the number of random initializations $K$ increases to infinity. Finally, the third term provides an upper bound for the generalization error and goes to zero as the number of training samples $M$ increases to infinity. Observe that the right-hand side of \eqref{eq:theo:intro} does not depend on the number of gradient steps $N$. In other words, if the best approximation is chosen from the random initializations $\Theta_{k,0}$, $k \in \{1,2, \ldots, K\}$, without performing any gradient steps, the rate of convergence is the same, as it depends only on the number of random initializations $K$.
Comparing the statement of \cref{theo:intro} above to \cite[Theorem 1.1]{JentzenWelti2020arxiv}, the main improvement is that the term $\bfL$ in the denominator of the first summand has been replaced by $2^{\bfL}$, and thus we obtain exponential convergence with respect to the number of hidden layers $\bfL$. We derive this improved convergence rate by employing a well-known neural network representation for the maximum of $\mathscr{d}$ numbers, $\mathscr{d} \in \N$, which uses only $\cO(\log \mathscr{d})$ instead of $\cO( \mathscr{d} )$ layers (cf.\ \cref{def:max_d} and \cref{Prop:max_d} below). 

In one of the main results of this article, \cref{theo:main} below, we consider more generally the $L^p$-norm of the overall $L^2$-error instead of the expectation of the $L^1$-error, we do not restrict the training samples to unit hypercubes, and we allow the random variables $\Theta_{k,n} \colon \Omega \to \R^\bfd$, $k, n \in \N$, to be computed via an arbitrary stochastic optimization algorithm. Another main result of this article is \cref{theo:1d} below, which provides an improved estimate compared to \cref{eq:theo:intro} in the special case of one-dimensional input data. 

The remainder of this article is organized as follows. In \cref{sec:dnns} we recall two approaches how DNNs can be described in a mathematical way. Afterwards, in \cref{sec:dnnrep} we present three elementary DNN representations for certain functions which will be needed for the error analysis. In \cref{sec:approx} we employ the neural network representations from \cref{sec:dnnrep} to establish upper bounds for the approximation error. Thereafter, in \cref{sec:general} we analyze the generalization error by using elementary Monte Carlo estimates. Finally, in \cref{sec:results} we combine the estimates for the approximation error from \cref{sec:approx}, the estimates for the generalization error from \cref{sec:general}, and the known estimates for the optimization error from \cite{JentzenWelti2020arxiv} in order to obtain strong estimates for the overall error.

\section{Basics on deep neural networks (DNNs)} \label{sec:dnns}
In this section we review two ways of describing DNNs in a mathematical fashion, both of which will be used for the error analyses in the later sections. More specifically, we present in \cref{subsec:dnnvect} a vectorized description and in \cref{subsec:dnn} a structured description of DNNs. In \cref{cor:dnnvect} below we recall the equivalence between the two approaches. Afterwards, in \cref{subsec:composition}--\cref{subsec:dnnconc} we define several elementary DNN operations.

The content of this section is well-known in  the scientific literature; cf., e.g., Beck, Jentzen, \& Kuckuck \cite[Section 2]{BeckJentzenKuckuck2019arXiv}, Grohs, Jentzen, \& Salimova \cite[Section 3]{GrohsJentzenSalimova2019arXiv}, and Grohs et al.\ \cite[Section 2]{GrohsHornungJentzenZimmermann2019arXiv}.
In particular, Definitions \ref{def:affine}, \ref{def:DNNbasic}, \ref{def:rect}, \ref{def:clip}, \ref{def:multiversion}, \ref{def:multirect}, \ref{def:multiclip}, \ref{def:clippedDNN} are
\cite[Definitions 2.1, 2.2, 2.4, 2.6, 2.3, 2.5, 2.7, 2.8] {BeckJentzenKuckuck2019arXiv},
\cref{def:DNN} is an extended version of
\cite[Definition 2.9]{BeckJentzenKuckuck2019arXiv},
\cref{def:realization} is
\cite[Definition 2.10]{BeckJentzenKuckuck2019arXiv},
\cref{def:DNNparam} is very similar to 
\cite[Definition 2.11] {BeckJentzenKuckuck2019arXiv}, 
\cref{cor:dnnvect} is a slight generalization of 
\cite[Corollary 2.15]{BeckJentzenKuckuck2019arXiv},
\cref{def:composition} is 
\cite[Definition 2.19] {BeckJentzenKuckuck2019arXiv},
\cref{prop:comp} is a reformulation of 
\cite[Proposition 2.6]{GrohsHornungJentzenZimmermann2019arXiv},
\cref{prop:compasso} is
\cite[Lemma 2.8]{GrohsHornungJentzenZimmermann2019arXiv},
\cref{def:parallelization} is 
\cite[Definition 2.16] {BeckJentzenKuckuck2019arXiv},
\cref{prop:parallelization} is a combination of 
\cite[Propositions 2.19 and 2.20] {GrohsHornungJentzenZimmermann2019arXiv},
\cref{def:DNN:aff} is based on \cite[Definitions 3.7 and 3.10]{GrohsJentzenSalimova2019arXiv} (cf.\ \cite[Definition 2.17]{BeckJentzenKuckuck2019arXiv}),
\cref{def:sum} is \cite[Definition 3.17]{GrohsJentzenSalimova2019arXiv},
\cref{prop:sum} is \cite[Lemma 3.19]{GrohsJentzenSalimova2019arXiv},
\cref{def:concat} is \cite[Definition 3.22]{GrohsJentzenSalimova2019arXiv},
and \cref{prop:concat} is \cite[Lemma 3.25]{GrohsJentzenSalimova2019arXiv}.
%\cref{def:ReLU_identity} is based on \cite[Definition 2.18]{BeckJentzenKuckuck2019arXiv}.

\subsection{Vectorized description of DNNs} \label{subsec:dnnvect}

\begin{definition} [Affine functions] \label{def:affine}
Let $\bfd,m,n \in \N$, $s \in \N_0$, $\theta = (\theta_1, \theta_2, \ldots, \theta_\bfd) \in \R^\bfd$ satisfy $\bfd \geq s + mn + m$. Then we denote by $\scrA_{m,n}^{\theta, s} \colon \R^n \to \R^m$ the function which satisfies for all $x = (x_1, x_2, \ldots, x_n) \in \R^n$ that
\begin{equation}
    \scrA_{m,n}^{\theta, s}(x) = 
    \begin{pmatrix} \theta_{s+1} & \theta_{s+2} & \cdots & \theta_{s+n} \\
    \theta_{s+n+1} & \theta_{s+n+2} & \cdots & \theta_{s+2n} \\
    \theta_{s+2n+1} & \theta_{s+2n+2} &\cdots & \theta_{s+3n} \\
    \vdots & \vdots & \ddots & \vdots \\
    \theta_{s+(m-1)n+1} & \theta_{s+ (m-1)n+2} & \cdots & \theta_{s+m n}
    \end{pmatrix}
    \begin{pmatrix}
    x_1 \\
    x_2 \\
    x_3 \\
    \vdots \\
    x_n
    \end{pmatrix} +
    \begin{pmatrix}
    \theta_{s + m n + 1} \\
    \theta_{s + m n + 2} \\
    \theta_{s + m n + 3} \\
    \vdots \\
    \theta_{s + m n + m}
    \end{pmatrix}.
\end{equation}
\end{definition}
\cfclear
\begin{definition} [Fully connected feedforward artificial neural networks] \label{def:DNNbasic}
Let $\bfd, \bfL \in \N$, $\bfl = (\bfl_0, \bfl_1, \ldots, \bfl_\bfL) \in \N^{\bfL+1}$, $s \in \N_0$, $\theta \in \R^\bfd$ satisfy $\bfd \geq s + \sum_{k=1}^\bfL \bfl_k(\bfl_{k-1}+1)$,
and let $\bfa _k \colon \R^{\bfl_k} \to \R^{\bfl_k}$, $k \in \{1,2, \ldots, \bfL\}$, be functions. Then we denote by $\cN_{\bfa_1, \bfa_2, \ldots, \bfa_\bfL}^{\theta, s, \bfl_0} \colon \R^{\bfl_0} \to \R^{\bfl_\bfL}$ the function given by \cfadd{def:affine}
\begin{equation}
    \cN_{\bfa_1, \bfa_2, \ldots, \bfa_\bfL}^{\theta, s, \bfl_0} 
    = \bfa_{\bfL} \circ \scrA_{\bfl_\bfL, \bfl_{\bfL-1}}^{\theta, s+\sum_{k=1}^{\bfL-1}\bfl_k(\bfl_{k-1}+1)} 
    \circ \bfa_{\bfL-1} \circ \scrA_{\bfl_{\bfL-1}, \bfl_{\bfL-2}}^{\theta, s+\sum_{k=1}^{\bfL-2}\bfl_k(\bfl_{k-1}+1)} \circ \cdots 
    \circ \bfa_1 \circ \scrA_{\bfl_1, \bfl_0}^{\theta, s}    
\end{equation}
\cfload.
\end{definition}

\begin{definition}[Rectifier function] \label{def:rect}
We denote by $\fr \colon \R \to \R$ the function which satisfies for all $x \in \R$ that
\begin{equation}
    \fr (x) = \max \{x,0\}.
\end{equation}
\end{definition}

\begin{definition}[Clipping functions] \label{def:clip}
Let $u \in [-\infty, \infty)$, $v \in (u, \infty]$. Then we denote by $\fc_{u,v} \colon \R \to \R$ the function which satisfies for all $x \in \R$ that
\begin{equation}
    \fc_{u,v}(x) = \max \{ \min\{x,v\}, u \}. 
\end{equation}
\end{definition}

\begin{definition}[Multidimensional versions] \label{def:multiversion}
Let $d \in \N$ and let $a \colon \R \to \R$ be a function. Then we denote by $\fM_{a, d} \colon \R^d \to \R^d$ the function which satisfies for all $x=(x_1, x_2, \ldots, x_d) \in \R^d$ that
\begin{equation}
    \fM_{ a , d} (x)=(a(x_1), a(x_2), \ldots, a(x_d)).
\end{equation}
\end{definition}
\cfclear
\begin{definition}[Multidimensional rectifier functions] \label{def:multirect}
Let $d \in \N$. Then we denote by $\fR_d \colon \R^d \to \R^d$ the function given by \cfadd{def:multiversion} \cfadd{def:rect}
\begin{equation}
    \fR_d = \fM_{ \fr, d} 
\end{equation}
\cfload.
\end{definition}

\cfclear
\begin{definition}[Multidimensional clipping functions] \label{def:multiclip}
Let $u \in [-\infty, \infty)$, $v \in (u, \infty]$, $d \in \N$. Then we denote by $\fC_{u,v,d} \colon \R^d \to \R^d$ the function given by \cfadd{def:multiversion} \cfadd{def:clip}
\begin{equation}
    \fC_{u,v, d} = \fM_{\fc_{u,v}, d}
\end{equation}
\cfload.
\end{definition}

\cfclear
\begin{definition}[Rectified clipped DNNs] \label{def:clippedDNN}
\cfconsiderloaded{def:clippedDNN}
Let $\bfL, \bfd \in \N$, $\bfl = (\bfl_0, \bfl_1, \ldots, \bfl_\bfL) \in \N^{\bfL+1}$, $u \in [-\infty, \infty)$, $v \in (u, \infty]$, $\theta \in \R^\bfd$ satisfy
$\bfd \geq \sum_{k=1}^\bfL \bfl_k(\bfl_{k-1}+1)$.
Then we denote by $\scrN_{u,v}^{\theta, \bfl} \colon \R^{\bfl_0} \to \R^{\bfl_\bfL}$ the function given by
\begin{equation} \cfadd{def:DNNbasic} \cfadd{def:multiclip} \cfadd{def:multirect} 
    \scrN_{u,v}^{\theta, \bfl} = \begin{cases}
    \cN_{\fC_{u,v, \bfl_1}}^{\theta, 0, \bfl_0} & \colon  \bfL = 1 \\
    \cN_{\fR_{\bfl_1}, \fR_{\bfl_2}, \ldots, \fR_{\bfl_{\bfL-1}}, \fC_{u,v, \bfl_\bfL}}^{\theta, 0, \bfl_0} & \colon  \bfL > 1
    \end{cases}
\end{equation}
\cfload.
\end{definition}

\subsection{Structured description of DNNs} \label{subsec:dnn}
\begin{definition} [Deep neural networks] \label{def:DNN}
We denote by $\bfN$ the set given by
\begin{equation}
    \bfN = \bigcup_{\bfL \in \N} \bigcup_{\bfl_0, \bfl_1, \ldots, \bfl_\bfL \in \N} \left( \bigtimes_{k=1}^{\bfL} \left(\R^{\bfl_k \times \bfl_{k-1}} \times \R^{\bfl_k}\right) \! \right),
\end{equation}
we denote by $\cL, \cI, \cO, \cP \colon \bfN \to \N$, $\cH \colon \NN \to \N_0$, $\cA \colon \NN \to \bigcup_{n=2}^\infty \N^n$, and $\cD_i \colon \bfN \to \N_0$, $i \in \N_0$, the functions which satisfy for all $\bfL \in \N$, $\bfl_0, \bfl_1, \ldots, \bfl_\bfL \in \N$, $i \in \{1,2, \ldots, \bfL\}$, $j \in \{\bfL+1, \bfL+2, \dots\}$, $\Phi \in \left( \bigtimes_{k=1}^{\bfL} \left(\R^{\bfl_k \times \bfl_{k-1}} \times \R^{\bfl_k}\right) \! \right)$ that $\cL(\Phi) = \bfL$, $ \cD_0 ( \Phi) = \cI(\Phi) = \bfl_0$, $\cO(\Phi) = \bfl_\bfL$, $\cH(\Phi) = \bfL - 1$, $\cP(\Phi) = \sum_{i=1}^\bfL \bfl_i( \bfl_{i-1}+1)$, $\cA(\Phi) = (\bfl_0, \bfl_1, \ldots, \bfl_\bfL)$, $\cD_i(\Phi) = \bfl_i$, and $\cD_j(\Phi) = 0$, and we denote by
$\cW_i \colon \NN \to \bigcup_{m,n \in \N} \R^{m \times n}$,  $i \in \N$,
and $\cB_i \colon \NN \to \bigcup_{n \in \N} \R^n$, $i \in \N$, the functions which satisfy for all
$\bfL \in \N$, $\bfl_0, \bfl_1, \ldots, \bfl_\bfL \in \N$, $i \in \{1,2, \ldots, \bfL\}$, $j \in \{\bfL+1, \bfL+2, \dots\}$, and $\Phi = ((W_1,B_1), (W_2,B_2), \ldots,  (W_\bfL, B_\bfL)) \in \left( \bigtimes_{k=1}^{\bfL} \left(\R^{\bfl_k \times \bfl_{k-1}} \times \R^{\bfl_k}\right) \! \right)$
that $\cW_i(\Phi) = W_i$, $\cW_j(\Phi) = 0 \in \R$, 
$\cB_i(\Phi) = B_i$, and $\cB_j(\Phi) = 0 \in \R$. 
We say that $\Phi$ is a neural network if and only if $ \Phi \in \bfN$.
\end{definition}

\cfclear
\begin{definition}[Realizations of DNNs] \label{def:realization}
Let $a \in C(\R, \R)$. Then we denote by $\cR_a \colon \NN \to \bigcup_{m,n \in \N} C(\R^m, \R^n)$ the function which satisfies for all $\bfL \in \N$, $\bfl_0, \bfl_1, \ldots, \bfl_\bfL \in \N$, $\Phi = \bigl( (W_1, B_1), \allowbreak (W_2, B_2), \ldots, (W_\bfL, B_\bfL) \bigr) \in \left( \bigtimes_{k=1}^{\bfL}  \left(\R^{\bfl_k \times \bfl_{k-1}} \times \R^{\bfl_k} \right) \! \right)$, and all $x_k \in \R^{\bfl_k}$, $k \in \{0,1, \allowbreak \ldots,  \bfL-1 \}$, with $\forall\, k \in \{1,2, \ldots, \bfL-1 \} \colon x_k = \fM_{a, \bfl_k}(W_k x_{k-1}+B_k)$ that \cfadd{def:DNN} \cfadd{def:multiversion}
\begin{equation}
    \cR_a(\Phi) \in C( \R^{\bfl_0}, \R^{\bfl_\bfL}) \qandq (\cR_a(\Phi))(x_0) = W_\bfL x_{\bfL-1}+B_\bfL
\end{equation}
\cfload.
\end{definition}

\cfclear
\begin{definition} [Parameters of DNNs] \label{def:DNNparam}
We denote by $\cT \colon \NN \to \bigcup_{n=2}^\infty \R^n$ the function which satisfies for all $\bfL \in \N$, $\bfl_0, \bfl_1, \ldots, \bfl_\bfL \in \N$, $W_k = (W_{k,i,j})_{(i, j ) \in \{1, \ldots, \bfl_{k}\} \times \{1, \ldots, \bfl_{k-1} \} } \in \R^{\bfl_l \times \bfl_{k-1}}$, $k \in \{1,2, \ldots, \bfL \}$, and $B_k=(B_{k,i})_{i \in \{1, \ldots, \bfl_k\}} \in \R^{\bfl_k}$, $k \in \{1,2, \ldots, \bfL \}$, that \cfadd{def:DNN}
\begin{equation}
    \begin{split}
        \cT \bigl( (&(W_1,B_1),(W_2,B_2), \ldots, (W_\bfL, B_\bfL)) \bigr) \\
        = \bigl( &W_{1,1,1}, W_{1,1,2}, \ldots, W_{1,1, \bfl_0}, W_{1,2,1}, W_{1,2,2}, \ldots, W_{1,2, \bfl_0}, \ldots, W_{1, \bfl_1, \bfl_0}, B_{1,1}, \ldots, B_{1, \bfl_1}, \\
        & W_{2,1,1}, W_{2,1,2}, \ldots, W_{2,1, \bfl_0}, W_{2,2,1}, W_{2,2,2}, \ldots, W_{2,2, \bfl_0}, \ldots, W_{2, \bfl_1, \bfl_0}, B_{2,1}, \ldots, B_{2, \bfl_1}, \\
        &\ldots, W_{\bfL,1,1}, W_{\bfL,1,2}, \ldots, W_{\bfL,1, \bfl_0}, W_{\bfL,2,1}, \ldots, W_{\bfL,2, \bfl_0}, \ldots, W_{\bfL, \bfl_1, \bfl_0}, B_{\bfL,1}, \ldots, B_{\bfL, \bfl_1} \bigr)
    \end{split}
\end{equation}
\cfload.
\end{definition}

\cfclear
\begin{prop} \label{cor:dnnvect}
Let $u \in [-\infty, \infty)$, $v \in (u, \infty]$, $\Phi \in \NN$ \cfload. Then it holds for all $x \in \R^{\cI(\Phi)}$ that \cfadd{def:clippedDNN} \cfadd{def:multiclip} \cfadd{def:DNN} \cfadd{def:realization} \cfadd{def:DNNparam} \cfadd{def:rect}
\begin{equation}
    (\scrN_{u,v}^{\cT(\Phi), \cA (\Phi)})(x) = \fC_{u,v, \cO(\Phi)}((\cR_\fr(\Phi))(x))
\end{equation}
\cfout.
\end{prop}
\begin{proof} [Proof of \cref{cor:dnnvect}]
This is a direct consequence of \cite[Corollary 2.15]{BeckJentzenKuckuck2019arXiv}.
\end{proof}

\subsection{Compositions of DNNs} \label{subsec:composition}
\cfclear
\begin{definition}[Compositions] \label{def:composition}
Let $\Phi_1, \Phi_2 \in \NN$ satisfy $\Phi_1 = ((W_1, B_1), (W_2, B_2), \ldots, \allowbreak (W_\bfL, B_\bfL))$, $\Phi_2 = ((\fW_1, \fB_1), (\fW_2, \fB_2), \ldots, (\fW_\fL, \fB_\fL))$, and $\cI(\Phi_1) = \cO(\Phi_2)$ \cfload. Then we denote by $\Phi_1 \bullet \Phi_2 \in \NN$ the neural network given by \cfadd{def:DNN}
\begin{equation}
    \Phi_1 \bullet \Phi_2 
    = \begin{cases}
    \bigl(W_1 \fW_1, W_1 \fB_1 + B_1 \bigr) & \colon \bfL = 1 = \fL \\
    \bigl( (\fW_1, \fB_1), \ldots, (\fW_{\fL-1}, \fB_{\fL-1}), (W_1 \fW_\fL, W_1 \fB_\fL + B_1) \bigr) & \colon \bfL=1<\fL \\
    \bigl( (W_1 \fW_1, W_1 \fB_1 + B_1), (W_2,B_2), \ldots, (W_\bfL, B_\bfL) \bigr) & \colon \bfL > 1 = \fL \\
    \bigl( (\fW_1, \fB_1), \ldots, (\fW_{\fL-1}, \fB_{\fL-1}), (W_1 \fW_\fL, W_1 \fB_\fL + B_1), \\ (W_2,B_2), \ldots, (W_\bfL, B_\bfL) \bigr) & \colon \bfL > 1 < \fL.
    \end{cases}
\end{equation}
\end{definition}
\cfclear
\begin{prop} \label{prop:comp}
Let $\Phi_1, \Phi_2 \in \NN$ satisfy $\cI(\Phi_1) = \cO(\Phi_2)$ \cfload. Then \cfadd{def:composition} \cfadd{def:realization}
\begin{enumerate} [(i)]
    \item it holds that $\cL(\Phi_1 \bullet \Phi_2) = \cL(\Phi_1) + \cL(\Phi_2) - 1$,
    \item it holds that $\cH(\Phi_1 \bullet \Phi_2) = \cH(\Phi_1) + \cH(\Phi_2)$,
    \item it holds for all $i \in \N_0$ that
    \begin{equation}
        \cD_i(\Phi_1 \bullet \Phi_2) = \begin{cases}
        \cD_i(\Phi_2) & \colon 0 \leq i \leq \cL(\Phi_2)-1 \\
        \cD_{i- \cL(\Phi_2)+1}(\Phi_1)  & \colon \cL(\Phi_2) \leq i \leq \cL(\Phi_1) + \cL(\Phi_2)-1 \\
        0 & \colon \cL(\Phi_1) + \cL(\Phi_2) \leq i,
        \end{cases}
    \end{equation}
    and
    \item it holds for all $a \in C(\R, \R) $ that $\cR_a(\Phi_1 \bullet \Phi_2) = [\cR_a(\Phi_1)] \circ [\cR_a(\Phi_2)]$
\end{enumerate}
\cfout.
\end{prop}

\cfclear 
\begin{lemma} \label{prop:compasso}
Let $\Phi_1, \Phi_2, \Phi_3 \in \NN$ satisfy $\cI(\Phi_1) = \cO(\Phi_2)$ and $\cI(\Phi_2) = \cO(\Phi_3)$ \cfload. Then $(\Phi_1 \bullet \Phi_2) \bullet \Phi_3 = \Phi_1 \bullet (\Phi_2 \bullet \Phi_3)$ \cfadd{def:composition}\cfout.
\end{lemma}

\subsection{Parallelizations of DNNs} \label{subsec:parallelization}
\cfclear
\begin{definition}[Parallelizations] \label{def:parallelization}
Let $n, \bfL \in \N$, $\Phi_1, \Phi_2, \ldots, \Phi_n \in \NN$ satisfy for all $i \in \{1,2, \ldots, n \}$ that $\cL(\Phi_i) = \bfL$ \cfload. Then we denote by $\bfP_n(\Phi_1, \Phi_2, \ldots, \Phi_n) \in \NN$ the neural network given by 
\begin{equation}
    \begin{split}
        \bfP_n(\Phi_1, \Phi_2, \ldots, \Phi_n) 
        =& \left( \! \left( \! \begin{pmatrix}
        \cW_1(\Phi_1) & 0 & \cdots & 0 \\
        0 & \cW_1(\Phi_2) & \cdots & 0 \\
        \vdots & \vdots & \ddots & \vdots \\
        0 & 0 & \cdots & \cW_1(\Phi_n)
        \end{pmatrix}, \begin{pmatrix}
        \cB_1(\Phi_1) \\
        \cB_1(\Phi_2) \\
        \vdots \\
        \cB_1(\Phi_n)
        \end{pmatrix}\right) \! \right. , \\
        & \left( \! \begin{pmatrix}
        \cW_2(\Phi_1) & 0 & \cdots & 0 \\
        0 & \cW_2(\Phi_2) & \cdots & 0 \\
        \vdots & \vdots & \ddots & \vdots \\
        0 & 0 & \cdots & \cW_2(\Phi_n)
        \end{pmatrix}, \begin{pmatrix}
        \cB_2(\Phi_1) \\
        \cB_2(\Phi_2) \\
        \vdots \\
        \cB_2(\Phi_n)
        \end{pmatrix} \! \right), \ldots,\\
        & \left.  \left( \! \begin{pmatrix}
        \cW_\bfL(\Phi_1) & 0 & \cdots & 0 \\
        0 & \cW_\bfL(\Phi_2) & \cdots & 0 \\
        \vdots & \vdots & \ddots & \vdots \\
        0 & 0 & \cdots & \cW_\bfL(\Phi_n)
        \end{pmatrix}, \begin{pmatrix}
        \cB_\bfL(\Phi_1) \\
        \cB_\bfL(\Phi_2) \\
        \vdots \\
        \cB_\bfL(\Phi_n)
        \end{pmatrix} \! \right) \! \right).
    \end{split}
\end{equation}
\end{definition}
\cfclear
\begin{prop} \label{prop:parallelization}
Let $n \in \N$, $\Phi_1, \Phi_2, \ldots, \Phi_n \in \NN$ satisfy $\cL(\Phi_1) = \cL(\Phi_2) = \cdots = \cL(\Phi_n)$ \cfload. Then \cfadd{def:realization} \cfadd{def:DNN} \cfadd{def:parallelization}
\begin{enumerate} [(i)]
    \item it holds for all $i \in \N_0$ that
    \begin{equation}
        \cD_i( \bfP_n(\Phi_1, \Phi_2, \ldots, \Phi_n) ) = \sum_{k=1}^n \cD_i(\Phi_k),
    \end{equation}
    \item it holds for all $a \in C(\R, \R)$ that
    \begin{equation}
        \cR_a( \bfP_n(\Phi_1, \Phi_2, \ldots, \Phi_n) ) \in C \left( \R^{\sum_{k=1}^n \cI(\Phi_k)}, \R^{\sum_{k=1}^n \cO(\Phi_k)}\right),
    \end{equation}
    and
    \item it holds for all $a \in C(\R, \R)$, $x_1 \in \R^{\cI(\Phi_1)}, \ldots, x_n \in \R^{\cI(\Phi_n)}$ that
    \begin{equation}
        \bigl( \cR_a( \bfP_n(\Phi_1, \Phi_2, \ldots, \Phi_n)) \bigr) (x_1, x_2, \ldots, x_n) = \bigl( (\cR_a(\Phi_1))(x_1), \ldots, (\cR_a(\Phi_n))(x_n) \bigr)
    \end{equation}
\end{enumerate}
\cfout.
\end{prop}

\subsection{Affine linear transformations as DNNs} \label{subsec:dnnaff}
\cfclear
\begin{definition}[Linear transformation DNNs] 
\label{def:DNN:aff} 
Let $m,n \in \N$, $W \in \R^{m \times n}$, $B \in \R^m$. Then we denote by $\bbA_{W,B} \in \NN$ the neural network given by $\bbA_{W,B} = (W,B)$ \cfout.
\end{definition}
\cfclear
\begin{prop} \label{prop:DNN:aff}
Let $m,n \in \N$, $W \in \R^{m \times n}$, $B \in \R^m$. Then
\begin{enumerate} [(i)] \cfadd{def:realization} \cfadd{def:DNN:aff} \cfadd{def:DNN}
    \item \label{item:prop:aff:1} it holds that $\cA (\bbA_{W,B})=(n,m)$,
    \item \label{item:prop:aff:2} it holds for all $a \in C(\R,\R)$ that $\cR_a(\bbA_{W,B}) \in C(\R^n, \R^m)$, and
    \item \label{item:prop:aff:3} it holds for all $a \in C(\R, \R)$, $x \in \R^n$ that $(\cR_a(\bbA_{W,B}))(x) = W x + B$
\end{enumerate}
\cfout.
\end{prop}
\begin{proof} [Proof of \cref{prop:DNN:aff}]
Note that the fact that $\bbA_{W,B} \in (\R^{m \times n} \times \R^m) \subseteq \NN$ establishes \eqref{item:prop:aff:1}. Moreover, observe that items \eqref{item:prop:aff:2} and \eqref{item:prop:aff:3} are direct consequences of \cref{def:realization}. This completes the proof of \cref{prop:DNN:aff}.
\end{proof}

\cfclear 
\begin{prop} \label{prop:affcomp}
Let $m,n \in \N$, $W \in \R^{m \times n}$, $B \in \R^m$, $a \in C(\R, \R)$, $\Phi, \Psi \in \NN$ satisfy $\cI(\Phi) = m$ and $\cO(\Psi) = n$ \cfload. Then
\begin{enumerate} [(i)] \cfadd{def:realization} \cfadd{def:DNN:aff}  \cfadd{def:composition}
    \item it holds that $\cA (\bbA_{W,B} \bullet \Psi) = (\cD_0(\Psi), \cD_1(\Psi), \ldots, \cD_{\cL(\Psi)-1}(\Psi), m)$,
    \item it holds that $\cR_a(\bbA_{W,B} \bullet \Psi) \in C(\R^{\cI(\Psi)}, \R^m)$,
    \item it holds for all $x \in \R^{\cI(\Psi)}$ that $(\cR_a(\bbA_{W,B} \bullet \Psi))(x) = W ( \cR_a(\Psi))(x)+B$,
    \item it holds that $\cA (\Phi \bullet \bbA_{W,B}) = (n, \cD_1(\Phi), \cD_2(\Phi), \ldots, \cD_{\cL(\Phi)}(\Phi))$,
    \item it holds that $\cR_a(\Phi \bullet \bbA_{W,B}) \in C(\R^n, \R^{\cO(\Phi)})$, and
    \item it holds for all $x \in \R^n$ that $(\cR_a(\bbA_{W,B}))(x) = (\cR_a(\Phi))(W x + B)$
\end{enumerate}
\cfout.
\end{prop}
\begin{proof} [Proof of \cref{prop:affcomp}]
Observe that \cref{prop:DNN:aff} establishes that it holds for all $x \in \R^n$ that $\cR_a(\bbA_{W,B}) \in C(\R^n, \R^m)$ and
\begin{equation}
    ( \cR_a(\bbA_{W,B}))(x) = W x + B.
\end{equation}
Combining this and \cref{prop:comp} completes the proof of \cref{prop:affcomp}.
\end{proof}

\subsection{Sums of vectors as DNNS} \label{subsec:dnnsum}

\begin{definition} [Identity matrix] \label{def:idmatrix}
Let $n \in \N$. Then we denote by $\id _n \in \R^{n \times n}$ the identity matrix in $\R^{n \times n}$.
\end{definition}

\cfclear
\begin{definition} [DNN representations for sums] \label{def:sum} \cfadd{def:DNN} \cfadd{def:DNN:aff} \cfadd{def:idmatrix}
Let $m,n \in \N$. Then we denote by $\fS_{m,n} \in (\R^{m \times (mn)} \times \R^m) \subseteq \NN$ the neural network given by $\fS_{m,n} = \bbA_{(\id _ m \, \id _m \, \ldots \, \id _m ), \, 0}$ \cfload.
\end{definition}

\cfclear
\begin{prop}\label{prop:sum}
Let $m,n \in \N$, $a \in C(\R, \R)$, $\Phi \in \NN$ satisfy $\cO(\Phi) = m n$ \cfload. Then 
\begin{enumerate} [(i)] \cfadd{def:sum} \cfadd{def:realization} \cfadd{def:composition}
    \item it holds that $\cR_a(\fS_{m,n} \bullet \Phi) \in C(\R^{\cI(\Phi)}, \R^m)$ and
    \item it holds for all $x \in \R^{\cI(\Phi)}$, $y_1, y_2, \ldots, y_n \in \R^m$ with $(\cR_a(\Phi))(x) = (y_1, y_2, \ldots, y_n)$ that
    \begin{equation}
        (\cR_a(\fS_{m,n} \bullet \Phi))(x) = \smallsum_{k=1}^n y_k
    \end{equation}
\end{enumerate}
\cfout.
\end{prop}
\begin{proof} [Proof of \cref{prop:sum}]
Note that it holds for all $y_1, y_2, \ldots, y_n \in \R^m$ that 
\begin{equation}
    \begin{pmatrix}
    \id _ m & \id _m & \ldots & \id _m
    \end{pmatrix} (y_1, y_2, \ldots, y_m) = \smallsum_{k=1}^n y_k.
\end{equation} Combining this with \cref{prop:affcomp} completes the proof of \cref{prop:sum}.
\end{proof}

\subsection{Concatenations of vectors as DNNs} \label{subsec:dnnconc}

\cfclear 
\begin{definition} [Transpose] \label{def:transpose}
Let $B \in \R^{m \times n}$. Then we denote by $B^T \in \R^{n \times m}$ the transpose of $B$.
\end{definition}

\cfclear
\begin{definition} [DNN representations for concatenations]\label{def:concat} \cfadd{def:transpose} \cfadd{def:DNN:aff}
Let $m,n \in \N$. Then we denote by $\fT_{m,n} \in (\R^{(mn) \times n} \times \R^{mn}) \subseteq \NN$ the neural network given by $\fT_{m,n} = \bbA_{(\id _ m \, \id _m \, \ldots \, \id _m)^T, \, 0}$ \cfload.
\end{definition}

\cfclear 
\begin{prop} \label{prop:concat}
Let $m,n \in \N$, $a \in C(\R, \R)$, $\Phi \in \NN$ satisfy $\cI(\Phi) = m n$ \cfload. Then
\begin{enumerate} [(i)] \cfadd{def:composition} \cfadd{def:concat} \cfadd{def:realization}
    \item it holds that $\cR_a(\Phi \bullet \fT_{m,n}) \in C(\R^n, \R^{\cO(\Phi)})$ and
    \item it holds for all $x \in \R^n$ that $(\cR_a(\Phi \bullet \fT_{m,n}))(x) = (\cR_a(\Phi))(x,x, \ldots, x)$
\end{enumerate}
\cfout.
\end{prop}
\begin{proof} [Proof of \cref{prop:concat}]
Note that it holds for all $x \in \R^n$ that $(\id _ m \, \id _m \, \ldots \, \id _m)^Tx = (x,x, \ldots, x) \allowbreak \in \R^{mn}$. Combining this with \cref{prop:affcomp} completes the proof of \cref{prop:concat}.
\end{proof}

\section{DNN representations} \label{sec:dnnrep}
In this section we present three DNN representation results which rely on the DNN calculus developed in \cref{sec:dnns}. These results are elementary, and we only include them for the purpose of being self-contained. First, in \cref{subsec:dnnl1} we recall in \cref{prop:dnn:l1norm} below that the standard $1$-norm on $\R^d$ (cf.\ \cref{def:p-norm} below) can be represented by a DNN with one hidden layer, and we analyze in \cref{lem:l1norm:param} the magnitude of the parameters of this DNN.

Afterwards, in \cref{subsec:dnnmax} we explain how the maximum of $d$ real numbers can be computed by a DNN with $\cO( \log d)$ hidden layers (cf.\ \cref{def:max_d} and \cref{Prop:max_d} below). This representation of the maximum is well-known in the scientific literature. The construction uses a DNN representation for the real identity with one hidden layer (cf.\ \cref{def:ReLU_identity}), which is also well-known in the literature (cf., e.g., \cite[Definition 2.18]{BeckJentzenKuckuck2019arXiv}).

In \cref{subsec:dnninterpolation} we employ these representations to construct a DNN which computes a maximum of a certain form: If $f \colon \R^d \to \R$ is a Lipschitz continuous function with Lipschitz constant $L \in [0, \infty)$ and if $\cM \subseteq \R^d$ is a suitably chosen finite subset of $\R^d$, then it is known that the function $F$ given by $F(x) = \max_{y \in \cM} (f(y) - L \norm{x-y}_1)$ is a good approximation for $f$ (cf., e.g., \cite[Lemma 3.1]{BeckJentzenKuckuck2019arXiv}). We show in \cref{dnn:intl1} below how this function $F$ can be represented by a DNN with depth $\cO(\log | \cM |)$ and we also estimate the layer dimensions and the magnitude of the parameters of this DNN from above. \cref{dnn:intl1} is a slightly strengthened version of \cite[Lemma 3.4] {BeckJentzenKuckuck2019arXiv}.

\subsection{DNN representations for the 1-norm} \label{subsec:dnnl1}

\begin{definition}[$p$-norms]
\label{def:p-norm}
We denote by
$ \norm{\cdot}_p \colon \bigcup_{d \in \N} \R^d \to [ 0, \infty ) $, $ p \in [ 1, \infty ] $,
the functions which satisfy for all
$ p \in [ 1, \infty ) $, $d \in \N$,
$ x = ( x_1, x_2, \ldots, x_d ) \in \R^d $
that
$\norm{\theta}_p = \bigl[\sum_{i=1}^d \abs{\theta_i}^p \bigr]^{ \nicefrac{1}{p} }$ and $\norm{\theta}_\infty
= \max_{ i \in \{ 1, 2, \ldots, d \} } \abs{\theta_i}$.
\end{definition}

\cfclear
\begin{definition} [$1$-norm DNN representations] \label{def:dnn:l1norm} \cfadd{def:sum} \cfadd{def:parallelization}
We denote by $(\bbL_d)_{d \in \N} \subseteq \NN$ the neural networks which satisfy that \cfadd{def:composition}
\begin{enumerate} [(i)]
    \item it holds that \begin{equation}
 \bbL_1 = \left( \! \left( \! \begin{pmatrix}
    1 \\ -1
    \end{pmatrix}, \begin{pmatrix}
    0 \\ 0
    \end{pmatrix} \!\right), \left( \begin{pmatrix}
    1 & 1
    \end{pmatrix}, \begin{pmatrix}
    0 
    \end{pmatrix} \right) \! \right) \in (\R^{2 \times 1} \times \R^{2}) \times (\R^{1 \times 2} \times \R^{1}) 
\end{equation}
and
\item it holds for all $d \in \{2,3,4, \ldots\}$ that $\bbL_d =  \fS_{1,d} \bullet \bfP_d(\bbL_1, \bbL_1, \ldots, \bbL_1)$
\end{enumerate}
\cfout.
\end{definition}

\cfclear
\begin{prop} \label{prop:dnn:l1norm} \cfadd{def:dnn:l1norm} \cfadd{def:realization} \cfadd{def:p-norm} \cfadd{def:rect}
Let $d \in \N$. Then \cfadd{def:DNN}
\begin{enumerate} [(i)]
    \item \label{item:prop:l1:1} it holds that $\cA (\bbL_d) = (d, 2d, 1)$,
    \item \label{item:prop:l1:2} it holds that $\cR_\fr(\bbL_d) \in C(\R^d, \R)$, and
    \item \label{item:prop:l1:3} it holds for all $x \in \R^d$ that $(\cR_\fr(\bbL_d))(x) = \norm{x}_1$
\end{enumerate} \cfout.
\end{prop}
\begin{proof} [Proof of \cref{prop:dnn:l1norm}]
Note that $\cA (\bbL_1) = (1,2,1)$. This and \cref{prop:parallelization} show for all $d \in \{2,3,4, \ldots\}$ that $\cA (\bfP_d(\bbL_1, \bbL_1, \ldots, \bbL_1)) = (d, 2d, d)$. Combining this, \cref{prop:comp}, and \cref{prop:DNN:aff} ensures for all $d \in \{2,3,4, \ldots\}$ that $\cA (\fS_{1,d} \bullet \bfP_d(\bbL_1, \bbL_1, \ldots, \bbL_1)) = (d, 2d, 1)$. This establishes item \eqref{item:prop:l1:1}.
Furthermore, observe that it holds for all $x \in \R$ that
\begin{equation}
    (\cR_\fr(\bbL_1))(x) = \fr(x) + \fr(-x) = \max \{x, 0\} + \max \{-x, 0\} = |x| = \| x \|_1.
\end{equation}
Combining this and \cref{prop:parallelization} shows for all $d \in \{2,3,4, \ldots, \}$, $x=(x_1, x_2, \ldots, x_d) \in \R^d$ that
\begin{equation}
    \bigl( \cR_\fr(\bfP_d(\bbL_1, \bbL_1, \ldots, \bbL_1)) \bigr) (x) = (|x_1|, |x_2|, \ldots, |x_d|).
\end{equation}
This and \cref{prop:sum} demonstrate that for all $d \in \{2,3,4, \ldots, \}$, $x=(x_1, x_2, \ldots, x_d) \in \R^d$ we have that
\begin{equation}
    (\cR_\fr(\bbL_d))(x) = \bigl( \cR_\fr(\fS_{1,d} \bullet \bfP_d(\bbL_1, \bbL_1, \ldots, \bbL_1)) \bigr) (x) = \smallsum_{k=1}^d |x_k| = \norm{x}_1.
\end{equation}
This establishes \eqref{item:prop:l1:2} and \eqref{item:prop:l1:3}. The proof of \cref{prop:dnn:l1norm} is thus complete.
\end{proof}

\cfclear
\begin{lemma} \label{lem:l1norm:param}
Let $d \in \N$. Then
\begin{enumerate} [(i)] \cfadd{def:dnn:l1norm} \cfadd{def:p-norm} \cfadd{def:DNN}
\item \label{item:lem:l1norm:1} it holds that $\cB_1(\bbL_d)=0 \in \R^{2d}$,
\item \label{item:lem:l1norm:1a} it holds that $\cB_2(\bbL_d) = 0 \in \R$,
    \item \label{item:lem:l1norm:2} it holds that $\cW_1(\bbL_d)  \in \{-1, 0, 1\} ^{(2d) \times d}$,
    \item \label{item:lem:l1norm:3} it holds for all $x \in \R^d$ that $\norm{\cW_1(\bbL_d) x }_\infty = \| x \|_\infty $, and
    \item \label{item:lem:l1norm:4} it holds that $\cW_2(\bbL_d) = ( 1 \ 1 \ \cdots \ 1) \in \R^{1 \times (2d)}$
\end{enumerate}
\cfout.
\end{lemma}
\begin{proof}[Proof of \cref{lem:l1norm:param}]
Note first that $\cB_1(\bbL_1) = 0 \in \R^2$ and $\cB_2(\bbL_1)=0 \in \R$. This, the fact that $\forall \, d \in \{2,3,4, \ldots \} \colon  \bbL_d =  \fS_{1,d} \bullet \bfP_d(\bbL_1, \bbL_1, \ldots, \bbL_1)$, and the fact that $\forall \, d \in \{2,3,4, \ldots \} \colon \cB_1(\fS_{1,d}) \allowbreak = 0 \in \R$ establish \eqref{item:lem:l1norm:1} and \eqref{item:lem:l1norm:1a}.
In addition, observe that it holds for all $d \in \{2,3,4, \ldots\}$ that
\begin{equation}
    \cW_1(\bbL_1) = \begin{pmatrix}
    1 \\ -1
    \end{pmatrix}
    \qandq
    \cW_1(\bbL_d) = \begin{pmatrix}
     \cW_1(\bbL_1) & 0 & \cdots & 0 \\
    0 & \cW_1(\bbL_1) & \cdots & 0 \\
    \vdots & \vdots & \ddots & \vdots \\
    0 & 0 &\cdots & \cW_1(\bbL_1)
    \end{pmatrix} \in \R^{(2d) \times d}.
\end{equation}
This proves items \eqref{item:lem:l1norm:2} and \eqref{item:lem:l1norm:3}. Finally, note that the fact that $\cW_2(\bbL_1) = (1 \ 1)$ and the fact that $\forall \, d \in \{2,3,4, \ldots \} \colon  \bbL_d =  \fS_{1,d} \bullet \bfP_d(\bbL_1, \bbL_1, \ldots, \bbL_1)$ show for all $d \in \{2,3,4, \ldots\}$ that
\begin{equation}
    \cW_2( \bbL_d) = \underbrace{\begin{pmatrix} 1 & 1 & \cdots & 1\end{pmatrix}}_{\in \R^{1 \times d}} \begin{pmatrix}
    \cW_2(\bbL_1) & 0 & \cdots & 0 \\
    0 & \cW_2(\bbL_1) & \cdots & 0 \\
    \vdots & \vdots & \ddots & \vdots \\
    0 & 0 &\cdots & \cW_2(\bbL_1) \end{pmatrix}
    = \begin{pmatrix} 1 & 1 & \cdots & 1 \end{pmatrix} \in \R^{1 \times (2d)}.
\end{equation}
This establishes \eqref{item:lem:l1norm:4} and thus completes the proof of \cref{lem:l1norm:param}.
\end{proof}

\subsection{DNN representations for maxima}
\label{subsec:dnnmax}

\begin{definition} [Real identity as DNN] \label{def:ReLU_identity}
We denote by $\ReLUidANN{1} \in \NN$ the neural network given by
\begin{equation}
    \ReLUidANN{1} = \left( \! \left( \! \begin{pmatrix}
    1 \\ -1
    \end{pmatrix}, \begin{pmatrix}
    0 \\ 0
    \end{pmatrix} \!\right), \left( \begin{pmatrix}
    1 & -1
    \end{pmatrix}, \begin{pmatrix}
    0 
    \end{pmatrix} \right) \! \right) \in (\R^{2 \times 1} \times \R^{2}) \times (\R^{1 \times 2} \times \R^{1}) 
\end{equation}
\cfload.
\end{definition}

\cfclear
\begin{prop} \label{prop:idDNN} \cfadd{def:realization} \cfadd{def:ReLU_identity}
It holds for all $x \in \R$ that $(\cR_\fr(\ReLUidANN{1}))(x)=x$ \cfload.
\end{prop}
\begin{proof} [Proof of \cref{prop:idDNN}]
Observe that it holds for all $x \in \R$ that
\begin{equation}
    (\cR_\fr(\ReLUidANN{1}))(x) = \fr(x) - \fr(-x) = \max \{x, 0 \} - \max \{ -x, 0 \} = x.
\end{equation}
The proof of \cref{prop:idDNN} is thus complete.
\end{proof}

\cfclear
\begin{lemma} \label{lem:max_d_welldef}
There exist unique $\phi_d \in  \NN$, $d \in \{2,3, 4,\ldots \}$, which satisfy that \cfadd{def:parallelization} \cfadd{def:composition} \cfadd{def:ReLU_identity} \cfadd{def:DNN}
\begin{enumerate} [(i)]
    \item \label{welldef:item1} it holds for all $d \in  \{2,3,4, \ldots \}$ that $\cI(\phi_d) = d$,
    \item \label{welldef:item2} it holds for all $d \in  \{2,3,4, \ldots \}$ that $\cO(\phi_d) = 1$,
    \item it holds that
    \begin{equation}
        \phi_2 = \left( \! \left( \!
	\begin{pmatrix}
		1 & -1 \\
		0 & 1 \\
		0 & -1
	\end{pmatrix},
	\begin{pmatrix}
		0 \\
		0 \\
		0
	\end{pmatrix} \!  \right), \left(
	\begin{pmatrix}
		1 & 1 & -1
	\end{pmatrix}, \begin{pmatrix} 0 \end{pmatrix} \right)  \! \right) \in (\R^{3 \times 2} \times \R^3) \times (\R^{1 \times 3} \times \R^1 ),
    \end{equation}
    \item \label{welldef:item3} it holds for all $d \in  \{2,3,4, \ldots \}$ that $\phi_{2d} = 	\phi_{d} \bullet \big(\bfP_{d}( \phi_2, \phi_2, \ldots, \phi_2) \big)$, and
    \item it holds for all $d \in \{2,3,4, \ldots \}$ that $\phi_{2d-1} = \phi_d \bullet \bigl( \bfP_d(\phi_2, \phi_2, \dots, \phi_2, \ReLUidANN{1} ) \bigr)$.
    \end{enumerate}
\cfload.
\end{lemma}
\begin{proof}[Proof of \cref{lem:max_d_welldef}]
Throughout this proof let $\psi \in \NN$ be given by
\begin{equation}
    \psi = \left( \! \left( \!
	\begin{pmatrix}
		1 & -1 \\
		0 & 1 \\
		0 & -1
	\end{pmatrix},
	\begin{pmatrix}
		0 \\
		0 \\
		0
	\end{pmatrix} \!  \right), \left(
	\begin{pmatrix}
		1 & 1 & -1
	\end{pmatrix},
	\begin{pmatrix} 0 \end{pmatrix} \right) \!  \right) \in (\R^{3 \times 2} \times \R^3) \times (\R^{1 \times 3} \times \R^1 ).
\end{equation} Observe that it holds that $\cI(\psi) = 2$, $\cO(\psi) = 1$, and $\cL(\psi) = \cL(\ReLUidANN{1})=2$. Combining this with \cref{prop:parallelization} shows for all $d \in \N$ that
$\cI(\bfP_d (\psi, \psi, \ldots, \psi)) = 2d$, $\cO(\bfP_{d} (\psi, \psi, \ldots, \psi)) = d$, 
$\cI(\bfP_d (\psi, \psi, \ldots, \psi, \ReLUidANN{1})) = 2d-1$,
and
$\cO( \bfP_{d} (\psi, \psi, \ldots, \psi, \ReLUidANN{1})) = d$.
This, \cref{prop:comp}, and induction establish that for all $d \in \{2,3,4, \ldots \}$, $\phi_d$ is well-defined and satisfies $\cI(\phi_d)=d$ and $\cO(\phi_d)=1$. The proof of \cref{lem:max_d_welldef} is thus complete.
\end{proof}

\cfclear

\begin{definition}
[Maxima DNN representations]
\label{def:max_d}
We denote by $(\bbM_d)_{d \in \{2,3,4, \ldots \}} \subseteq \NN$ \cfadd{def:DNN} the neural networks which satisfy that \cfadd{def:parallelization} \cfadd{def:composition} \cfadd{def:ReLU_identity} \cfadd{lem:max_d_welldef}
\begin{enumerate} [(i)]
    \item it holds for all $d \in  \{2,3,4, \ldots \}$ that $\cI(\bbM_d) = d$,
    \item it holds for all $d \in  \{2,3,4, \ldots \}$ that $\cO(\bbM_d) = 1$, and
     \item it holds that
    \begin{equation} \label{eq:def:m2}
        \bbM_2 = \left( \! \left( \!
	\begin{pmatrix}
		1 & -1 \\
		0 & 1 \\
		0 & -1
	\end{pmatrix},
	\begin{pmatrix}
		0 \\
		0 \\
		0
	\end{pmatrix} \!  \right), \left(
	\begin{pmatrix}
		1 & 1 & -1
	\end{pmatrix}, \begin{pmatrix} 0 \end{pmatrix} \right)  \! \right) \in (\R^{3 \times 2} \times \R^3) \times (\R^{1 \times 3} \times \R^1 ),
    \end{equation}
    \item it holds for all $d \in  \{2,3,4, \ldots \}$ that $\bbM_{2d} = 	\bbM_{d} \bullet \big(\bfP_{d}( \bbM_2, \bbM_2, \ldots, \bbM_2) \big)$, and
    \item it holds for all $d \in \{2,3,4, \ldots \}$ that $\bbM_{2d-1} = \bbM_d \bullet \bigl( \bfP_d(\bbM_2, \bbM_2, \dots, \bbM_2, \ReLUidANN{1} ) \bigr)$.
\end{enumerate}
\cfload.
\end{definition}

\begin{definition}
[Floor and ceiling of real numbers]
\label{def:ceiling}
We denote by $\ceil{\cdot} \! \colon \R \to \Z$ and $\floor{\cdot} \! \colon \R \to \Z$ the functions which satisfy for all $x \in \R$ that $\ceil{x} = \min(\Z \cap [x, \infty))$ and $\floor{x} = \max(\Z \cap (-\infty, x])$.
\end{definition}

\cfclear
\begin{prop}
\label{Prop:max_d}
Let $d \in  \{2,3,4, \ldots \}$.
Then \cfadd{def:DNN} \cfadd{def:max_d} \cfadd{def:realization} 
\begin{enumerate}[(i)]
\item \label{max_d:item_3} it holds that
$\cH(\bbM_d) = \ceil{\log_2(d)}$,
\item \label{max_d:item_4} it holds for all $i \in \N$ that
$\cD_i(\bbM_d) \leq 3\ceil{\tfrac{d}{2^{i}}}$,
\item \label{max_d:item_5} it holds that $\cR_\fr(\bbM_d)  \in C(\R^d,\R)$, and
\item \label{max_d:item_6} it holds for all $x = (x_1,x_2, \ldots, x_d) \in \R^d$ that $(\cR_\fr(\bbM_d))(x) = \max\{x_1, x_2, \ldots, x_d\}$
\end{enumerate}
\cfload.
\end{prop}

\begin{proof}[Proof of \cref{Prop:max_d}.]
Note that \eqref{eq:def:m2} ensures that $\cH(\bbM_2) = 1$.
This and \cref{def:parallelization} demonstrate that for all $\fd \in \{2,3,4, \ldots\}$ it holds that
\begin{equation}
    \cH (\bfP_{\fd} (\bbM_2, \bbM_2, \ldots, \bbM_2)) =  \cH (\bfP_{\fd} (\bbM_2, \bbM_2, \ldots, \bbM_2, \ReLUidANN{1})) = 1.
\end{equation}
Combining this with \cref{prop:comp} establishes for all $\fd \in \{3,4,5,\ldots\}$ that $\cH (\bbM_\fd) = \cH (\bbM_{\ceil{\nicefrac{\fd}{2}}})+1$. This and induction establish item \eqref{max_d:item_3}. Next note that $\cA (\bbM_2) = (2,3,1)$.
Moreover, observe that \cref{def:max_d},
\cref{prop:parallelization}, and \cref{prop:comp} imply that for all $\fd \in \{2,3,4, \ldots\}$, $i \in \N$ it holds that
\begin{equation}
    \cD_i(\bbM_{2 \fd}) = \begin{cases}
    3 \fd & \colon i=1 \\
    \cD_{i-1}(\bbM_\fd) & \colon i \geq 2
    \end{cases}
\end{equation}
and
\begin{equation}
     \cD_i(\bbM_{2\fd-1}) = \begin{cases}
    3\fd-1 & \colon i=1 \\
    \cD_{i-1}(\bbM_\fd) & \colon i \geq 2.
    \end{cases}
\end{equation}
Together with induction this proves item \eqref{max_d:item_4}. In addition, observe that
\eqref{eq:def:m2} ensures for all $x=(x_1, x_2) \in \R^2$ that
\begin{equation}
\begin{split}
    (\cR_\fr(\bbM_2))(x) &= \max \{x_1-x_2, 0\} + \max \{ x_2 , 0\} - \max\{ -x_2 , 0\} \\
    &= \max \{x_1-x_2, 0\} + x_2 = \max\{x_1, x_2\}.
    \end{split}
\end{equation}
Combining this, \cref{prop:parallelization}, \cref{prop:comp}, \cref{prop:idDNN}, and induction implies that for all $d \in  \{2,3, 4, \ldots \}$, $x= ( x_1,x_2,\dots,x_d) \in \R^d$ it holds that $\cR_\fr(\bbM_d) \in C(\R^d,\R)$
and
$\left(\cR_\fr({\bbM_d})\right)(x) = \max\{x_1,x_2,\dots,x_d\}$.
This establishes items \eqref{max_d:item_5}--\eqref{max_d:item_6} and thus completes the proof of \cref{Prop:max_d}.
\end{proof}

 \cfclear
\begin{lemma} \cfadd{def:max_d} \cfadd{def:DNN}
\label{t:max:d} Let $d \in \{2,3,4, \ldots \}$, $i \in \{1,2,  \ldots, \cL(\bbM_d) \}$ \cfload. 
Then \cfadd{def:p-norm} 
\begin{enumerate}[(i)]
\item \label{t:max_d:item_1} it holds that $\cB_i(\bbM_d)= 0 \in \R^{\cD_i(\bbM_d)}$,
\item \label{t:max_d:item_2} it holds that $ \cW_i(\bbM_d) \in \{-1,0,1\}^{\cD_i(\bbM_d) \times \cD_{i-1}(\bbM_d) }$, and
\item \label{t:max:d:item_3} it holds for all $x \in \R^d$ that $\| \cW_1( \bbM_d) x \|_\infty \leq 2 \| x \|_\infty$.
\end{enumerate}
\cfout.
\end{lemma}

\begin{proof}[Proof of \cref{t:max:d}]
Throughout this proof let $A_1 \in \R^{3 \times 2}$, $A_2 \in \R ^{1 \times 3}$, $C_1 \in \R^{2 \times 1}$, $C_2 \in \R^{1 \times 2}$ be given by $A_1 = \begin{pmatrix} 1 & -1 \\ 0 & 1 \\ 0 & -1 \end{pmatrix}$, $A_2 = \begin{pmatrix}1 & 1 & -1 \end{pmatrix}$, $C_1 = \begin{pmatrix}1 \\ -1 \end{pmatrix}$, $C_2 = \begin{pmatrix} 1 & -1 \end{pmatrix}$. Observe that \eqref{eq:def:m2} ensures that all four statements hold for $d=2$. Furthermore, note that for all $\fd \in \{2,3,4,\ldots \}$ it holds that
\begin{equation} \label{propmd:eq1}
\begin{split}
    \cW_1(\bbM_{2\fd-1}) =  \underbrace{\begin{pmatrix}A_1 & 0 & \cdots & 0 \\
    0 & A_1 & \cdots & 0 \\
    \vdots & \vdots & \ddots & \vdots \\
    0 & 0 & \cdots & C_1 \end{pmatrix}}_{\in \R^{(3\fd-1)\times (2\fd-1)}}, & \quad 
    \cW_1( \bbM_{2\fd }) =  \underbrace{
    \begin{pmatrix}A_1 & 0 & \cdots & 0 \\
    0 & A_1 & \cdots & 0 \\
    \vdots & \vdots & \ddots & \vdots \\
    0 & 0 & \cdots & A_1 \end{pmatrix}}_{\in \R^{(3 \fd) \times (2 \fd)}}, \\
    \cB_1(\bbM_{2\fd-1}) = 0 \in \R^{3 \fd -1}, &\qandq  \cB_1(\bbM_{2\fd}) = 0 \in \R^{3 \fd}.
    \end{split}
\end{equation}
This proves item \eqref{t:max:d:item_3}. In addition, observe that for all $\fd \in \{2,3,4,\ldots \}$ it holds that
\begin{equation} \label{propmd:eq2}
\begin{split}
    \cW_2(\bbM_{2\fd-1}) = \cW_1(\bbM_\fd)   \underbrace{\begin{pmatrix}A_2 & 0 & \cdots & 0 \\
    0 & A_2 & \cdots & 0 \\
    \vdots & \vdots & \ddots & \vdots \\
    0 & 0 & \cdots & C_2  \end{pmatrix}}_{\in\R^{\fd\times (3\fd -1)}}, &\quad 
      \cW_2(\bbM_{2\fd}) = \cW_1(\bbM_\fd) \underbrace{
    \begin{pmatrix}A_2 & 0 & \cdots & 0 \\
    0 & A_2 & \cdots & 0 \\
    \vdots & \vdots & \ddots & \vdots \\
    0 & 0 & \cdots & A_2  \end{pmatrix}}_{\in \R^{\fd \times (3 \fd)}}, \\
    \cB_2(\bbM_{2\fd-1}) = \cB_1(\bbM_\fd), &\qandq \cB_2(\bbM_{2 \fd}) = \cB_1(\bbM_\fd).
\end{split}
\end{equation}
Finally, observe that \cref{prop:comp} demonstrates that for all $ \fd \in \{2,3,4, \ldots, \}$, $i \in \{3,4, \allowbreak \ldots, \cL(\bbM_\fd)+1 \}$ we have that
\begin{equation} \label{propmd:eq3}
    \cW_i(\bbM_{2 \fd -1}) = \cW_i(\bbM_{2 \fd}) = \cW_{i-1}(\bbM_\fd) \qandq   \cB_i(\bbM_{2 \fd -1}) = \cB_i(\bbM_{2 \fd}) = \cB_{i-1}(\bbM_\fd).
\end{equation}
Combining \eqref{propmd:eq1}--\eqref{propmd:eq3} with induction establishes items \eqref{t:max_d:item_1} and \eqref{t:max_d:item_2}. The proof of \cref{t:max:d} is thus complete.
\end{proof}

\subsection{DNN representations for maximum convolutions}
\label{subsec:dnninterpolation}
\cfclear
\begin{lemma}\label{dnn:intl1}
Let $d \in \N$, $L \in [0, \infty)$, $K \in \{2,3,4, \ldots \}$, $\fx_1, \fx_2, \ldots, \fx_K \in \R^d$, $\fy = (\fy_1, \fy_2, \ldots, \allowbreak \fy_K) \in \R^K$, let $F \colon \R^d \to \R$ satisfy for all $x  \in \R^d$ that \cfadd{def:p-norm}
\begin{equation}
    F(x) = \max\nolimits_{k  \in \{1, 2, \ldots, K \} } \left( \fy_k - L \norm{x-\fx_k}_1 \right),
\end{equation}
and let $\Phi \in \NN$ be given by \cfadd{def:concat} \cfadd{def:dnn:l1norm} \cfadd{def:composition} \cfadd{def:DNN:aff} \cfadd{def:parallelization} \cfadd{def:max_d}
\begin{equation}
    \Phi = \bbM_{K} \bullet \bbA_{-L \id_{K}, \fy} \bullet \bfP_{K}  \bigl(  \bbL_d \bullet \bbA_{\id _d, -\fx_1},  \bbL_d \bullet \bbA_{\id _d, -\fx_2}, \ldots,  \bbL_d \bullet \bbA_{\id _d, -\fx_K} \bigr) \bullet \fT_{K, d}
\end{equation}
\cfload. Then \cfadd{def:realization} \cfadd{def:DNNparam}
\begin{enumerate}[(i)]
\item \label{item:dnn:intl1:1}
it holds that $\cI(\Phi) = d$,
\item \label{item:dnn:intl1:1a} it holds that $\cO(\Phi) = 1$,
   \item \label{item:dnn:intl1:2}
    it holds that $\cH (\Phi) = \ceil{ \log _2 K } + 1$,
\item \label{item:dnn:intl1:3}
it holds that $\cD_1 ( \Phi) = 2 d K$,
\item \label{item:dnn:intl1:4}
it holds for all $i \in \{2,3, \ldots \}$ that $\cD_i (\Phi) \leq 3 \ceil{ \frac{K}{2^{i-1}}} $,
    \item \label{item:dnn:intl1:5}
    it holds that $\norm{ \cT(\Phi)}_\infty \leq \max \{ 1,L, \max_{k \in \{1, 2, \ldots, K\}} \norm{ \fx_k}_\infty, 2 \norm{\fy}_\infty\}$, and
    \item \label{item:dnn:intl1:6}
    it holds that $\cR_\fr(\Phi) = F$
\end{enumerate}
\cfout.
\end{lemma}

\begin{proof}[Proof of \cref{dnn:intl1}]
Throughout this proof let $\Psi_k \in \NN$, $k \in \{1,2, \ldots, K\}$, satisfy for all $k \in \{1,2, \ldots, K\}$ that $\Psi_k =  \bbL_d \bullet \bbA_{\id _d, -\fx_k}$, let $\Xi \in \NN$ be given by 
\begin{equation}
    \Xi = \bbA_{-L \id_{K}, \fy} \bullet \bfP_{K}  \bigl( \Psi_1, \Psi_2, \ldots, \Psi_{K} \bigr) \bullet \fT_{K, d},
\end{equation}
 and let $\normmm{\cdot} \colon \bigcup_{m, n \in \N} \R^{m \times n} \to [0, \infty)$ satisfy for all $m,n \in \N$, $M = (M_{i,j})_{i \in \{1, \ldots, m\}, \,  j \in \{1, \ldots, n \} } \in \R^{m \times n}$ that $\normmm{M} = \max_{i \in \{1, \ldots, m\}, \, j \in \{1, \ldots, n\}} |M_{i,j}|$.
Observe that it holds that $\Phi = \bbM_K \bullet \Xi$. \cref{def:max_d} therefore shows that $\cO(\Phi) = \cO(\bbM_{K})=1$. Next note that \cref{def:concat} implies that $\cI(\Phi) = \cI( \fT_{K, d}) = d$. This proves \eqref{item:dnn:intl1:1} and \eqref{item:dnn:intl1:1a}. Moreover, the fact that $\cH (\bbL_d) = 1$, the fact that $\forall \, m,n \in \N, \, \fW \in \R^{m \times n}, \, \fB \in \R^m \colon \cH (\bbA_{\fW, \fB}) = 0$, \cref{prop:comp}, and \cref{Prop:max_d} ensure that for all $i \in \{2,3, \ldots\}$ we have that
\begin{equation}
    \cH (\Phi) = \cH (\bbM_{K} \bullet \Xi) = \cH (\bbM_{K}) + \cH (\Xi) = \cH (\bbM_{K}) + 1 = \ceil{ \log _2 K} +1
\end{equation}
and
\begin{equation}
    \cD_i (\Phi) = \cD_{i-1}(\bbM_{K}) \leq 3 \ceil{ \tfrac{K}{2^{i-1}}} .
\end{equation}
Furthermore, \cref{prop:comp}, \cref{prop:parallelization}, and \cref{prop:dnn:l1norm} assure that
\begin{equation}
    \cD_1(\Phi) = \cD_1(\Xi)= \cD_1 \! \left( \bfP_{K}(\Psi_1, \Psi_2, \ldots, \Psi_{K})\right) = \sum_{i=1}^{K} \cD_1 \! \left( \Psi_i\right) = \sum_{i=1}^{K} \cD_1( \bbL_d) = 2 d K.
\end{equation}
This establishes items \eqref{item:dnn:intl1:2}--\eqref{item:dnn:intl1:4}. In the next step we prove \eqref{item:dnn:intl1:5}. Observe that \cref{t:max:d} implies that
\begin{equation}  \label{proof:eqphi}
\begin{split}
    \Phi = \bigl( &(\cW_1(\Xi), \cB_1(\Xi)),(\cW_1( \bbM_{K}) \cW_2(\Xi), \cW_1( \bbM_{K}) \cB_2(\Xi)), \\
    &(\cW_2( \bbM_{K}), 0), \ldots, (\cW_{\cL( \bbM_{K})}( \bbM_{K}), 0) \bigr).
\end{split}
\end{equation}
Moreover, note that it holds for all $k \in \{1,2, \ldots, K \}$ that $\cW_1(\Psi_k) = \cW_1(\bbL_d)$. This proves that
\begin{equation}
    \cW_1(\Xi) = \cW_1 \bigl( \bfP_{K}(\Psi_1, \Psi_2, \ldots, \Psi_{K}) \bullet \fT_{K, d} \bigr) = \begin{pmatrix}
    \cW_1(\Psi_1) \\
    \cW_1(\Psi_2) \\
    \vdots \\
    \cW_1(\Psi_{K})
    \end{pmatrix}
    = \begin{pmatrix}
    \cW_1(\bbL_d) \\
    \cW_1(\bbL_d) \\
    \vdots \\
    \cW_1(\bbL_d)
    \end{pmatrix}.
\end{equation}
\cref{lem:l1norm:param} hence demonstrates that $\normmm{\cW_1(\Xi)} = 1$. 
In addition, observe that \cref{lem:l1norm:param} implies for all $k \in \{1,2, \ldots, K \}$ that $\cB_1(\Psi_k) = \cB_1(\bbL_d \bullet \bbA_{\id_d, -\fx_k}) = -\cW_1( \bbL_d) \fx_k$ and therefore $\| \cB_1(\Psi_k) \|_\infty = \| \fx_k \|_\infty \leq \max_{k \in \{1, 2, \ldots, K\}} \norm{ \fx_k}_\infty$. This and the fact that
\begin{equation}
    \cB_1(\Xi) = \cB_1 \bigl(\bfP_{K}(\Psi_1, \Psi_2, \ldots, \Psi_{K}) \bullet \fT_{K, d} \bigr) = \begin{pmatrix}
    \cB_1(\Psi_1) \\
    \cB_1(\Psi_2) \\
    \vdots \\
    \cB_1(\Psi_{K})
    \end{pmatrix}
\end{equation}
demonstrate that $\| \cB_1(\Xi)\|_\infty \leq \max_{k \in \{1, 2, \ldots, K\}} \norm{ \fx_k}_\infty$.
Combining this, \eqref{proof:eqphi}, and \cref{t:max:d} shows that
\begin{equation} \label{eq:tphiinfty}
    \begin{split}
        \norm{\cT(\Phi)}_\infty 
        &= \max \left\{ \normmm{\cW_1(\Xi)}, \norm{\cB_1(\Xi)}_\infty, \normmm{\cW_1 (\bbM_{K}) \cW_2(\Xi)}, \|\cW_1(\bbM_{K}) \cB_2(\Xi)\|_\infty , 1 \right\} \\
        &\leq \max \left\{ 1, \max\nolimits_{k \in \{1, 2, \ldots, K\}} \norm{ \fx_k}_\infty, \normmm{ \cW_1(\bbM_{K}) \cW_2 (\Xi)}, \| \cW_1(\bbM_{K}) \cB_2 (\Xi) \|_\infty \right\}.
    \end{split}
\end{equation}
Next note that \cref{lem:l1norm:param} ensures for all $k \in \{1,2, \ldots, K \}$ that $\cB_2(\Psi_k) = \cB_2(\bbL _d) = 0$ and therefore $\cB_2 \bigl( \bfP_{K}(\Psi_1, \Psi_2, \ldots, \Psi_{K}) \bigr) = 0$. This implies that
\begin{equation} \label{proof:b2xi}
    \cB_2(\Xi) = \cB_2 \bigl( \bbA_{-L \id_{K}, \fy} \bullet \bfP_{K}  ( \Psi_1, \Psi_2, \ldots, \Psi_{K} ) \bigr) = \fy.
\end{equation}
In addition, observe that it holds for all $k \in \{1,2, \ldots, K \}$ that $\cW_2(\Psi_k) = \cW_2(\bbL_d) $ and thus
\begin{equation} \label{proof:w2xi}
\begin{split}
\cW_2(\Xi) &= \cW_2 \bigl( \bbA_{-L \id_{K}, \fy} \bullet \bfP_{K}  ( \Psi_1, \Psi_2, \ldots, \Psi_{K} ) \bigr) = -L \cW_2 \bigl( \bfP_{K}  ( \Psi_1, \Psi_2, \ldots, \Psi_{K} ) \bigr) \\
&= \begin{pmatrix}
 -L\cW_2(\bbL_d) & 0 & \cdots & 0 \\
    0 & -L\cW_2(\bbL_d) & \cdots & 0 \\
    \vdots & \vdots & \ddots & \vdots \\
    0 & 0 &\cdots & -L\cW_2(\bbL_d)
\end{pmatrix}
\end{split}
\end{equation}
Moreover, note that \cref{lem:l1norm:param} ensures that $\cW_2(\bbL_d) = \begin{pmatrix} 1 & 1 & \cdots & 1 \end{pmatrix}$.
Combining this, \eqref{proof:b2xi}, and \eqref{proof:w2xi} with \cref{t:max:d} implies that $\normmm{\cW_1(\bbM_{K}) \cW_2(\Xi) } \leq L$ and $\| \cW_1 (\bbM_{K}) \cB_2(\Xi) \|_\infty \leq 2 \norm{\cB_2(\Xi)}_\infty = 2 \norm{\fy}_\infty $. Together with \eqref{eq:tphiinfty} this completes the proof of \eqref{item:dnn:intl1:5}. 

It remains to prove \eqref{item:dnn:intl1:6}. Observe that \cref{prop:dnn:l1norm} and \cref{prop:affcomp} show for all $x \in \R^d$, $k \in \{1,2, \ldots, K \}$ that $(\cR_\fr(\Psi_k))(x) = \| x- \fx_k \|_1$.
This, \cref{prop:parallelization}, and \cref{prop:concat} imply that for all $x \in \R^d$ we have that
\begin{equation} 
 \bigl(\cR_\fr \! \left( \bfP_{K}(\Psi_1, \Psi_2, \ldots, \Psi_{K}) \bullet \fT_{K, d} \right) \bigr) (x) 
= \bigl( \| x- \fx_1 \|_1, \| x- \fx_2 \|_1, \ldots, \| x - \fx_K \|_1 \bigr).
\end{equation}
Combining this and \cref{prop:affcomp} proves for all $x \in \R^d$ that
\begin{equation}
    \begin{split}
          (\cR_\fr(\Xi))(x) &= \bigl(\cR_\fr  \left( \bbA_{-L \id_{K}, \fy} \bullet \bfP_{K}(\Psi_1, \Psi_2, \ldots, \Psi_{K}) \bullet \fT_{K, d} \right) \bigr) (x) \\
    &= \bigl( \fy_1 - L \| x- \fx_1 \|_1, \fy_2 - L \| x - \fx_2 \|_1, \ldots, \fy_K - L \| x- \fx_K \|_1 \bigr).
    \end{split}
\end{equation}
This, \cref{prop:comp}, and \cref{Prop:max_d} establish \eqref{item:dnn:intl1:6}. The proof of \cref{dnn:intl1} is thus complete.
\end{proof}

\section{Analysis of the approximation error}
\label{sec:approx}
 In this section we show how Lipschitz continuous functions defined on a hypercube $[a,b]^d \subseteq \R^d$ can be approximated by DNNs with respect to the uniform norm. These results are elementary and we only include the detailed proofs for completeness. First, in \cref{subsection:1dim} we consider the case $d=1$. In this particular case a neural network with a single hidden layer with $K \in \N$ neurons is sufficient in order for the approximation error to converge to zero with a rate $\cO(K^{-1})$ (cf.\ \cref{lem:1dapprox} below). The construction relies on well-known and elementary properties of the linear interpolation. Afterwards, in \cref{cor: 1dapproxinf} and \cref{cor:1dapprox} we reformulate this approximation result in terms of the vectorized DNN description. Using the fact that DNNs can be embedded into larger architectures (cf., e.g., \cite[Subsection 2.2.8]{BeckJentzenKuckuck2019arXiv}), we replace the exact values of the parameters by lower bounds.
 
 The main result in \cref{subsec:multidim} is \cref{prop:approximation_error}, which provides an upper estimate for the approximation error in the multidimensional case.  We use as an approximation for a Lipschitz continuous function $f \colon [a,b]^d \to \R$ with Lipschitz constant $L$ the maximum convolution $F(x) = \max\nolimits_{k \in \{1, \ldots, K\}} (f( \fx _k) - L \norm{x- \fx_k}_1)$ for a suitably chosen finite subset $\{\fx_1, \fx_2, \ldots, \fx _K\} \allowbreak \subseteq [a,b]^d$ (cf., e.g., \cite[Lemma 3.1]{BeckJentzenKuckuck2019arXiv}). This function has been implemented as a DNN in \cref{dnn:intl1} above. In \cref{dnn:intp1} we estimate the distance between this approximation and the function $f$ in the uniform norm. Next, in \cref{approx:lip} and \cref{approx:lipuv} we express the results in terms of the vectorized description of DNNs, similarly to \cref{subsection:1dim}.
Finally, \cref{prop:approximation_error} follows from \cref{approx:lipuv} by defining the points $\fx_1, \fx_2, \ldots, \fx_K$ appropriately. The choice of $\fx_1, \fx_2, \ldots, \fx_K$ in the proof of \cref{prop:approximation_error} relies on the covering numbers of certain hypercubes, which we introduce in \cref{def:covering_number} below (cf., e.g., \cite[Definition 3.11]{BeckJentzenKuckuck2019arXiv} or \cite[Definition 3.2]{JentzenWelti2020arxiv}).
Since these covering numbers grow exponentially in the dimension, we obtain a convergence rate of $A^{-1/d}$ with respect to the architecture parameter $A$, and therefore this rate of convergence suffers from the curse of dimensionality. The main improvement in \cref{prop:approximation_error} compared to \cite[Proposition 3.5]{JentzenWelti2020arxiv} is that the length of the employed neural network only increases logarithmically with respect to the parameter $A$.  Finally, in \cref{cor:approxerror:eps} we reformulate \cref{prop:approximation_error} in terms of the number of parameters of the employed DNN. In particular, we show for arbitrary $\varepsilon \in (0,1 ]$ that $\cO( \varepsilon ^{-2d})$ parameters are sufficient to obtain an $\varepsilon$-approximation with respect to the uniform norm.

Recently, in \cite{LuShenYangZhang2020} and \cite{ShenYangZhang2020} faster convergence rates for the approximation error have been obtained. We employ the construction from \cref{dnn:intl1} for simplicity and because we also need clear control over the size of parameters of the DNN. For further results on the approximation error we refer, e.g., to \cite{Barron1993,
Cybenko1989,
Funahashi1989,
HartmanKeelerKowalski1990,
Hornik1991,
HornikStinchcombeWhite1989}. 

\subsection{One-dimensional DNN approximations}
\label{subsection:1dim}

\cfclear
\begin{lemma} \label{lem:1dapprox}
Let $A \in (0, \infty)$, $L \in [0, \infty)$, $a \in \R$, $b \in (a, \infty)$, and let $f \colon [a,b] \to \R$ satisfy for all $x,y \in [a,b]$ that $|f(x)-f(y)| \leq L|x-y|$. Then there exists $\Phi \in \NN$ such that  \cfadd{def:DNNparam} \cfadd{def:realization} \cfadd{def:p-norm}
\begin{enumerate} [(i)]
    \item it holds that $\cH(\Phi) = 1$,
    \item it holds that $\cI(\Phi) = \cO(\Phi)=1$,
    \item it holds that $\cD_1(\Phi) \leq A+2$,
    \item it holds that $\norm{\cT(\Phi)}_\infty \leq \max \{ 1, 2L, \sup\nolimits_{x \in [a,b]} |f(x)|, |a|, |b| \}$,
    and
    \item it holds that 
    \begin{equation}
        \sup\nolimits _{x \in [a,b]} | (\cR_\fr(\Phi))(x)-f(x)| \leq \frac{L(b-a)}{A}
    \end{equation}
\end{enumerate}
\cfload.
\end{lemma}
\begin{proof} [Proof of \cref{lem:1dapprox}]
Throughout this proof let $K \in \N$ be given by $K = \ceil{A}$ \cfload, let $(r_k)_{k \in \{0,1, \ldots, K\}} \subseteq [a,b]$ be given by $\forall\, k \in \{0,1, \ldots, K\} \colon r_k=a+\frac{k(b-a)}{K}$, let $(f_k)_{k \in \{0,1, \ldots, K\}} \subseteq \R$ be given by $\forall\, k \in \{0,1, \ldots, K\} \colon f_k = f(r_k)$, let $(c_k)_{k \in \{0, 1, \ldots, K\}} \subseteq \R$ satisfy for all $k \in \{0,1, \ldots, K\}$ that
\begin{equation}
    c_k = \frac{f_{\min\{k+1, K\}} - f_k}{r_{\min\{k+1, K\}} - r_{\min\{k, K-1 \}}} - \frac{f_k - f_{\max\{k-1, 0\}}}{r_{\max\{k,1\}}-r_{\max\{k-1, 0\}}},
\end{equation}
and let $\Phi \in \left( (\R^{(K+1) \times 1} \times \R^{K+1}) \times (\R^{1 \times (K+1)} \times \R)\right) \subseteq \NN$ be given by
\begin{equation}
    \Phi = \left( \! \left( \! \begin{pmatrix}
    1 \\ 1 \\ \vdots \\ 1
    \end{pmatrix}, \begin{pmatrix}
    -r_0 \\ -r_1 \\ \vdots \\ -r_K
    \end{pmatrix}
    \! \right), \left( \begin{pmatrix}
    c_0 & c_1 & \cdots & c_K
    \end{pmatrix}, \begin{pmatrix}
    f_0
    \end{pmatrix}
    \right) \! \right).
\end{equation}
Observe that it holds that $\cH(\Phi) = \cI(\Phi) = \cO(\Phi)=1$ and $\cD_1(\Phi) = K+1 = \ceil{A}+1 \leq A+2$. Moreover, the facts that $\forall\, k \in \{0,1, \ldots, K\} \colon f_k = f(r_k)$ and $\forall \, x,y \in [a,b] \colon |f(x)-f(y)| \leq L|x-y|$ imply for all $k \in \{0,1, \ldots, K\}$ that
\begin{equation}
    \frac{|f_{\min\{k+1, K\}} - f_k|}{|r_{\min\{k+1, K\}} - r_{\min\{k, K-1 \}}|} \leq L \frac{|r_{\min\{k+1, K\}} - r_k|}{|r_{\min\{k+1, K\}} - r_{\min\{k, K-1 \}}|} \leq L
\end{equation}
and 
\begin{equation}
    \frac{|f_k - f_{\max\{k-1, 0\}}|}{|r_{\max\{k,1\}}-r_{\max\{k-1, 0\}}|} \leq  L\frac{|r_k - r_{\max\{k-1, 0\}}|}{|r_{\max\{k,1\}}-r_{\max\{k-1, 0\}}|} \leq L.
\end{equation}
This shows for all $k \in \{0,1, \ldots, K\}$ that $|c_k| \leq 2L$. Combining this with the fact that for all $k \in \{0,1, \ldots, K\}$ it holds that $|r_k| \leq \max \{|a|, |b| \}$ demonstrates that
  \begin{equation}
        \norm{\cT(\Phi)}_\infty \leq \max \{ 1, 2L, \sup\nolimits_{x \in [a,b]} |f(x)|, |a|, |b| \}.
    \end{equation}
Next observe that it holds for all $x \in \R$ that
\begin{equation}
    (\cR_\fr(\Phi))(x) = \sum_{k=0}^K c_k \fr (x - r_k ) + f_0 = \sum_{k=0}^K c_k \max \{ x - r_k , 0\} + f_0.
\end{equation}
This implies for all $k \in \{1,2, \ldots, K\}$, $x \in [r_{k-1}, r_k]$ that $(\cR_\fr(\Phi))(r_0) = f_0$ and
\begin{equation}
    \begin{split}
        (\cR_\fr(\Phi))(x) - (\cR_\fr(\Phi))(r_{k-1}) &= \sum_{j=0}^K c_j ( \max \{ x - r_j , 0\} - \max \{ r_{k-1} - r_j , 0\}) \\
        &= \sum_{j=0}^{k-1} c_j (x-r_{k-1}) = \frac{f_k-f_{k-1}}{r_k-r_{k-1}} (x-r_{k-1}).
    \end{split}
\end{equation}
Combining this and induction establishes for all $k \in \{1,2, \ldots, K\}$, $x \in [r_{k-1}, r_k]$ that $(\cR_\fr(\Phi))(x) \allowbreak = f_{k-1} + \frac{f_k -f_{k -1}}{r_k -r_{k -1}} (x-r_{k -1})$. Hence, $\cR_\fr(\Phi)$ is linear on each interval $[r_{k-1}, r_k]$, $k \in \{1,2, \ldots, K\}$, and satisfies for all $k \in \{0,1, \ldots, K \}$ that $ (\cR_\fr(\Phi))(r_k) = f_k = f(r_k)$.  This and the fact that $f$ is Lipschitz continuous with Lipschitz constant $L$ imply that
  \begin{equation}
        \sup\nolimits _{x \in [a,b]} \abs{ (\cR_\fr(\Phi))(x)-f(x) } \leq L \max\nolimits_{k \in \{1, \ldots, K \}} (r_k-r_{k-1}) = \frac{L(b-a)}{K} \leq \frac{L(b-a)}{A}.
    \end{equation}
    The proof of \cref{lem:1dapprox} is thus complete.
    \end{proof}
    
 \cfclear   
\begin{cor} \label{cor: 1dapproxinf}
Let $A \in (0, \infty)$, $L \in [0, \infty)$, $a \in \R$, $b \in (a, \infty)$, $\bfd , \bfL \in \N$, $\bfl = (\bfl_0, \bfl_1, \ldots,  \bfl_\bfL) \in \N^{\bfL+1}$ satisfy $\bfL \geq 2$, $\bfl_0=\bfl_\bfL = 1$, $\bfl_1 \geq A+2$, and $\bfd \geq \sum_{i=1}^\bfL \bfl_i (\bfl_{i-1}+1)$, assume for all $i \in \{2,3, \ldots, \bfL-1\}$ that $\bfl_i \geq 2$, and let $f \colon [a,b] \to \R$ satisfy for all $x,y \in [a,b]$ that $|f(x)-f(y)| \leq L|x-y|$. Then there exists $\vartheta \in \R^\bfd$ such that $\norm{\vartheta}_\infty \leq \max \{ 1, 2L, \sup\nolimits_{x \in [a,b]}  |f(x)|, |a|, |b| \}$ and \cfadd{def:p-norm} 
\begin{equation}
      \sup\nolimits _{x \in [a,b]} | \scrN_{-\infty, \infty} ^{\vartheta, \bfl}(x)-f(x)| \leq \frac{L(b-a)}{A}
\end{equation}
\cfload.
\end{cor}

\begin{proof}[Proof of \cref{cor: 1dapproxinf}]
Observe that \cref{lem:1dapprox} ensures that there exists $\Phi \in \NN$ such that \cfadd{def:DNN} \cfadd{def:DNNparam} \cfadd{def:realization}
\begin{enumerate} [(i)]
    \item it holds that $\cH(\Phi) = 1$,
    \item it holds that $\cI(\Phi) = \cO(\Phi) = 1$,
    \item it holds that $\cD_1(\Phi)  \leq A+2 $,
    \item it holds that $\norm{ \cT(\Phi)}_\infty \leq \max \{ 1, 2L,  \sup\nolimits_{x \in [a,b]} |f(x)|, |a|, |b| \}$, and
    \item it holds that
  \begin{equation}
        \sup\nolimits _{x \in [a,b]} | (\cR_\fr(\Phi))(x)-f(x)|  \leq \frac{L(b-a)}{A}
    \end{equation}
\end{enumerate}
\cfload.
Combining this, the facts that $\bfL \geq  2$, $\bfl_0=1, \bfl_\bfL=1$, $\bfl_1 \geq A+2$, and the fact that for all $i \in \{2,3, \ldots, \bfL-1\}$ it holds that $\bfl_i \geq 2 $ with \cite[Lemma 2.30]{BeckJentzenKuckuck2019arXiv} completes the proof of \cref{cor: 1dapproxinf}.
\end{proof}

\cfclear
\begin{cor} \label{cor:1dapprox}
Let $A \in (0, \infty)$, $L \in [0, \infty)$, $a \in \R$, $u \in [-\infty, \infty)$, $b \in (a, \infty)$, $v \in (u, \infty]$, $\bfd , \bfL \in \N$, $\bfl = (\bfl_0, \bfl_1, \ldots,  \bfl_\bfL) \in \N^{\bfL+1}$ satisfy $\bfL \geq 2$, $\bfl_0=\bfl_\bfL = 1$, $\bfl_1 \geq A+2$, and $\bfd \geq \sum_{i=1}^\bfL \bfl_i (\bfl_{i-1}+1)$, assume for all $i \in \{2,3, \ldots, \bfL-1\}$ that $\bfl_i \geq 2$, and let $f \colon [a,b] \to [u,v]$ satisfy for all $x,y \in [a,b]$ that $|f(x)-f(y)| \leq L|x-y|$. Then there exists $\vartheta \in \R^\bfd$ such that $\norm{\vartheta}_\infty \leq \max \{ 1, 2L, \sup\nolimits_{x \in [a,b]} |f(x)|, |a|, |b| \}$ and \cfadd{def:p-norm}
\begin{equation}
      \sup\nolimits _{x \in [a,b]} | \scrN_{u, v} ^{\vartheta, \bfl}(x)-f(x)| \leq \frac{L(b-a)}{A}
\end{equation}
\cfload.
\end{cor}
\begin{proof} [Proof of \cref{cor:1dapprox}]
Observe that \cref{cor: 1dapproxinf} establishes that there exists $\vartheta \in \R^\bfd$ such that
$\norm{\vartheta}_\infty \leq \max \{ 1, 2L, \sup\nolimits_{x \in [a,b]} |f(x)|, |a|, |b| \}$ and
\begin{equation}
      \sup\nolimits _{x \in [a,b]} | \scrN_{-\infty, \infty} ^{\vartheta, \bfl}(x)-f(x)| \leq \frac{L(b-a)}{A}.
\end{equation}
Moreover, the assumption that $ f( [a,b]) \subseteq [u,v]$ implies for all $x \in [a,b]$ that\cfadd{def:clip} $\fc_{u,v}(f(x))=f(x)$ \cfload. Combining this with the fact that for all $x,y \in \R$ it holds that $|\fc_{u,v}(x)-\fc_{u,v}(y)| \leq |x-y|$ demonstrates that
\begin{equation}
    \begin{split}
         \sup\nolimits _{x \in [a,b]} | \scrN_{u, v} ^{\vartheta, \bfl}(x)-f(x)| 
         &=   \sup\nolimits _{x \in [a,b]} | \fc_{u,v}(\scrN_{-\infty, \infty} ^{\vartheta, \bfl}(x))- \fc_{u,v}(f(x))| \\
         &\leq  \sup\nolimits _{x \in [a,b]} | \scrN_{-\infty, \infty} ^{\vartheta, \bfl}(x)-f(x)| \leq \frac{L(b-a)}{A}.
    \end{split}
\end{equation}
The proof of \cref{cor:1dapprox} is thus complete. 
\end{proof}

\subsection{Multidimensional DNN approximations}
\label{subsec:multidim}

\cfclear
\begin{prop} \label{dnn:intp1}
Let $d  \in \N$, $L \in [0, \infty)$, $K \in \{2,3,4, \dots \}$, let $E \subseteq \R^d$ be a set, let $\fx_1, \fx_2, \ldots, \fx_K \in E$, let $f \colon E \to \R$ satisfy for all $x, y  \in E$ that $|f(x)-f(y)| \leq L \norm{x-y}_1$, 
let $\fy \in \R^{K}$ be given by $\fy = (f(\fx_1), f(\fx_2), \ldots, f(\fx_K))$,
 and let $\Phi \in \NN$ satisfy \cfadd{def:concat} \cfadd{def:dnn:l1norm} \cfadd{def:composition} \cfadd{def:DNN:aff} \cfadd{def:parallelization} \cfadd{def:max_d} \cfadd{def:p-norm}
\begin{equation}
    \Phi = \bbM_{K} \bullet \bbA_{-L \id_{K}, \fy} \bullet \bfP_{K}  \bigl( \bbL_d \bullet \bbA_{\id _d, -\fx_1}, \bbL_d \bullet \bbA_{\id _d, -\fx_2}, \ldots, \bbL_d \bullet \bbA_{\id _d, -\fx_K} \bigr) \bullet \fT_{K, d}
\end{equation}
\cfload. Then \cfadd{def:realization}
    \begin{equation}
        \sup\nolimits_{x \in E} |f(x)- (\cR_\fr(\Phi))(x)| \leq 2L \left[\sup\nolimits_{x \in E} \left( \inf\nolimits_{k \in \{1, 2, \ldots, K \} } \norm{x-\fx_k}_1 \right)\right]
    \end{equation}
\cfout.
\end{prop}
\begin{proof}[Proof of \cref{dnn:intp1}]
Let $F \colon \R^d \to \R$ satisfy for all $x \in \R ^d$ that 
\begin{equation}
       F(x) = \max\nolimits_{k \in \{1, 2, \ldots, K \} } \left( f(\fx_k) - L \norm{x- \fx_k}_1 \right).
\end{equation}
Observe that \cref{dnn:intl1} (applied with $(\fy_k)_{k \in \{1, \ldots, K \}} \with (f(\fx _k))_{k \in \{1, \ldots, K \}}$ in the notation of \cref{dnn:intl1}) ensures for all $x \in E $ that $F(x)= (\cR_\fr(\Phi))(x)$. Combining this and \cite[Lemma 3.1]{BeckJentzenKuckuck2019arXiv} (applied with $\cM \with \{ \fx_1, \fx_2, \ldots, \fx_K \}$, $(E, \delta) \with (E, \delta_1 |_{E \times E})$ in the notation of \cite[Lemma 3.1]{BeckJentzenKuckuck2019arXiv}) completes the proof of \cref{dnn:intp1}.
\end{proof} 

\cfclear
\begin{cor} \label{approx:lip} \cfadd{def:p-norm}
Let $d, \bfd, \bfL \in \N$, $\bfl=(\bfl_0,\bfl_1, \ldots, \bfl_\bfL) \in \N^{\bfL+1}$, $L \in [0, \infty)$, $K \in \{2,3,4, \ldots\}$ satisfy for all $i \in \{2,3, \ldots, \bfL-1\}$ that $\bfL \geq \ceil{ \log _2 K} +2$, $\bfl_0=d$, $\bfl_\bfL=1$, $\bfl_1 \geq 2d K$, $\bfl_i \geq 3 \ceil{ \frac{ K}{2^{i-1}} }$, and $\bfd \geq \sum_{i=1}^\bfL \bfl_i(\bfl_{i-1}+1)$, let $E \subseteq \R^d$ be a set, let $\fx_1, \fx_2, \ldots, \fx_K \in E$, and let $f \colon E \to \R$ satisfy for all $x, y \in E$ that $|f(x)-f(y)| \leq L \norm{x-y}_1$ \cfload. Then there exists $\theta \in \R^{\bfd}$ such that 
\begin{equation}
    \norm{ \theta }_\infty \leq \max \{ 1,L, \max\nolimits_{k \in \{1, 2, \ldots, K\}} \norm{ \fx_k }_\infty, 2 \max\nolimits_{k \in \{1, 2, \ldots, K\}} |f(\fx_k)|\}
\end{equation}
and 
\begin{equation}
    \sup\nolimits_{x \in E}  \bigl| f(x)- \scrN_{-\infty, \infty}^{\theta, \bfl}(x) \bigr| \leq 2L \left[\sup\nolimits_{x \in E} \left( \inf\nolimits_{k  \in \{1,2, \ldots, K \}} \norm{x- \fx_k}_1\right)\right]
\end{equation}
\cfout.
\end{cor}
\begin{proof}[Proof of \cref{approx:lip}]
Note that the assumption that $K \in \{2,3,\ldots \}$ implies for all $i \in \{2,3, \ldots, \bfL-1\}$ that $\bfl_i \geq 3 \ceil{ \frac{ K }{2^{i-1}} } \geq 2$.
Furthermore, observe that \cref{dnn:intl1} and \cref{dnn:intp1} establish that there exists $\Phi \in \NN$ such that \cfadd{def:DNN} \cfadd{def:DNNparam} \cfadd{def:realization}
\begin{enumerate} [(i)]
    \item it holds that $\cH (\Phi) = \ceil{ \log _2 K} + 1$,
    \item it holds that $\cI(\Phi)=d$, $\cO(\Phi)=1$,
    \item it holds that $\cD_1(\Phi) =2d K$,
    \item it holds for all $i \in \{2,3, \ldots, \cL(\Phi)-1\}$ that $\cD_i(\Phi) \leq 3 \ceil{ \frac{K}{2^{i-1}} } $,
    \item it holds that $\norm{ \cT(\Phi)}_\infty \leq \max \{ 1,L, \max_{k \in \{1, 2, \ldots, K\}} \norm{ \fx_k }_\infty, 2 \max_{k \in \{1, 2, \ldots, K\}} |f(\fx_k)|\}$, and
    \item it holds that
$\sup\nolimits_{x \in E} |f(x)- (\cR_\fr(\Phi))(x)| \leq 2L \left[\sup\nolimits_{x \in E} \left( \inf\nolimits_{k \in \{1,2, \ldots, K\}} \norm{x - \fx_k}_1 \right)\right]$
\end{enumerate}
\cfload. Combining this, the fact that $\bfL \geq \ceil{ \log _2 K} +2$, and the fact that for all $i \in \{2,3, \ldots, \bfL-1\}$ it holds that $\bfl_0=d$, $\bfl_\bfL=1$, $\bfl_1 \geq 2d K$, and $\bfl_i \geq 3 \ceil{ \tfrac{ K }{2^{i-1}} } \geq 2 $ with \cite[Lemma 2.30]{BeckJentzenKuckuck2019arXiv} completes the proof of \cref{approx:lip}.
\end{proof}

\cfclear 
\begin{cor} \label{approx:lipuv} \cfadd{def:p-norm} 
Let $d, \bfd, \bfL \in \N$, $\bfl=(\bfl_0,\bfl_1, \ldots, \bfl_\bfL) \in \N^{\bfL+1}$, $L \in [0, \infty)$, $K \in \{2,3,4, \ldots\}$, $u \in [-\infty, \infty)$, $v \in (u, \infty]$ satisfy for all $i \in \{2,3, \ldots, \bfL-1\}$ that $\bfL \geq \ceil{ \log _2 K} +2$, $\bfl_0=d$, $\bfl_\bfL=1$, $\bfl_1 \geq 2d K$, $\bfl_i \geq 3 \ceil{ \frac{ K}{2^{i-1}} }$, and $\bfd \geq \sum_{i=1}^\bfL \bfl_i(\bfl_{i-1}+1)$, let $E \subseteq \R^d$ be a set, let $\fx_1, \fx_2, \ldots, \fx_K \in E$, and let $f \colon E \to ([u,v] \cap \R)$ satisfy for all $x , y \in E$ that $|f(x)-f(y)| \leq L \norm{x-y}_1$ \cfload. Then there exists $\theta \in \R^{\bfd}$ such that 
\begin{equation}
    \norm{ \theta }_\infty \leq \max \{ 1,L, \max\nolimits_{k \in \{1, 2, \ldots, K\}} \norm{ \fx_k }_\infty, 2 \max\nolimits_{k \in \{1, 2, \ldots, K\}} |f(\fx_k)|\}
\end{equation}
and 
\begin{equation}
    \sup\nolimits_{x \in E}  \bigl| f(x)- \scrN_{u , v}^{\theta, \bfl}(x) \bigr| \leq 2L \left[\sup\nolimits_{x \in E} \left( \inf\nolimits_{k  \in \{1,2, \ldots, K \}} \norm{x- \fx _k}_1\right)\right]
\end{equation}
\cfload.
\end{cor}
\begin{proof}[Proof of \cref{approx:lipuv}]
Observe that \cref{approx:lip} implies that there exists $\theta \in \R^{\bfd}$ such that
\begin{equation}
    \norm{ \theta }_\infty \leq \max \{ 1,L, \max\nolimits_{k \in \{1, 2, \ldots, K\}} \norm{ \fx_k }_\infty, 2 \max\nolimits_{k \in \{1, 2, \ldots, K\}} |f(\fx_k)|\}
\end{equation}
and
\begin{equation}
      \sup\nolimits_{x \in E}  \bigl| f(x)- \scrN_{-\infty, \infty}^{\theta, \bfl}(x) \bigr| \leq 2L \left[\sup\nolimits_{x \in E} \left( \inf\nolimits_{k  \in \{1,2, \ldots, K \}} \norm{x- \fx_k}_1\right)\right].
\end{equation}
Moreover, the assumption that $f(E) \subseteq [u,v]$ shows that for all $x \in E$ it holds that\cfadd{def:clip} $f(x)= \fc_{u,v}(f(x))$ \cfout. The fact that for all $x,y \in \R$ it holds that $|\fc_{u,v}(x)-\fc_{u,v}(y)| \leq |x-y|$ hence establishes that
\begin{equation}
    \begin{split}
          &\sup\nolimits_{x \in E} \bigl| f(x)- \scrN_{u,v}^{\theta, \bfl}(x) \bigr| =   \sup\nolimits_{x \in E} \abs{ \fc_{u,v}(f(x))- \fc_{u,v}(\scrN_{\infty,\infty}^{\theta, \bfl}(x)) } \\ 
          &\leq  \sup\nolimits_{x \in E} \bigl| f(x)- \scrN_{-\infty, \infty}^{\theta, \bfl}(x) \bigr| \leq 2L \left[\sup\nolimits_{x \in E} \left( \inf\nolimits_{k  \in \{1,2, \ldots, K \}} \norm{x- \fx_k}_1\right)\right].  
    \end{split}
\end{equation}
The proof of \cref{approx:lipuv} is thus complete.
\end{proof}

\begin{definition}[Covering numbers]
\label{def:covering_number}
Let
$ ( E, \delta ) $ be a metric space and let
$ r \in [ 0, \infty ) $. Then we denote by
$ \cC_{ ( E, \delta ), r } \in \N_0 \cup \{ \infty \} $ the extended real number given by
\begin{equation}
\begin{split}
\cC_{ ( E, \delta ), r }
=
\inf\biggl(
    \biggl\{
        n \in \N_0 \colon
        \biggl[ \exists \, \scrD \subseteq E \colon \biggl(
                \arraycolsep=0pt \begin{array}{c}
        ( | \scrD| \leq n ) \wedge ( \forall \, x \in E \colon \\
            \exists \, y \in \scrD \colon  \delta(x,y) \leq r )
        \end{array}
        \biggr) \biggr]
    \biggr\}
    \cup \{ \infty \}
\biggr).
\end{split}
\end{equation}
\end{definition}

\cfclear
\begin{prop} \cfadd{def:p-norm}
\label{prop:approximation_error}
Let
$ d, \bfd, \bfL \in \N $,
$ A \in ( 0, \infty ) $, $L \in [0, \infty)$, $a \in \R$,
$ b \in ( a, \infty ) $,
$ u \in [ -\infty, \infty ) $,
$ v \in ( u, \infty ] $,
$ \bfl = ( \bfl_0, \bfl_1, \ldots, \bfl_\bfL ) \in \N^{ \bfL + 1 } $,
assume
$ \bfL \geq 1 + (\ceil{ \log_2 \left(\nicefrac{ A  }{ ( 2d ) } \right)} + 1) \indicator{ ( 6^d, \infty )}( A ) $,
$ \bfl_0 = d $,
$ \bfl_1 \geq A \indicator{ ( 6^d, \infty )}( A ) $,
$ \bfl_\bfL = 1 $,
and
$ \bfd \geq \sum_{i=1}^{\bfL} \bfl_i( \bfl_{ i - 1 } + 1 ) $,
assume for all
$ i \in \{ 2, 3, \ldots , \bfL-1\}$
that
$ \bfl_i \geq   3 \ceil{\nicefrac{A}{(2^id)}} \indicator{ ( 6^d, \infty )}( A )$,
and
let
$ f \colon [ a, b ]^d \to ( [ u, v ] \cap \R ) $
satisfy for all
$ x, y \in [ a, b ]^d $
that
$ \abs{ f( x ) - f( y ) } \leq L \norm{ x - y }_{ 1 } $ \cfload.
Then
there exists $ \vartheta \in \R^\bfd $
such that
$ \norm{ \vartheta }_\infty
\leq \max\{ 1, L, |a|, |b|, \allowbreak 2[ \sup_{ x \in [ a, b ]^d } \abs{ f( x ) } ] \} $
and
\begin{equation}
\sup\nolimits_{ x \in [ a, b ]^d }
    \abs{ \scrN^{\vartheta,\bfl}_{u,v}( x ) - f( x ) }
\leq
\frac{ 3 d L ( b - a ) }{ A^{ \nicefrac{1}{d} } }
\end{equation}
\cfout.
\end{prop}
\begin{proof}[Proof of \cref{prop:approximation_error}]
Throughout this proof we assume w.l.o.g.\ that $A > 6^d$ (if $A \leq 6^d$ the assertion follows from \cite[Lemma 3.4]{JentzenWelti2020arxiv}). Let $ \fZ \in \Z $ be given by $\fZ = \floor{\bigl(\tfrac{ A }{ 2d } \bigr)^{ \nicefrac{1}{d} }}$.
Note that it holds for all
$ k \in \N $
that
\begin{equation}
\label{eq:exponential_estimate}
2 k \leq 2 \cdot 2^{ k - 1 } = 2^k.
\end{equation}
This implies that
$ 3^d
= \nicefrac{6^d}{2^d}
\leq \nicefrac{A}{(2d)} $ and therefore
\begin{equation} \label{eq:fnestimate}
    2 \leq \tfrac{2}{3}\bigl(\tfrac{ A }{ 2d } \bigr)^{ \nicefrac{1}{d} } \leq \bigl(\tfrac{ A }{ 2d } \bigr)^{ \nicefrac{1}{d} }-1 < \fZ. 
\end{equation}
Next, let $r \in (0, \infty)$ be given by $r=\nicefrac{d(b-a)}{2\fZ}$\cfadd{def:covering_number}, let $\delta \colon [a,b]^d \times [a,b]^d \to \R$ satisfy for all $x,y \in [a,b]^d$ that $\delta(x,y) = \norm{x-y}_1$, and let $K \in \N \cup \{ \infty \}$ be given by $K = \max(2, \cC_{ ( [ a, b ]^d, \delta ), r })$ \cfload.
Observe that equation \eqref{eq:fnestimate} and item (i)
in \cite[Lemma 3.3]{JentzenWelti2020arxiv}
(applied with $ p \with 1 $
in the notation of \cite[Lemma 3.3]{JentzenWelti2020arxiv}) establish that
\begin{equation}
\begin{split}
K = \max\{ 2, \cC_{ ( [ a, b ]^d, \delta ), r } \}
\leq\max\Bigl\{ 2, \Bigl( \ceil{ \tfrac{ d ( b - a ) }{2r} } \Bigr)^{ \! d }
\Bigr\}
=\max\{ 2, ( \lceil \fZ \rceil )^d \}
=\fZ^d < \infty.
\end{split}
\end{equation}
This implies that
\begin{equation}
\label{eq:fN_upper_bound}
4 \leq 2 d K \leq 2 d \fZ^d
\leq\tfrac{ 2 d A }{2d} = A.
\end{equation}
Combining this and the fact that
$ \bfL \geq 1+ (\ceil{ \log_2 \left(\nicefrac{ A  }{ ( 2d ) } \right)} + 1) \indicator{ ( 6^d, \infty )}( A ) = \ceil{ \log_2 \left(\nicefrac{ A  }{ ( 2d ) } \right)} + 2$
hence proves that
$ \ceil{\log_2 K}
\leq \ceil{\log_2 \left(\nicefrac{A}{(2d)}\right)}
\leq \bfL - 2 $.
This,
\cref{eq:fN_upper_bound},
and the assumptions that
$ \bfl_1 \geq A \indicator{ ( 6^d, \infty )}( A ) = A $
and
$ \forall \, i \in  \{ 2, 3, \ldots, \bfL - 1 \} \colon
 \bfl_i \geq  3 \ceil{\nicefrac{A}{(2^id)}}  \indicator{ ( 6^d, \infty )}( A ) =   3 \ceil{\nicefrac{A}{(2^id)}} $
imply for all
$ i \in \{ 2, 3, \ldots, \bfL-1 \} $ that
\begin{equation}
\label{eq:bfl_assumptions1}
\bfL
\geq \ceil{\log_2 K} + 2,
\quad
\bfl_1 \geq A \geq 2 d K,
\qandq 
\bfl_i \geq 3\ceil{\tfrac{A}{2^i d}}
\geq  3\ceil{\tfrac{ K }{ 2^{i-1} }}.
\end{equation}
Let $\scrD = \{ \fx_1, \fx_2, \ldots, \fx_K \} \subseteq [a,b]^d$ satisfy
\begin{equation} \label{eq:maxdist}
    \sup\nolimits_{x \in [a,b]^d} \left[ \inf\nolimits_{y \in \scrD} \delta(x,y) \right] = \sup\nolimits_{x \in [a,b]^d} \left[ \inf\nolimits_{k  \in \{1,2, \ldots, K \}} \delta (x,  \fx_k) \right] \leq r.
\end{equation}
Observe that \eqref{eq:bfl_assumptions1}, the assumptions that
$ \bfl_0 = d $,
$ \bfl_\bfL = 1 $,
$ \bfd \geq \sum_{i=1}^{\bfL} \bfl_i( \bfl_{ i - 1 } + 1 ) $, and $\forall \, x,y \in [a,b]^d \colon |f(x)-f(y)| \leq L \norm{x-y}_1$, and \cref{approx:lipuv} (applied with $E \with [a,b]^d$ in the notation of \cref{approx:lipuv}) show that there exists
$ \vartheta \in \R^\bfd $
such that
\begin{equation} \label{eq:prop:approx:normtheta}
     \norm{ \vartheta }_{ \infty } \leq
\max \{ 1,L, \max\nolimits_{k \in \{1, 2, \ldots, K\}} \norm{ \fx_k }_\infty, 2 \max\nolimits_{k \in \{1, 2, \ldots, K\}} |f(\fx_k)|\}
\end{equation}
and
\begin{equation}
\label{eq:NN_approximation}
\begin{split}
\sup\nolimits_{ x \in [ a, b ]^d }
    \abs{ \scrN^{\vartheta, \bfl}_{u, v}( x ) - f( x ) }
&\leq
2 L\left[
\sup\nolimits_{ x \in [ a, b ]^d }
 \left(\inf\nolimits_{k  \in \{1, 2, \ldots, K \}}
    \norm{x- \fx_k}_{ 1 } \right)\right] \\
&=2 L\left[\sup\nolimits_{ x \in [ a, b ]^d }
\left( \inf\nolimits_{k  \in \{1, 2, \ldots, K \}}
    \delta ( x, \fx_k ) \right) \right].
    \end{split}
\end{equation}
Note that \eqref{eq:prop:approx:normtheta} implies that
\begin{equation}
\label{eq:norm_estimate}
\norm{\vartheta }_{ \infty }
\leq
\max\{ 1, L, |a|, |b|, 2  \sup\nolimits_{ x \in [ a, b ]^d } \abs{f(x)}  \}.
\end{equation}
Moreover,
\cref{eq:NN_approximation}, \eqref{eq:exponential_estimate}, \eqref{eq:fnestimate}, and \eqref{eq:maxdist} demonstrate that
\begin{equation}
\begin{split}
\sup\nolimits_{ x \in [ a, b ]^d }
    \abs{ \scrN^{\vartheta, \bfl}_{u, v}( x ) - f( x ) }
&\leq 2 L
\bigl[
\sup\nolimits_{ x \in [ a, b ]^d }
\left(  \inf\nolimits_{k  \in \{1, 2, \ldots, K \}}
    \delta ( x, \fx_k )\right)
\bigr]
\leq
2 L r
=
\frac{ d L ( b - a ) }{ \fZ } \\
& \leq
\frac{ d L ( b - a ) }{ \frac{2}{3} \bigl( \frac{ A }{ 2d } \bigr)^{ \nicefrac{1}{d} } }
= \frac{ ( 2 d )^{ \nicefrac{1}{d} } 3 d L ( b - a ) }{ 2 A^{ \nicefrac{1}{d} } }
\leq \frac{ 3 d L ( b - a ) }{ A^{ \nicefrac{1}{d} } }.
\end{split}
\end{equation}
Combining this with
\cref{eq:norm_estimate} completes the proof of \cref{prop:approximation_error}.
\end{proof}

\cfclear
\begin{cor} \cfadd{def:p-norm}
\label{cor:approxerror:eps}
Let
$ d  \in \N $,
$L \in [0, \infty)$, $a \in \R$,
$ b \in ( a, \infty ) $,
and let
$ f \colon [ a, b ]^d \to  \R $
satisfy for all
$ x, y \in [ a, b ]^d $
that
$ \abs{ f( x ) - f( y ) } \leq L \norm{ x - y }_{ 1 } $ \cfload.
Then there exist $\fC = \fC(a,b,L) \in \R$ (not depending on $d$), $C = C(a,b,d, L) \in \R$, and $\Phi = (\Phi_\varepsilon)_{\varepsilon \in (0,1]} \colon (0,1] \to \NN$ such that for all $\varepsilon \in (0,1]$ it holds that
$ \norm{ \cT(\Phi_\varepsilon ) }_\infty
\leq \max\{ 1, L, |a|, |b|, \allowbreak 2[ \sup_{ x \in [ a, b ]^d } \abs{ f( x ) } ] \} $, $\sup\nolimits_{ x \in [ a, b ]^d }
    \abs{ (\cR_\fr(\Phi_\varepsilon))( x ) - f( x ) }
\leq \varepsilon$, $\cH(\Phi_\varepsilon) \leq d( \log _2(\varepsilon^{-1}) + \log_2 (d) + \fC )$,
and $\cP(\Phi_\varepsilon) \leq C \varepsilon^{-2d}$ \cfadd{def:realization}\cfout.
\end{cor}

\begin{proof} [Proof of \cref{cor:approxerror:eps}]
Throughout this proof assume w.l.o.g.\ that $L > 0$, let $\varepsilon \in (0, 1 ]$, let $A \in (0, \infty)$ satisfy $A = \left( \tfrac{3d L (b-a)}{\varepsilon} \right) ^d$, let $\bfL \in \N$ be given by $\bfL = \max \{2 + \ceil{\log_2 ( \nicefrac{A}{(2d)})} , 1 \}$, let $\bfl = (\bfl_0, \bfl_1, \ldots, \bfl_\bfL) \in \N^{\bfL + 1}$ satisfy $\bfl_0=d$, $\bfl_\bfL = 1$, $\bfl_1 = \ceil{A}$, and $\forall \, i \in \{2,3, \ldots, \bfL - 1 \} \colon \bfl_i = 3 \ceil {\nicefrac{A}{(2^i d) } }$, let $\fC \in \R$ satisfy
\begin{equation}
    \fC =  \max \{\log_2 ( 3L(b-a)) + 1, 0 \},
\end{equation}
and let $C \in \R$ be given by
\begin{equation} \label{eq:approxeps:defc}
    C = \tfrac{9}{8} \bigl( 3dL(b-a) \bigr) ^{2d} + (d+ 19) \bigl( 3dL(b-a) \bigr) ^d + d + 11.
\end{equation}\cfload. Observe that $\bfL \geq 1 + \left( \ceil{ \log_2 \left(\nicefrac{ A  }{ (2d) } \right)} + 1 \right) \indicator{ ( 6^d, \infty )}( A )$,
 $\bfl_1 \geq A \indicator{ ( 6^d, \infty )}( A )$, and 
$\forall \, i \in \{2,3, \ldots, \bfL - 1 \} \colon \bfl_i \geq   3 \ceil{\nicefrac{A}{(2^i d)}} \indicator{ ( 6^d, \infty )}( A )$.
Combining this and the facts that $\bfl_0 = d$ and $\bfl_\bfL = 1$ with \cref{cor:dnnvect} and \cref{prop:approximation_error} (applied with $\bfd \with \sum_{i=1}^\bfL \bfl_i ( \bfl_{i-1} + 1)$ in the notation of \cref{prop:approximation_error}) demonstrates that there exists $\Psi \in \left( \bigtimes_{i=1}^{\bfL}  \left(\R^{\bfl_i \times \bfl_{i-1}} \times \R^{\bfl_i} \right) \! \right) \subseteq \NN$ which satisfies $\norm{\cT(\Psi)}_\infty \leq \max\{ 1, L, |a|, |b|, \allowbreak 2[ \sup_{ x \in [ a, b ]^d } \abs{ f( x ) } ] \} $ and
\begin{equation}
    \sup \nolimits_{x \in [a,b]^d }   \abs{ (\cR_\fr(\Psi ))( x ) - f( x ) } \leq \frac{ 3 d L ( b - a ) }{ A^{ \nicefrac{1}{d} } } = \varepsilon.
\end{equation}
Note that the facts that $d \geq 1$ and $\varepsilon \in (0,1]$ imply that
\begin{equation}
    \begin{split}
        \cH(\Psi) &= \bfL - 1 = \max \{ 1 + \ceil{\log_2 ( \nicefrac{A}{(2d)})} , 0 \} =  \max \left\{ \ceil{\log_2 ( \tfrac{A}{d})} , 0 \right\} \\
        &\leq \max \left\{ \ceil{\log_2(A)} , 0 \right\} 
        = \max \left\{ \ceil {d \log_2 \left( \tfrac{3dL(b-a)}{\varepsilon} \right)} , 0 \right\} \\
        &\leq  \max \left\{ d \big( \log_2(\varepsilon^{-1}) + \log_2 (d) + \log_2(3L(b-a)) \bigr) + 1 , 0\right \} \\
        &\leq d( \log _2(\varepsilon^{-1}) + \log_2 (d) + \fC ).
    \end{split}
\end{equation}
To prove the estimate for $\cP(\Psi)$, observe that
\begin{equation} \label{eq:approxeps1}
    \cP(\Psi) = \sum_{i=1}^\bfL \bfl_i (\bfl_{i-1} + 1) = \ceil{A} (d+1) + 3 \ceil{\tfrac{A}{4d}}( \ceil{A}+1) + \sum_{i=3}^{\bfL - 1} \left[ 3 \ceil{\tfrac{A}{2^i d}} \left( 3 \ceil{\tfrac{A}{2^{i-1}d}} + 1 \right) \right] + 3 \ceil{\tfrac{A}{2^{\bfL - 1}d}}+1.
\end{equation}
Next note that $\bfL  \geq 1+\ceil{\log_2 ( \nicefrac{A}{d})} \geq 1 + \log _2 ( \nicefrac{A}{d})$. This, the fact that $d \geq 1$, and the fact that $\forall \, x \in \R \colon \ceil{x} \leq x + 1$ show that
\begin{equation} \label{eq:approxeps2}
    \begin{split}
          \ceil{A} (d+1) + 3 \ceil{\tfrac{A}{4d}}( \ceil{A}+1)  + 3 \ceil{\tfrac{A}{2^{\bfL - 1}d}}+1 
   & \leq (A+1)(d+1) + 3 \left( \tfrac{A}{4d} + 1 \right) (A+2) + 4 \\
    &= \tfrac{3}{4d} A^2 + (d+4+\tfrac{3}{2d})A + d + 11 \\
    &\leq \tfrac{3}{4} A^2 + (d + \tfrac{11}{2}) A + d + 11.
    \end{split}
\end{equation}
Moreover, note that the fact that $\forall \, x \in (0, \infty) \colon \log_2 x = \log_2 ( \nicefrac{x}{2}) + 1 \leq \nicefrac{x}{2} + 1$ implies that
\begin{equation}
    \bfL \leq \max \left\{ 2 + \log_2 ( \tfrac{A}{d}), 1 \right\} \leq 3 + \tfrac{A}{2d} \leq 3 + \tfrac{A}{2}.
\end{equation}
This demonstrates that
\begin{equation} \label{eq:approxeps3}
    \begin{split}
        \sum_{i=3}^{\bfL - 1}  \left[3 \ceil{\tfrac{A}{2^i d}} \left( 3 \ceil{\tfrac{A}{2^{i-1}d}} + 1 \right) \right]
        &\leq 3 \sum_{i=3}^{\bfL - 1} \left( \tfrac{A}{2^i d } + 1 \right) \left( \tfrac{3A}{2^{i-1}d} + 4 \right) \\
        &= \tfrac{9A^2 }{d^2 } \sum_{i=3}^{\bfL - 1} 2^{1-2i} + \tfrac{12 A }{d} \sum_{i=3}^{\bfL - 1} 2^{-i} + \tfrac{9A}{d} \sum_{i=3}^{\bfL - 1} 2^{1-i} + 12(\bfL - 3)  \\
        &\leq \tfrac{3}{8} A^2 + 3A + \tfrac{9}{2} A + 6A = \tfrac{3}{8} A^2 + \tfrac{27}{2} A.
    \end{split}
\end{equation}
  Combining \eqref{eq:approxeps:defc}, \eqref{eq:approxeps1}, \eqref{eq:approxeps2}, and \eqref{eq:approxeps3} with the fact that $\varepsilon \in (0, 1 ]$ shows that
  \begin{equation}
  \begin{split}
          \cP(\Psi) &\leq \tfrac{9}{8} A^2 + (d+19) A + d + 11 \\
          &= \tfrac{9}{8} \bigl(3dL(b - a ) \bigr) ^{2d} \varepsilon ^{-2d} + (d+19) \bigl( 3 d L ( b-a) \bigr) ^d \varepsilon ^{- d} + d + 11 \\
          &\leq \left[ \tfrac{9}{8} \bigl(3dL(b - a ) \bigr) ^{2d} + (d+19) \bigl( 3 d L ( b-a) \bigr) ^d + d + 11 \right] \varepsilon ^{-2d} = C \varepsilon^{-2d}.
  \end{split}
  \end{equation}
The proof of \cref{cor:approxerror:eps} is thus complete.
\end{proof}

\section{Analysis of the generalization error}
\label{sec:general}
In this section, we provide in \cref{cor:generalization_error} below a uniform upper bound for the generalization error over all artificial neural networks with a fixed architecture and coefficients in a given hypercube. More specifically, \cref{cor:generalization_error} shows that the $L^p$-distance between the true risk $\bfR$ and the empirical risk $\cR$ based on $M$ i.i.d.\ samples is of order $\cO(\log M / \sqrt{M})$ as $M \to \infty$, uniformly over all DNNs with a fixed architecture and parameters in a given interval $[-\beta, \beta]$. The conclusion of \cref{cor:generalization_error} improves on \cite[Corollary 4.15]{JentzenWelti2020arxiv} by reducing the constant factor on the right-hand side.

 In \cref{Subsection:montecarlo} we prove some elementary estimates for sums of independent random variables. The results in this subsection are well-known, in particular, \cref{lem:l2montecarlo} is a reformulation of, e.g., \cite[Lemma 2.3]{GrohsHornungJentzenVonWurstemberger2018arXiv} and \cref{lem:lpmontecarlo} is an immediate consequence of \cite[Theorem 2.1]{Rio2009}. It should be noted that \cref{cor:lpmontecarlo} improves on \cite[Corollary 4.5]{JentzenWelti2020arxiv} by removing a factor of $2$ on the right-hand side. The simple estimates in \cref{lem:vartrivial} and \cref{cor:lpvar} are also well-known, and we only include their proofs for completeness.

In \cref{subsection:generalization} we combine these results to obtain an upper estimate for the generalization error. The proof of \cref{cor:generalization_error} is very similar to the proofs in \cite[Section 4]{JentzenWelti2020arxiv}. The main difference is that we apply the improved Monte Carlo inequality from \cref{cor:lpmontecarlo} instead of \cite[Corollary 4.15]{JentzenWelti2020arxiv}. Most of the proofs are therefore omitted.

\subsection{General Monte Carlo estimates}
\label{Subsection:montecarlo}
\begin{lemma} \label{lem:l2montecarlo}
Let $M \in \N$, let $(\Omega, \cF, \P)$ be a probability space, let $X_j \colon \Omega \to \R$, $j \in \{1,2, \ldots, M \}$, be i.i.d.\ random variables, and assume $ \E  [ |X_1| ] < \infty$. Then
\begin{equation}
    \E \! \left[ \left( \smallsum_{j=1}^M X_j - \smallsum_{j=1}^M \E [X_j] \right) ^2\right] = M\, \E \bigl[ \abs{ X_1 - \E [X_1] } ^2 \bigr].
\end{equation}
\end{lemma}
\begin{proof} [Proof of \cref{lem:l2montecarlo}]
The assumption that $X_j$, $j \in \{1,2, \ldots, M \}$, are independent random variables and $\forall \, j \in \{1,2, \ldots, M\} \colon \E [ |X_j| ] < \infty$ ensures that for all $i,j \in \{1, 2,\ldots, M\}$ with $i \not= j$ it holds that
\begin{equation}
    \E \bigl[ (X_i-\E [X_i])(X_j-\E [X_j]) \bigr] = \E \bigl[ X_i-\E [X_i] \bigr]  \E \bigl[ X_j - \E [X_j] \bigr] = 0.
\end{equation}
This implies that
\begin{equation}
    \begin{split}
         \E \! \left[ \left( \smallsum_{j=1}^M X_j - \smallsum_{j=1}^M \E [X_j] \right) ^2\right] 
         &= \smallsum_{i,j=1}^M   \E \bigl[ (X_i-\E [X_i])(X_j-\E [X_j]) \bigr] \\
         &= \smallsum_{j=1}^M \E \bigl[ \abs{ X_j - \E [X_j] } ^2 \bigr] =  M\, \E \bigl[ \abs{ X_1 - \E [X_1] } ^2 \bigr].
    \end{split}
\end{equation}
   The proof of \cref{lem:l2montecarlo} is thus complete.
\end{proof}

\begin{lemma}\label{lem:lpmontecarlo}
Let $M \in \N$, $p \in (2, \infty)$, let $(\Omega, \cF, \P)$ be a probability space, let $X_j \colon \Omega \to \R, \ j \in \{1,2, \ldots, M \}$, be i.i.d.\ random variables, and assume that $ \E [ |X_1| ] < \infty$. Then
\begin{equation} \label{eq:lpmontecarlo}
    \left(   \E  \! \left[\abs{\smallsum_{j=1}^M X_j - \smallsum_{j=1}^M \E [X_j] } ^p\right] \right)^{\! \nicefrac{1}{p}} 
    \leq  \sqrt{M(p-1)} \left(\E \bigl[ \abs{ X_1 - \E [X_1] }  ^p \bigr] \right)^{\nicefrac{1}{p}}
\end{equation}
\end{lemma}
\begin{proof}[Proof of \cref{lem:lpmontecarlo}]
Note that the fact that $X_j \colon \Omega \to \R$, $j \in \{1,2, \ldots, M \}$, are i.i.d.\ random variables and the fact that $\forall \, j \in \{1,2, \ldots, M \} \colon \E [|X_j|] < \infty$ ensures that for all $n \in \{1,2, \ldots, M \}$ it holds that
\begin{equation}
    \E \! \left[ X_n - \E [X_n] \mid \smallsum_{j=1}^{n-1} (X_j - \E [X_j])\right] = \E \bigl[X_n - \E [X_n] \bigr]=0.
\end{equation}
Rio \cite[Theorem 2.1]{Rio2009} (applied with $n \with M$, $(X_k)_{k \in \{1, \ldots, n\}} \with (X_j - \E [X_j])_{j \in \{1, \ldots M\}}$, $(S_k)_{k \in \{0, \ldots, n\}} \allowbreak \with (\sum_{j=1}^k (X_j - \E [X_j]))_{k \in \{0, \dots, M\}}$ in the notation of \cite[Theorem 2.1]{Rio2009}) therefore demonstrates that
\begin{equation}
\begin{split}
      \left(   \E \!  \left[\abs{ \smallsum_{j=1}^M X_j - \smallsum_{j=1}^M \E [X_j] } ^p\right] \right)^{\! \nicefrac{2}{p}} 
      &\leq (p-1) \smallsum_{j=1}^M  \left(\E \bigl[ \abs{ X_j - \E [X_j] } ^p \bigr] \right)^{\nicefrac{2}{p}} \\
      &= M(p-1) \left(\E \bigl[ \abs{ X_1 - \E [X_1] } ^p \bigr] \right)^{\nicefrac{2}{p}}.
      \end{split}
\end{equation} 
This establishes \eqref{eq:lpmontecarlo} and thus completes the proof of \cref{lem:lpmontecarlo}.
\end{proof}

\begin{cor} \label{cor:lpmontecarlo}
Let $M \in \N$, $p \in [2, \infty)$, let $(\Omega, \cF, \P)$ be a probability space, let $X_j \colon \Omega \to \R$, $j \in \{1,2, \ldots, M \}$, be i.i.d.\ random variables, and assume that $\E [|X_1| ] < \infty$. Then
\begin{equation}
    \left( \E \! \left[ \abs{ \frac{1}{M} \left[\sum_{j=1}^M X_j \right] - \frac{1}{M} \left[ \sum_{j=1}^M \E [X_j] \right] }^p \, \right] \right)^{\! \! \nicefrac{1}{p}} 
    \leq \sqrt{\frac{p-1}{M}} \left(\E \bigl[ \abs{ X_1 - \E [X_1] } ^p \bigr] \right)^{\nicefrac{1}{p}}.
\end{equation}
\end{cor}
\begin{proof}[Proof of \cref{cor:lpmontecarlo}]
Observe that \cref{lem:l2montecarlo} and \cref{lem:lpmontecarlo} show that
\begin{equation}
    \begin{split}
          \left( \E \! \left[ \abs{ \frac{1}{M} \left[ \sum_{j=1}^M X_j \right] - \frac{1}{M} \left[ \sum_{j=1}^M \E [X_j] \right] }^p \right] \right)^{\! \!  \nicefrac{1}{p}} 
          &= \frac{1}{M} \left(   \E \!  \left[\abs{ \sum_{j=1}^M X_j - \sum_{j=1}^M \E [X_j] } ^p\right] \right)^{\! \! \nicefrac{1}{p}} \\
          &\leq  \sqrt{\frac{p-1}{M}} \left( \E \bigl[ \abs{ X_1 - \E [X_1] } ^p \bigr] \right)^{\nicefrac{1}{p}}.
    \end{split}
\end{equation}
The proof of \cref{cor:lpmontecarlo} is thus complete.
\end{proof}
\begin{lemma} \label{lem:vartrivial}
Let $p \in [2, \infty)$, let $(\Omega, \cF, \P)$ be a probability space, and let $Y \colon \Omega \to [0,1]$ be a random variable. Then
\begin{equation}
    \E  \bigl[ | Y-\E[Y]|^p \bigr] \leq \tfrac{1}{4}.
\end{equation}
\end{lemma}
\begin{proof}[Proof of \cref{lem:vartrivial}]
The assumption that $\forall \, \omega \in \Omega \colon 0 \leq Y(\omega) \leq 1$ ensures that $\E[Y^2] \leq \E[Y]$. Combining this and the fact that $\forall \, u \in \R \colon u \leq u^2+\tfrac{1}{4}$ shows that
\begin{equation} \label{eq:vartrivial:1}
      \E \bigl[ | Y-\E[Y]|^2 \bigr] = \E[Y^2] - \left( \E[Y] \right)^2 \leq  \E[Y]- \left( \E[Y] \right)^2 \leq \tfrac{1}{4}.
\end{equation}
Furthermore, note that the assumption that $p \geq 2$ implies that for all $v \in [0,1]$ it holds that $v^p \leq v^2$. The facts that $0 \leq Y \leq 1$ and $0 \leq \E[Y] \leq 1$ hence prove that $|Y-\E[Y]|^p \leq |Y-\E [Y]|^2$. Combining this and \eqref{eq:vartrivial:1} demonstrates that
\begin{equation}
     \E \bigl[ | Y-\E[Y]|^p \bigr] \leq   \E \bigl[ | Y-\E[Y]|^2 \bigr] \leq \tfrac{1}{4}.
\end{equation}
The proof of \cref{lem:vartrivial} is thus complete.
\end{proof}
\begin{cor} \label{cor:lpvar}
Let $D \in (0, \infty)$, $p \in [2, \infty)$, let $(\Omega, \cF, \P)$ be a probability space, and let $Y \colon \Omega \to [0, D]$ be a random variable. Then
\begin{equation} \label{eq:lpvar}
    \left( \E \bigl[ |Y-\E[Y]|^p \bigr] \right)^{\nicefrac{1}{p}} \leq 4^{-\nicefrac{1}{p}} D.
\end{equation}
\end{cor}
\begin{proof}[Proof of \cref{cor:lpvar}]
Note that \cref{lem:vartrivial} (applied with $Y \with \bigl[ \Omega \ni \omega \mapsto D^{-1} Y(\omega) \in [0,1]\bigr]$ in the notation of \cref{lem:vartrivial}) establishes \eqref{eq:lpvar}. The proof of \cref{cor:lpvar} is thus complete.
\end{proof}

\subsection{Upper bounds for the generalization error}
\label{subsection:generalization}

\cfclear
\begin{lemma}
\label{lem:abstract_generalization_error}
Let
$ M \in \N $,
$ p \in [ 2, \infty ) $,
$ L , D \in ( 0, \infty ) $,
let
$ ( E, \delta ) $
be a separable metric space,
assume
$ E \neq \emptyset $,
let
$ ( \Omega, \cF, \P ) $
be a probability space,
for every $x \in E$ let
$ (X_{ x, j }, Y_j) \colon \Omega \to \R \times \R  $,
$ j \in \{ 1, 2, \ldots, M \} $,
be i.i.d.\ random variables,
assume for all
$ x, y \in E $,
$ j \in \{ 1, 2, \ldots, M \} $
that
$ \abs{ X_{ x, j } - Y_j } \leq D $
and
$ \abs{ X_{ x, j } - X_{ y, j } } \leq L \delta( x, y ) $,
let
$ \bfR \colon E \to [ 0, \infty ) $
satisfy for all
$ x \in E $
that
$ \bfR( x ) = \E \bigl[ \abs{ X_{ x, 1 } - Y_1 } ^2 \bigr] $,
and
let
$ \cR \colon E \times \Omega \to [ 0, \infty ) $
satisfy for all
$ x \in E $,
$ \omega \in \Omega $
that
\begin{equation}
\cR( x, \omega )
= \frac{1}{M} \left[
\sum_{j=1}^M  \abs{ X_{ x, j }( \omega ) - Y_j( \omega ) }^2 \right].
\end{equation}
Then \cfadd{def:covering_number}
\begin{enumerate}[(i)]
\item
\label{item:lem:abstract_generalization_error:1}
the function
$ \Omega \ni \omega \mapsto
\sup\nolimits_{ x \in E } \abs{ \cR( x, \omega ) - \bfR( x ) } \in [ 0, \infty ] $
is measurable
and
\item
\label{item:lem:abstract_generalization_error:2}
it holds for all $C \in (0, \infty)$ that
\begin{equation} \begin{split}
\bigl( \E\bigl[
    \sup\nolimits_{ x \in E }
    \abs{ \cR( x ) - \bfR( x ) } ^p
\bigr] \bigr)^{ \nicefrac{1}{p} }
&\leq
\Bigl( \cC_{ ( E, \delta ), \frac{ C D \sqrt{ p - 1 } }{ 2 L \sqrt{M} } } \Bigr)^{ \! \nicefrac{1}{p} }
\biggl[ \frac{ (2C + 4^{-\nicefrac{1}{p}} ) D ^2 \sqrt{ p - 1 } }{ \sqrt{M} } \biggr]
\\ &\leq \Bigl( \cC_{ ( E, \delta ), \frac{ C D \sqrt{ p - 1 } }{ 2 L \sqrt{M} } } \Bigr)^{ \! \nicefrac{1}{p} }
\biggl[ \frac{ (2C + 1 ) D ^2 \sqrt{ p - 1 } }{ \sqrt{M} } \biggr]
\end{split}
\end{equation}
\end{enumerate}
\cfload.
\end{lemma}

\begin{proof}[Proof of \cref{lem:abstract_generalization_error}]
Throughout this proof
let
$ \cY_{ x, j } \colon \Omega \to \R $,
$ j \in \{ 1, 2, \ldots, M \} $,
$ x \in E $,
satisfy for all
$ x \in E $,
$ j \in \{ 1, 2, \ldots, M \} $
that
$ \cY_{ x, j } = \abs{ X_{ x, j } - Y_j }^2 $.
Observe that the assumption that $\forall \, x \in E,
\, j \in \{ 1, 2, \ldots, M \} \colon
\abs{ X_{ x, j } - Y_j } \leq D $
shows for all
$ x \in E $,
$ j \in \{ 1, 2, \ldots, M \} $
that $0 \leq \cY_{x,j} \leq D^2$. \cref{cor:lpvar} therefore implies for all $x \in E$, $j \in \{1,2, \ldots, M\}$ that
\begin{equation}\label{eq:bound_Lp_centred}
   \bigl( \E\bigl[ |\cY_{ x, j } - \E[ \cY_{ x, j } ] |^p \bigr] \bigr)^{ \nicefrac{1}{p} }
\leq 4^{-\nicefrac{1}{p}} D^2.
\end{equation}
The rest of the proof is identical to the proof of \cite[Lemma 4.13]{JentzenWelti2020arxiv}, the only difference being that we apply the improved Monte Carlo estimate \cref{cor:lpmontecarlo} instead of \cite[Corollary 4.5]{JentzenWelti2020arxiv}.
The proof of \cref{lem:abstract_generalization_error} is thus complete.
\end{proof}

\cfclear
\begin{prop}
\label{prop:generalization_error}
Let $ d, \bfd, M \in \N $, $ L, D, p \in ( 0, \infty ) $, $ \alpha \in \R $, $ \beta \in ( \alpha, \infty ) $, $ \scrD \subseteq \R^d $, let
$ ( \Omega, \cF, \P ) $
be a probability space,
let
$ (X_j, Y_j) \colon \Omega \to \scrD \times \R $,
$ j \in \{ 1, 2, \ldots, M \} $,
be i.i.d.\ random variables,
let
$ H = ( H_\theta )_{ \theta \in [ \alpha, \beta ]^\bfd } \colon
[ \alpha, \beta ]^\bfd \to C( \scrD, \R ) $
be a function,
assume for all
$ \theta, \vartheta \in [ \alpha, \beta ]^\bfd $,
$ j \in \{ 1, 2, \ldots, M \} $,
$ x \in \scrD $
that
$ \abs{ H_\theta( X_j ) - Y_j } \leq D $
and
$ \abs{ H_\theta( x ) - H_\vartheta( x ) }
\leq
L \norm{  \theta - \vartheta }_\infty $,
let
$ \bfR \colon [ \alpha, \beta ]^\bfd \to [ 0, \infty ) $
satisfy for all
$ \theta \in [ \alpha, \beta ]^\bfd $
that
$ \bfR( \theta )
= \E[ \abs{ H_\theta( X_1 ) - Y_1 }^2 ] $,
and
let
$ \cR \colon [ \alpha, \beta ]^\bfd \times \Omega \to [ 0, \infty ) $
satisfy for all
$ \theta \in [ \alpha, \beta ]^\bfd $,
$ \omega \in \Omega $
that \cfadd{def:p-norm}
\begin{equation}
\cR( \theta, \omega )
= \frac{1}{M} \left[ \sum_{j=1}^M
    \abs{ H_\theta( X_j( \omega ) ) - Y_j( \omega ) }^2 \right]
\end{equation}
\cfload. Then 
\begin{enumerate}[(i)]
\item
\label{item:prop:generalization_error:1}
the function
$ \Omega \ni \omega \mapsto
\sup\nolimits_{ \theta \in [ \alpha, \beta ]^\bfd } \abs{ \cR( \theta, \omega ) - \bfR( \theta ) }
\in [ 0, \infty ] $
is measurable
and
\item \label{item:prop:generalization_error:2}
it holds that
\begin{equation}
\bigl( \E\bigl[
    \sup\nolimits_{ \theta \in [ \alpha, \beta ]^\bfd }
    \abs{ \cR( \theta ) - \bfR( \theta ) } ^p \bigr]
\bigr)^{ \nicefrac{1}{p} } 
 \leq \frac{2 D^2
\sqrt{ e \max\{ p, 1, \bfd \ln( 16 M L^2 ( \beta - \alpha )^2  D ^{ -2 } ) \} }
}{ \sqrt{M} }.
\end{equation}
\end{enumerate}
\end{prop}

\begin{proof}[Proof of \cref{prop:generalization_error}]
Analogously to the proof of \cite[Proposition 4.14]{JentzenWelti2020arxiv}, one can show that \cref{lem:abstract_generalization_error} implies that the function
$ \Omega \ni \omega \mapsto
\sup\nolimits_{ \theta \in [ \alpha, \beta ]^\bfd } \abs{ \cR( \theta, \omega ) - \bfR( \theta ) }
\in [ 0, \infty ] $
is measurable
and
\begin{equation}
\begin{split}
    & \bigl( \E\bigl[
    \sup\nolimits_{ \theta \in [ \alpha, \beta ]^\bfd }
    \abs{ \cR( \theta ) - \bfR( \theta ) }^p \bigr] \bigr)^{ \nicefrac{1}{p} } \\&
    \leq  \inf_{ C \in ( 0, \infty ) }
\left[ \frac{ (2C + 1 ) D^2
\sqrt{ e \max\{ 1, p, \bfd \ln( 4 M L^2 ( \beta - \alpha )^2 ( C D )^{ -2 } ) \} }
}{ \sqrt{M} } \right] 
\\& \leq  \frac{2 D^2
\sqrt{ e \max\{ p, 1, \bfd \ln( 16 M L^2 ( \beta - \alpha )^2  D ^{ -2 } ) \} }
}{ \sqrt{M} }.
\end{split} 
\end{equation}
The proof of \cref{prop:generalization_error} is thus complete.
\end{proof}

\cfclear
\begin{cor}
\label{cor:generalization_error}
Let
$ d, \bfd, \bfL, M \in \N $,
$ \beta , c \in [ 1, \infty ) $, $p \in (0, \infty)$,
$ u \in \R $,
$ v \in [ u + 1, \infty ) $,
$ \bfl = ( \bfl_0, \bfl_1, \ldots, \bfl_\bfL ) \in \N^{ \bfL + 1 } $,
$ \scrD \subseteq [ -c, c ]^d $,
assume
$ \bfl_0 = d $,
$ \bfl_\bfL = 1 $,
and
$ \bfd \geq \sum_{i=1}^{\bfL} \bfl_i( \bfl_{ i - 1 } + 1 ) $,
let
$ ( \Omega, \cF, \P ) $
be a probability space,
let
$ (X_j,Y_j) \colon \Omega \to \scrD \times [u,v] $, $j \in \{1,2, \ldots, M\}$, be i.i.d.\ random variables,
let
$ \bfR \colon [ -\beta, \beta ]^\bfd \to [ 0, \infty ) $
satisfy for all
$ \theta \in [ -\beta, \beta ]^\bfd $
that
$ \bfR( \theta )
= \E \bigl[ \abs{ \scrN^{\theta, \bfl}_{u,v}( X_1 ) - Y_1 }^2 \bigr] $,
and let
$ \cR \colon [ -\beta, \beta ]^\bfd \times \Omega \to [ 0, \infty ) $
satisfy for all
$ \theta \in [ -\beta, \beta ]^\bfd $,
$ \omega \in \Omega $
that 
\begin{equation}
\cR( \theta, \omega )
= \frac{1}{M} \left[ \sum_{j=1}^M
    \abs{ \scrN^{\theta,\bfl}_{u, v}( X_j( \omega ) ) - Y_j( \omega ) }^2 \right]
\end{equation}
\cfload.
Then
\begin{enumerate}[(i)]
\item
\label{item:cor:generalization_error:1}
the function
$ \Omega \ni \omega \mapsto
\sup\nolimits_{ \theta \in [ -\beta, \beta ]^\bfd } \abs{ \cR( \theta, \omega ) - \bfR( \theta ) }
\in [ 0, \infty ] $
is measurable
and
\item
\label{item:cor:generalization_error:2}
it holds that \cfadd{def:p-norm}
\begin{equation}
\begin{split}
& \bigl(
\E\bigl[
    \sup\nolimits_{ \theta \in [ -\beta, \beta ]^\bfd }
    \abs{ \cR( \theta ) - \bfR( \theta ) }^p
\bigr]
\bigr)^{ \nicefrac{1}{p} }
\\ &
\leq \frac{
    5 ( v - u )^2
    \bfL (\norm{\bfl}_{ \infty } + 1 )
    \sqrt{
        \max\{
            p,
            \ln( 8 ( M c )^{ \nicefrac{1}{ \bfL } } ( \norm{\bfl}_{ \infty } + 1 ) \beta )
        \}
    }
}{ \sqrt{M} }
\\ &
\leq
\frac{
    5 ( v - u )^2
    \bfL ( \norm{\bfl}_{ \infty } + 1 )^{\nicefrac{3}{2}}
    \max\{
        p,
        \ln( 4 M \beta c )
    \}
}{ \sqrt{M} }
\end{split}
\end{equation}
\end{enumerate}
\cfout.
\end{cor}
\begin{proof}[Proof of \cref{cor:generalization_error}]
Throughout this proof let
$ L \in ( 0, \infty ) $
be given by $ L = c \bfL ( \norm{\bfl}_\infty + 1 )^\bfL \beta^{ \bfL - 1 } $. Using the same arguments as in the proof of \cite[Corollary 4.15]{JentzenWelti2020arxiv}, we obtain from \cite[Corollary 2.37]{BeckJentzenKuckuck2019arXiv} and \cref{prop:generalization_error} that the function $ \Omega \ni \omega \mapsto
\sup\nolimits_{ \theta \in [ -\beta, \beta ]^\bfd } \abs{ \cR( \theta, \omega ) - \bfR( \theta ) }
\in [ 0, \infty ] $
is measurable
and 
\begin{equation}
         \bigl( \E\bigl[
    \sup\nolimits_{ \theta \in [ -\beta, \beta ]^\bfd }
    \abs{ \cR( \theta ) - \bfR( \theta ) }^p
\bigr] \bigr)^{ \nicefrac{1}{p} }
\leq
\frac{ 2 ( v - u )^2
\sqrt{ e \max\{ 1, p, \bfL (\norm{\bfl}_{ \infty } + 1 )^2 \ln( 64 M L^2 \beta^2 ) \} }
}{ \sqrt{M} }.
\end{equation}
This, the fact that 
$ 64 \bfL^2 \leq 2^6  2^{ 2 ( \bfL - 1 ) } = 2^{ 4 + 2 \bfL } \leq 2^{ 4 \bfL + 2 \bfL } = 2^{ 6 \bfL } $, 
and the facts that $ \sqrt{2e} < \tfrac{5}{2} $, $ \beta \geq 1 $, $ \bfL \geq 1 $, $ M \geq 1 $, and $ c \geq 1 $ show that
\begin{equation}
\label{eq:generalization_error_NN1}
\begin{split}
& \bigl(
\E\bigl[
    \sup\nolimits_{ \theta \in [ -\beta, \beta ]^\bfd }
    \abs{ \cR( \theta ) - \bfR( \theta ) }^p
\bigr]
\bigr)^{ \nicefrac{1}{p} }
\\ &
\leq
 \frac{ 2 ( v - u )^2
\sqrt{
    e \max\{
        p,
        \bfL ( \norm{ \bfl }_{ \infty } + 1 )^2
        \ln( 64 M c^2 \bfL^2 (\norm{\bfl}_{ \infty } + 1 )^{ 2 \bfL } \beta^{ 2 \bfL } )
    \}
    }
}{ \sqrt{M} }
\\ &
\leq
\frac{ 2  ( v - u )^2
\sqrt{
    e
    \max\{
        p,
        2 \bfL^2 (\norm{\bfl}_{ \infty } + 1 )^2
        \ln( [ 2^{ 6 \bfL } M c^2 (\norm{\bfl}_{ \infty } + 1 )^{ 2 \bfL } \beta^{ 2 \bfL } ]^{ \nicefrac{1}{ ( 2 \bfL ) } } )
    \}
    }
}{ \sqrt{M} }
\\ &
\leq
\frac{ 2 ( v - u )^2
\sqrt{
    e
    \max\{
        p,
        2 \bfL^2 (\norm{\bfl}_{ \infty } + 1 )^2
        \ln( 2^3 ( M c^2 )^{ \nicefrac{1}{ ( 2 \bfL ) } } ( \norm{\bfl}_{ \infty } + 1 ) \beta )
    \}
    }
}{ \sqrt{M} }
\\ &
\leq
\frac{ 5 ( v - u )^2
\bfL (\norm{\bfl}_{ \infty } + 1 )
\sqrt{
    \max\{
        p,
        \ln( 8 ( M c )^{ \nicefrac{1}{ \bfL } } (\norm{\bfl}_{ \infty } + 1 ) \beta )
    \}
    }
}{ \sqrt{M} }.
\end{split}
\end{equation}
In addition,
the facts that $\norm{\bfl}_{ \infty } \geq 1$ and $ \forall \, n \in \N \colon n \leq 2^{ n - 1 }$ imply that
\begin{equation}
8 (\norm{\bfl}_{ \infty } + 1 )
\leq 2^3  2^{ (\norm{\bfl}_{ \infty } + 1 ) - 1 }
=  2^{ \norm{\bfl}_{ \infty } +3 }
\leq 2 ^{2\norm{\bfl}_\infty+2} = 4^{\norm{\bfl}_\infty+1}.
\end{equation}
This demonstrates that
\begin{equation}
\begin{split}
&
\frac{ 5 ( v - u )^2
\bfL (\norm{\bfl}_{ \infty } + 1 )
\sqrt{
    \max\{
        p,
        \ln( 8 ( M c )^{ \nicefrac{1}{ \bfL } } (\norm{\bfl}_{ \infty } + 1 ) \beta )
    \}
    }
}{ \sqrt{M} }
\\ &
\leq
\frac{ 5 ( v - u )^2
\bfL (\norm{\bfl}_{ \infty } + 1 )
\sqrt{
    \max\{
        p, (\norm{\bfl}_{ \infty } + 1 )
        \ln( [ 4^{ (\norm{\bfl}_{ \infty } + 1 ) } ( M c )^{ \nicefrac{1}{ \bfL } } \beta ]^{ \nicefrac{1}{ ( \norm{\bfl}_{ \infty } + 1 ) } } )
    \} }
}{ \sqrt{M} }
\\ 
& \leq \frac{
    5 ( v - u )^2
    \bfL (\norm{\bfl}_{ \infty } + 1 )^{\nicefrac{3}{2}}
    \max\{
        p,
        \ln( 4 M \beta c )
    \}
}{ \sqrt{M} }.
\end{split}
\end{equation}
Combining this and \eqref{eq:generalization_error_NN1} completes the proof of \cref{cor:generalization_error}.
\end{proof}

\section{Overall error analysis}
\label{sec:results}
In this section we combine the upper bounds for the approximation error and for the generalization error with the results for the optimization error from \cite{JentzenWelti2020arxiv} to obtain strong overall error estimates for the training of DNNs. The two main results of this section are \cref{theo:1d} and \cref{theo:main} which estimate the overall error in the case of one-dimensional and multidimensional input data, respectively. In both cases the random variables $\Theta_{k,n}$, $k,n \in \N_0$, are allowed to be computed via an arbitrary optimization algorithm with i.i.d.\ random initializations. A central feature of \cref{theo:1d} is that a single hidden layer is sufficient in order for the overall error to converge to zero reasonably fast. In \cref{theo:main} the error converges to zero exponentially with respect to the number of layers. 

The two main theorems follow from the more general \cref{prop:fullerr} in \cref{subsec:generr} below. The proof of \cref{prop:fullerr} relies on the $L^2$-error decomposition \cite[Proposition 6.1]{JentzenWelti2020arxiv} (cf.\ \cite[Lemma 4.3]{BeckJentzenKuckuck2019arXiv}), which shows that the overall error can be decomposed into the approximation error, the generalization error, and the optimization error. To estimate the generalization error we use \cref{cor:generalization_error}, and for the optimization error we apply the result from Jentzen \& Welti \cite[Corollary 5.8]{JentzenWelti2020arxiv}. We also note that the overall $L^2$-error is measurable and therefore the expectation on the left-hand side of \eqref{eq:theo:1dd} (and similarly in subsequent results) is well-defined. \cref{prop:fullerr} is very similar to \cite[Proposition 6.3]{JentzenWelti2020arxiv} and the proof is only included for completeness.

Afterwards, in \cref{subsec:1d} we combine this with the estimate for one-dimensional functions, \cref{cor:1dapprox}, to derive \cref{theo:1d}.  In \cref{cor:1d} we state a simplified estimate where the architecture parameter $A$ is replaced by the dimension of the single hidden layer. The proof uses the elementary \cref{lem:est:exp}, which is very similar to \cite[Lemma 6.4]{JentzenWelti2020arxiv}. Next, in \cref{subsec:1dsgd} we specialize the results by assuming that the $\Theta_{k,n}$, $k,n \in \N_0$, are computed via stochastic gradient descent (cf.\ \cref{cor:sgd1d} below). 

Thereafter, in \cref{subsec:multid} we combine \cref{prop:fullerr} with the upper bound for the approximation error in the multidimensional case, \cref{prop:approximation_error}, to obtain in \cref{theo:main} a strong overall error estimate. In \cref{cor:main} we replace the parameter $A$ by the minimum of $2^{\bfL-1}$, where $\bfL$ is the depth of the employed DNN, and the layer dimensions $\bfl_1, \bfl_2, \ldots, \bfl_{\bfL - 1}$. The next result, \cref{cor:simple}, is a simplified version of \cref{cor:main}. In particular, the $L^p$-norm of the $L^2$-error is replaced by the expectation of the $L^1$-error and the training samples are restricted to unit hypercubes. 

Finally, in \cref{subsec:multisgd} we apply the results from \cref{subsec:multid} to the case where the random parameter vectors $\Theta_{k,n}$, $k,n \in \N_0$, are again computed by stochastic gradient descent. More specifically, \cref{cor:sgd} specializes \cref{cor:main} to this case, and \cref{cor:sgdsimple} is an immediate consequence of \cref{cor:simple}.

\subsection{General error analysis for the training of DNNs via random initializations} \label{subsec:generr}
\cfclear
\begin{prop} \label{prop:fullerr}
Let $d, \bfd, \bfL, M,K,N \in \N$, $p \in (0, \infty)$, $a, u \in \R$, $b \in (a, \infty)$, $v \in (u, \infty)$, $c \in [ \max\{1, |a|, |b| \}, \infty)$, $\beta \in [c, \infty)$, $\bfl = (\bfl_0, \bfl_1, \ldots, \bfl_\bfL) \in \N^{\bfL+1}$, $\fN \subseteq \{0,1,\ldots, N\}$ satisfy $0 \in \fN$, $\bfl_0=d$, $\bfl_\bfL = 1$, and $\bfd \geq \sum_{i=1}^\bfL \bfl_i(\bfl_{i-1}+1)$, let $(\Omega, \cF, \P)$ be a probability space, let $(X_j, Y_j) \colon \Omega \to [a,b]^d \times [u,v]$, $j \in \{1,2, \ldots, M\}$, be i.i.d.\ random variables, let $\cE \colon [a,b]^d \to [u,v]$ be a measurable function, assume that it holds $\P$-a.s.\ that $\cE(X_1) = \E[Y_1|X_1]$, let $\Theta_{k,n} \colon \Omega \to \R^\bfd$, $k,n \in \N_0$, and $\bfk \colon \Omega \to (\N_0)^2$ be random variables, assume that $\Theta_{k,0}$, $k \in \{1,2, \ldots, K\}$, are i.i.d., assume that $\Theta_{1,0}$ is continuous uniformly distributed on $[-c,c]^\bfd$, and let $\cR \colon \R^\bfd \times \Omega \to [0, \infty)$  satisfy for all $\theta \in \R^\bfd$, $\omega \in \Omega$ that \cfadd{def:p-norm}
\begin{equation}
    \cR(\theta, \omega) = \frac{1}{M} \left[ \sum_{j=1}^M \abs{ \scrN^{\theta, \bfl}_{u,v}(X_j(\omega))-Y_j(\omega)}^2 \right],
\end{equation}
\begin{equation}
    \bfk( \omega) \in \arg \min\nolimits_{(k,n) \in \{1,2, \ldots, K\} \times \fN,\, \| \Theta_{k,n} (\omega) \|_\infty \leq \beta} \cR(\Theta_{k,n}(\omega), \omega)
\end{equation}
\cfload.
Then
\begin{enumerate} [(i)]
\item \label{item:prop:fullerr1} 
the function
\begin{equation}
    \Omega \ni \omega \mapsto \int_{[a,b]^d} \bigl|  \scrN^{\Theta_{\bfk(\omega)}(\omega), \bfl}_{u,v}  (x)-\cE (x) \bigr| ^2 \, \P_{X_1}(\dx x) \in [0, \infty]
\end{equation}
is measurable and
    \item \label{item:prop:fullerr2}
    it holds that
\begin{equation}
    \begin{split}
        & \left( \E \! \left[ \left(   \int_{[a,b]^d} \abs{  \scrN^{\Theta_\bfk, \bfl}_{u,v}(x)-\cE (x) }^2 \, \P_{X_1}(\dx x)  \right)^p \right] \right) ^{\! \nicefrac{1}{p}} \\
        &\leq \inf\nolimits_{\theta \in[-c,c]^\bfd} \sup\nolimits_{x \in [a,b]^d} \abs{ \scrN_{u,v}^{\theta, \bfl} (x) - \cE(x) } ^2+
        \frac{4 (v-u) \bfL ( \norm{\bfl}_\infty+1)^\bfL c^{\bfL + 1} \max \{p, 1 \}}
            {K^{[\bfL^{-1}(\norm{\bfl}_\infty+1)^{-2}]}} \\
        &+ \frac{10 (  \max \{v - u, 1 \} )^2 \bfL ( \norm{\bfl}_\infty +1)^{\nicefrac{3}{2}} \max \{p, \ln (4 M \beta c)\}}{\sqrt{M}}.
    \end{split}
\end{equation}
\end{enumerate}
\cfout
\end{prop}

\begin{proof} [Proof of \cref{prop:fullerr}]
Observe that item \eqref{item:prop:fullerr1} follows from \cite[Lemma 6.2 (iii)]{JentzenWelti2020arxiv} (applied with $D \with [a,b]^d $, $X \with X_1$, $p \with 2$ in the notation of \cite[Lemma 6.2]{JentzenWelti2020arxiv}). To prove \eqref{item:prop:fullerr2}, let $\bfR \colon \R^\bfd \to [0, \infty)$ satisfy for all $\theta \in \R^\bfd$ that $\bfR(\theta) = \E | \scrN^{\theta, \bfl}_{u,v}(X_1)-Y_1|^2$. The assumptions that $\Theta_{k,0}$, $k \in \{1,2, \ldots, K\}$, are continuous uniformly distributed on $[-c, c]^\bfd$ and $\beta \geq c$ ensure that $\P$-a.s.\ we have that $\left( \bigcup_{k=1}^K \Theta_{k,0}(\Omega) \right) \subseteq [-\beta,\beta]^\bfd$. This,  the fact that $0 \in \fN$, and the $L^2$-error decomposition \cite[Proposition 6.1]{JentzenWelti2020arxiv} show that for all $\vartheta \in [-\beta,\beta]^\bfd$ it holds $\P$-a.s.\ that
\begin{equation}
    \begin{split}
         &\int_{[a,b]^d} | \scrN^{\Theta_\bfk, \bfl}_{u,v}(x)-\cE (x) |^2 \, \P_{X_1}(\dx x) \\
        &\leq \sup\nolimits_{x \in [a,b]^d} |\scrN_{u,v}^{\vartheta, \bfl}(x)- \cE(x)|^2 
        + 2 \sup\nolimits_{\theta \in [-\beta,\beta]^\bfd} |\cR(\theta)- \bfR(\theta)|\\
        &+ \min\nolimits_{(k,n) \in \{1,2, \ldots, K\} \times \fN,\, \| \Theta_{k,n}(\omega) \|_\infty \leq \beta} |\cR(\Theta_{k,n})-\cR(\vartheta)| \\
        &\leq \sup\nolimits_{x \in [a,b]^d} |\scrN_{u,v}^{\vartheta, \bfl}(x)- \cE(x)|^2 
        + 2 \sup\nolimits_{\theta \in [-\beta,\beta]^\bfd} |\cR(\theta)- \bfR(\theta)|\\
        &+ \min\nolimits_{k \in \{1,2, \ldots, K\}} |\cR(\Theta_{k,0})-\cR(\vartheta)|.
    \end{split}
\end{equation}
Combining this and Minkowski's inequality demonstrates for all $q \in [1, \infty)$, $\vartheta \in [-c,c]^\bfd$ that
\begin{equation} \label{eq:prop:fullerr1}
    \begin{split}
         & \left( \E \! \left[ \left(   \int_{[a,b]^d} \abs{ \scrN^{\Theta_\bfk, \bfl}_{u,v}(x)-\cE (x) }^2 \, \P_{X_1}(\dx x)  \right)^q \right] \right)^{\! \nicefrac{1}{q}} \\
         &\leq \left(\E \! \left[\sup\nolimits_{x \in [a,b]^d} |\scrN_{u,v}^{\vartheta, \bfl}(x)- \cE(x)|^{2q} \right] \right)^{\nicefrac{1}{q}} 
         + 2 \left( \E \! \left[\sup\nolimits_{\theta \in [-\beta,\beta]^\bfd} |\cR(\theta)- \bfR(\theta)|^q \right]\right)^{\nicefrac{1}{q}} \\
         &+ \left( \E \! \left[\min\nolimits_{k \in \{1,2, \ldots, K\}} |\cR(\Theta_{k,0})-\cR(\vartheta)|^q \right]\right)^{\nicefrac{1}{q}} \\
         &\leq \sup\nolimits_{x \in [a,b]^d} |\scrN_{u,v}^{\vartheta, \bfl}(x)- \cE(x)|^2  
          + 2 \left( \E \! \left[\sup\nolimits_{\theta \in [-\beta,\beta]^\bfd} |\cR(\theta)- \bfR(\theta)|^q \right]\right)^{\nicefrac{1}{q}} \\
          &+ \sup\nolimits_{\theta \in [-c,c]^\bfd} \left( \E \! \left[\min\nolimits_{k \in \{1,2, \ldots, K\}} |\cR(\Theta_{k,0})-\cR(\theta)|^q \right]\right)^{\nicefrac{1}{q}}.
    \end{split}
\end{equation}
Next, observe that the fact that $[a,b]^d \subseteq [-c, c ]^d$ and \cref{cor:generalization_error} (applied with $ \scrD  \with [a,b]^d$, $v \with \max \{u+1, v\}$, $p \with q$ in the notation of \cref{cor:generalization_error}) imply for all $q \in (0, \infty)$ that
\begin{equation} \label{eq:prop:fullerr2}
\begin{split}
   &2 \left( \E \! \left[\sup\nolimits_{\theta \in [-\beta,\beta]^\bfd} |\cR(\theta)- \bfR(\theta)|^q \right]\right)^{\nicefrac{1}{q}} \\ &
     \leq \frac{
    10 (  \max \{v - u, 1 \} )^2
    \bfL ( \norm{ \bfl }_{ \infty } + 1 )^{\nicefrac{3}{2}}
    \max\{ q, \ln( 4 M \beta c ) \} }{ \sqrt{M} }.
\end{split}
\end{equation}
Moreover, \cite[Corollary 5.8]{JentzenWelti2020arxiv} (applied with $D \with [a,b]^d$, $b \with c$, $\beta \with c$, $(\Theta_k)_{k \in \{1, \ldots, K\}} \with (\Theta_{k,0})_{k \in \{1, \ldots, K \} }$, $p \with q$ in the notation of \cite[Corollary 5.8]{JentzenWelti2020arxiv}) proves for all $q \in (0, \infty)$ that
\begin{equation} \label{eq:prop:fullerr3}
\begin{split}
    &\sup\nolimits _{\theta \in [-c,c]^\bfd} \left( \E  \bigl[ \min\nolimits_{k \in \{1,2, \ldots, K\}}|\cR(\Theta_{k,0})-\cR(\theta)|^q \bigr] \right)^{\nicefrac{1}{q}} \\
            &\leq \frac{4(v-u) \bfL ( \norm{\bfl}_\infty+1)^\bfL c^{\bfL + 1} \max \{q, 1 \}}
            {K^{[\bfL^{-1}(\norm{\bfl}_\infty+1)^{-2}]}}.
\end{split}
\end{equation}
Combining \eqref{eq:prop:fullerr1}--\eqref{eq:prop:fullerr3} with Jensen's inequality and the fact that $\ln (4 M \beta c) \geq 1$ demonstrates that
\begin{equation}
      \begin{split}
         & \left( \E \! \left[ \left(   \int_{[a,b]^d} \abs{ \scrN^{\Theta_\bfk, \bfl}_{u,v}(x)-\cE (x) }^2 \, \P_{X_1}(\dx x)  \right)^p \right] \right)^{ \! \nicefrac{1}{p}} \\
         &\leq \left( \E \! \left[ \left(   \int_{[a,b]^d} \abs{ \scrN^{\Theta_\bfk, \bfl}_{u,v}(x)-\cE (x) }^2 \, \P_{X_1}(\dx x)  \right)^{\max\{p, 1\}} \right] \right)^{\! \nicefrac{1}{\max\{p, 1\}}} \\
         &\leq \inf\nolimits_{\theta \in [-c,c]^\bfd}\sup\nolimits_{x \in [a,b]^d} |\scrN_{u,v}^{\theta, \bfl}(x)- \cE(x)|^2  
          + 2 \left( \E \! \left[\sup\nolimits_{\theta \in [-\beta,\beta]^\bfd} |\cR(\theta)- \bfR(\theta)|^{\max\{p, 1 \}} \right]\right)^{\nicefrac{1}{\max\{ p, 1\}}} \\
          &+ \sup\nolimits_{\theta \in [-c,c]^\bfd} \left( \E \! \left[\min\nolimits_{k \in \{1,2, \ldots, K\}} |\cR(\Theta_{k,0})-\cR(\theta)|^{\max\{p, 1 \}} \right]\right)^{\nicefrac{1}{{\max\{p, 1 \}}}} \\
          &\leq \inf\nolimits_{\theta \in [-c,c]^\bfd}\sup\nolimits_{x \in [a,b]^d} \!  |\scrN_{u,v}^{\theta, \bfl}(x)- \cE(x)|^2  
          + \frac{10 (  \max \{v - u, 1 \} )^2
    \bfL ( \norm{ \bfl }_{ \infty } + 1 )^{\nicefrac{3}{2}} \!
    \max\{ p, \ln( 4 M \beta c ) \}}{ \sqrt{M} }\\
          &+ \frac{4(v-u) \bfL ( \norm{\bfl}_\infty+1)^\bfL c^{\bfL + 1} \max \{p, 1 \}}
            {K^{[\bfL^{-1}(\norm{\bfl}_\infty+1)^{-2}]}}.
         \end{split}
\end{equation}
This completes the proof of \cref{prop:fullerr}.
\end{proof}

\subsection{One-dimensional strong overall error analysis} \label{subsec:1d}
\cfclear
\begin{theorem} \label{theo:1d}
Let $\bfd, \bfL, M,K,N \in \N$, $A, p \in (0, \infty)$, $L \in [0, \infty)$, $a, u \in \R$, $v \in (u, \infty)$, $b \in (a, \infty)$, $c \in [ \max\{1, |u|, |v|, |a|, |b|, 2L \}, \infty)$, $\beta \in [c, \infty)$, $\bfl = (\bfl_0, \bfl_1, \ldots, \bfl_\bfL) \in \N^{\bfL+1}$, $\fN \subseteq \{0,1,\ldots, N\}$ satisfy $0 \in \fN$,
$ \bfL \geq 2 $,
$ \bfl_0 = \bfl_\bfL=1 $,
$ \bfl_1 \geq A +2 $,
and
$ \bfd \geq \sum_{i=1}^{\bfL} \bfl_i( \bfl_{ i - 1 } + 1 ) $,
assume for all
$ i \in \{ 2, 3, \ldots , \bfL-1\}$
that
$ \bfl_i \geq 2$, let $(\Omega, \cF, \P)$ be a probability space, let $(X_j, Y_j) \colon \Omega \to [a,b] \times [u,v]$, $j \in \{1,2, \ldots, M\}$, be i.i.d.\ random variables, let $\cE \colon [a,b] \to [u,v]$ satisfy $\P$-a.s.\ that $\cE(X_1) = \E[Y_1|X_1]$, assume for all $x,y \in [a,b]$ that $|\cE (x)- \cE (y)| \leq L |x-y|$, let $\Theta_{k,n} \colon \Omega \to \R^\bfd$, $k,n \in \N_0$, and $\bfk \colon \Omega \to (\N_0)^2$ be random variables, assume that $\Theta_{k,0}$, $k \in \{1,2, \ldots, K\}$, are i.i.d., assume that $\Theta_{1,0}$ is continuous uniformly distributed on $[-c,c]^\bfd$, and let $\cR \colon \R^\bfd \times \Omega \to [0, \infty)$  satisfy for all $\theta \in \R^\bfd$, $\omega \in \Omega$ that \cfadd{def:p-norm}
\begin{equation}
    \cR(\theta, \omega) = \frac{1}{M} \left[ \sum_{j=1}^M \abs{ \scrN^{\theta, \bfl}_{u,v}(X_j(\omega))-Y_j(\omega)} ^2 \right],
\end{equation}
\begin{equation}
    \bfk( \omega) \in \arg \min\nolimits_{(k,n) \in \{1,2, \ldots, K\} \times \fN,\, \| \Theta_{k,n} (\omega) \|_\infty \leq \beta} \cR(\Theta_{k,n}(\omega), \omega)
\end{equation}
\cfload. Then
\begin{equation} \label{eq:theo:1dd}
 \begin{split}
        & \left( \E \! \left[ \left(   \int_{[a,b]} \abs{ \scrN^{\Theta_\bfk, \bfl}_{u,v}(x)-\cE (x) }^2 \, \P_{X_1}(\dx x)  \right)^p \right] \right)^{\! \nicefrac{1}{p}} \\
        &\leq \frac{ L^2(b-a)^2}{A^2}+
        \frac{4(v-u) \bfL ( \norm{\bfl}_\infty+1)^\bfL c^{\bfL+1} \max \{p, 1 \}}
            {K^{[\bfL^{-1}(\norm{\bfl}_\infty+1)^{-2}]}} \\
        &+ \frac{10 (  \max \{v - u, 1 \} )^2 \bfL ( \norm{\bfl}_\infty +1)^{\nicefrac{3}{2}} \max \{p, \ln (4 M \beta c)\}}{\sqrt{M}}.
    \end{split}
\end{equation}
\end{theorem}

\begin{proof} [Proof of \cref{theo:1d}]
First of all, note that \cref{prop:fullerr} proves that
\begin{equation} \label{eq:theo:1d}
 \begin{split}
        & \left( \E \! \left[ \left(   \int_{[a,b]} \abs{ \scrN^{\Theta_\bfk, \bfl}_{u,v}(x)-\cE (x) }^2 \, \P_{X_1}(\dx x)  \right)^p \right] \right)^{\! \nicefrac{1}{p}} \\
        &\leq \inf\nolimits_{\theta \in[-c,c]^\bfd} \sup\nolimits_{x \in [a,b]} |\scrN_{u,v}^{\theta, \bfl} (x) - \cE(x)|^2+
        \frac{4(v-u) \bfL ( \norm{\bfl}_\infty+1)^\bfL c^{\bfL+1} \max \{p, 1 \}}
            {K^{[\bfL^{-1}(\norm{\bfl}_\infty+1)^{-2}]}} \\
        &+ \frac{10 (  \max \{v - u, 1 \} )^2 \bfL ( \norm{\bfl}_\infty +1)^{\nicefrac{3}{2}} \max \{p, \ln (4 M \beta c)\}}{\sqrt{M}}.
    \end{split}
\end{equation}
Furthermore, observe that \cref{cor:1dapprox} ensures that there exists $\vartheta \in \R^\bfd$ such that $\norm{\vartheta}_\infty \leq \max \{1, 2L, |a|, |b|, 2\sup\nolimits_{x \in [a,b]} |\cE(x)| \}$ and
\begin{equation} \label{theo:est:1dapprox}
    \sup\nolimits _{x \in [a,b]} |\scrN_{u,v}^{\vartheta, \bfl} (x) - \cE(x)| \leq \frac{L(b-a)}{A}.
\end{equation}
The assumption that $\cE ([a,b]) \subseteq [u,v]$ implies that $\sup\nolimits_{x \in [a,b]} |\cE(x)| \leq \max \{|u|, |v|\} \leq c$ and thus $\norm{\vartheta}_\infty \leq c$. Equation \eqref{theo:est:1dapprox} therefore demonstrates that
\begin{equation}
    \inf\nolimits_{\theta \in [-c,c]^\bfd}\sup\nolimits_{x \in [a,b]} \abs{ \scrN_{u,v}^{\theta, \bfl}(x)- \cE(x) } ^2  \leq \frac{L^2(b-a)^2}{A^2}.
\end{equation}
Combining this and \eqref{eq:theo:1d} completes the proof of \cref{theo:1d}.
\end{proof}

\begin{lemma} \label{lem:est:exp}
Let $c, M \in [1, \infty)$ and $\beta \in [c, \infty)$. Then it holds that
\begin{equation}
    \ln (4 M \beta c) \leq \tfrac{3\beta}{2} \ln (e M).
\end{equation}
\end{lemma}
\begin{proof} [Proof of \cref{lem:est:exp}]
First, observe that the fact that $\forall \, y \in \R \colon e^{y-1} \geq y$ implies that
\begin{equation}
    \exp \left( \tfrac{4\beta}{e}\right) = \left( e  \exp \left( \tfrac{2\beta}{e}-1 \right) \right)^2 \geq \left( e \left( \tfrac{2\beta}{e} \right) \right)^2 = 4\beta^2.
\end{equation}
This, the assumption that $\beta \geq c$, and the facts that $\tfrac{4\beta}{e} > 1$, $\ln (M) \geq 0$, and $\tfrac{4}{e} < \tfrac{3}{2}$ demonstrate that
\begin{equation}
    \begin{split}
           \ln (4 M \beta c) &\leq \ln (4 M \beta^2) \leq \ln \left(M \exp \left( \tfrac{4\beta}{e}\right) \right) = \ln (M) +  \tfrac{4\beta}{e} \\
           &\leq \tfrac{4\beta}{e} (1+ \ln (M)) \leq \tfrac{3\beta}{2} \ln (e M).
    \end{split}
\end{equation}
The proof of \cref{lem:est:exp} is thus complete.
\end{proof}

\cfclear
\begin{cor} \label{cor:1d}
Let $\bfd, M,K,N \in \N$, $p \in (0, \infty)$, $ L \in [0, \infty)$, $a, u \in \R$, $v \in (u, \infty)$, $b \in (a, \infty)$, $c \in [ \max\{1, 2|u|, 2|v|, |a|, |b|, 2L \}, \infty)$, $\beta \in [c, \infty)$, $\ell  \in \N \cap [3, \infty) \cap (0, \nicefrac{(\bfd - 1)}{3}]$, let $\fN \subseteq \{0,1,\ldots, N\}$ satisfy $0 \in \fN$,
 let $(\Omega, \cF, \P)$ be a probability space, let $(X_j, Y_j) \colon \Omega \to [a,b] \times [u,v]$, $j \in \{1,2, \ldots, M\}$, be i.i.d.\ random variables, let $\cE \colon [a,b] \to [u,v]$ satisfy $\P$-a.s.\ that $\cE(X_1) = \E[Y_1|X_1]$, assume for all $x,y \in [a,b]$ that $| \cE (x)- \cE (y)| \leq L |x-y|$, let $\Theta_{k,n} \colon \Omega \to \R^\bfd$, $k,n \in \N_0$, and $\bfk \colon \Omega \to (\N_0)^2$ be random variables, assume that $\Theta_{k,0}$, $k \in \{1,2, \ldots, K\}$, are i.i.d., assume that $\Theta_{1,0}$ is continuous uniformly distributed on $[-c,c]^\bfd$, and let $\cR \colon \R^\bfd \times \Omega \to [0, \infty)$  satisfy for all $\theta \in \R^\bfd$, $\omega \in \Omega$ that  \cfadd{def:p-norm}
\begin{equation}
    \cR(\theta, \omega) = \frac{1}{M} \left[ \sum_{j=1}^M \abs{ \scrN^{\theta, (1, \ell, 1)}_{u,v}(X_j(\omega))-Y_j(\omega) } ^2 \right],
\end{equation}
\begin{equation}
    \bfk( \omega) \in \arg \min\nolimits_{(k,n) \in \{1,2, \ldots, K\} \times \fN,\, \| \Theta_{k,n} (\omega) \|_\infty \leq \beta} \cR(\Theta_{k,n}(\omega), \omega)
\end{equation}
\cfload. Then
\begin{equation}
 \begin{split}
        & \left( \E \! \left[ \left(   \int_{[a,b]} \abs{ \scrN^{\Theta_\bfk, (1, \ell, 1) }_{u,v}(x)-\cE (x) } ^2 \, \P_{X_1}(\dx x)  \right)^{\! \nicefrac{p}{2}} \right] \right)^{\! \nicefrac{1}{p}} \\
        &\leq \frac{ L(b-a)}{\ell - 2}+
        \frac{\left[8(v-u)  (\ell+1)^2 c^{3} \max \{ \nicefrac{p}{2}, 1 \}\right]^{\nicefrac{1}{2}}}
            {K^{[(\ell+1)^{-2}/4]}} \\
        &+ \frac{ \max \{v-u, 1\} \left[20 (\ell+1)^{\nicefrac{3}{2}} \max \{ \nicefrac{p}{2}, \ln (4 M \beta c)\}\right]^{\nicefrac{1}{2}}}{M^{\nicefrac{1}{4}}} \\
           &\leq \frac{3c^2}{\ell} + \frac{4c^2 \ell \max \{p, 1 \}}{K^{[(\ell+1)^{-2}/4]}} + \frac{6 \beta c \, \ell \max \{p, \ln (eM) \}}{M^{\nicefrac{1}{4}}}.
    \end{split}
    \end{equation}

\end{cor}
\begin{proof} [Proof of \cref{cor:1d}]
Observe that the assumption that $\bfd \geq 3 \ell +1$ and \cref{theo:1d} (applied with $\bfL \with 2$, $\bfl \with (1, \ell, 1)$,  $A \with \ell-2$, $p \with \nicefrac{p}{2}$ in the notation of \cref{theo:1d}) demonstrate that
\begin{multline}
         \left( \E \! \left[ \left(   \int_{[a,b]} \abs{ \scrN^{\Theta_\bfk, (1, \ell, 1)}_{u,v}(x)-\cE (x) }^2 \, \P_{X_1}(\dx x)  \right)^{\! \nicefrac{p}{2}} \right] \right)^{\! \nicefrac{2}{p}} 
        \leq \frac{ L^2(b-a)^2}{(\ell-2)^2} \\
        + \frac{8(v-u) ( \ell+1)^2 c^{3} \max \{ \nicefrac{p}{2} , 1 \}}
            {K^{[2^{-1}(\ell +1)^{-2}]}} 
        + \frac{20 (  \max \{v - u, 1 \} )^2  ( \ell + 1)^{\nicefrac{3}{2}} \max \{\nicefrac{p}{2}, \ln (4 M \beta c)\}}{\sqrt{M}}.
    \end{multline}
Combining this, \cref{lem:est:exp}, and the facts that $c \geq |u|+|v| \geq |u-v|$, $L \leq \frac{c}{2}$, and $|b-a| \leq 2c$ shows that
\begin{equation}
    \begin{split}
          & \left( \E \! \left[ \left(   \int_{[a,b]} \abs{ \scrN^{\Theta_\bfk, (1, \ell, 1)}_{u,v}(x)-\cE (x) }^2 \, \P_{X_1}(\dx x)  \right)^{\! \nicefrac{p}{2}} \right] \right)^{\! \nicefrac{1}{p}} \\
        &\leq \frac{ L(b-a)}{\ell - 2}+
        \frac{ \left[8(v-u)  (\ell+1)^2 c^{3} \max \{ \nicefrac{p}{2}, 1 \}\right]^{\nicefrac{1}{2}}}
            {K^{[(\ell+1)^{-2}/4 ] }} \\
        &+ \frac{ \max \{v-u, 1 \} \left[20 (\ell+1)^{\nicefrac{3}{2}} \max \{ \nicefrac{p}{2} , \ln (4 M \beta c)\}\right]^{\nicefrac{1}{2}}}{M^{\nicefrac{1}{4}}} \\
        &\leq \frac{ c^2}{\ell-2}
        +  \frac{\sqrt{8} c^2 ( \ell+1) \max \{p,1 \}^{\nicefrac{1}{2}}}
            {K^{[(\ell+1)^{-2}/4 ] }} 
         + \frac{\sqrt{30 } c \beta^{\nicefrac{1}{2}}  ( \ell +1)^{\nicefrac{3}{4}} \max \{p, \ln (e M)\}}{M^{\nicefrac{1}{4}}}.
    \end{split}
\end{equation}
Furthermore, the fact that  $\ell  \in \{3,4,5, \ldots\}$ implies that $\ell ^4 \geq 3 \ell ^3 > \tfrac{4^3}{3^3}  \ell ^3 = \left( \tfrac{4 \ell }{3} \right) ^3 \geq ( \ell +1)^3$, $4 \ell  = 3 \ell  +  \ell  > \sqrt{8} \,  \ell  + \sqrt{8} = \sqrt{8}( \ell +1)$, and $3( \ell -2) \geq \ell $
and therefore
\begin{equation}
\begin{split}
    &\frac{ c^2}{\ell-2}
        +  \frac{\sqrt{8} c^2 ( \ell+1) \max \{p, 1 \}^{\nicefrac{1}{2}}}
            {K^{[(\ell+1)^{-2}/4 ] }} 
         + \frac{\sqrt{30 } c \beta^{\nicefrac{1}{2}}  ( \ell +1)^{\nicefrac{3}{4}} \max \{p, \ln (e M)\}}{M^{\nicefrac{1}{4}}} \\
         &\leq \frac{3c^2}{\ell} + \frac{4c^2 \ell \max \{p, 1 \}}{K^{[(\ell+1)^{-2}/4 ] }} + \frac{6 \beta c \, \ell \max \{p, \ln (eM) \}}{M^{\nicefrac{1}{4}}}.
\end{split}
\end{equation}
This completes the proof of \cref{cor:1d}.
\end{proof}

\subsection{Stochastic gradient descent for one-dimensional target functions} \label{subsec:1dsgd}

\cfclear
\begin{cor} \label{cor:sgd1d}
Let $\bfd, M,K,N \in \N$, $p \in (0, \infty)$, $ L \in [0, \infty)$, $a, u \in \R$, $v \in (u, \infty)$, $b \in (a, \infty)$, $c \in [ \max\{1, 2|u|, 2|v|, |a|, |b|, 2L \}, \infty)$, $\beta \in [c, \infty)$, 
$(\gamma_n)_{n \in \N} \subseteq \R$,
$(J_n)_{n \in \N} \subseteq \N$,
$\ell  \in \N \cap [3, \infty) \cap (0, \nicefrac{(\bfd - 1)}{3}]$, let $\fN \subseteq \{0,1,\ldots, N\}$ satisfy $0 \in \fN$,
 let $(\Omega, \cF, \P)$ be a probability space, let $(X_j^{k,n}, Y_j^{k,n}) \colon \Omega \to [a,b] \times [u,v]$, $j \in \N$, $k,n \in \N_0$, be random variables, assume that $(X_j^{0,0}, Y_j^{0,0})$, $j \in \{1,2, \ldots, M\}$, are i.i.d.,
 let $\cE \colon [a,b] \to [u,v]$ satisfy $\P$-a.s.\ that $\cE(X_1^{0,0}) = \E[Y_1^{0,0}|X_1^{0,0}]$, assume for all $x,y \in [a,b]$ that $| \cE (x)- \cE (y)| \leq L |x-y|$, let $\Theta_{k,n} \colon \Omega \to \R^\bfd$, $k,n \in \N_0$, and $\bfk \colon \Omega \to (\N_0)^2$ be random variables, assume that $\Theta_{k,0}$, $k \in \{1,2, \ldots, K\}$, are i.i.d., assume that $\Theta_{1,0}$ is continuous uniformly distributed on $[-c,c]^\bfd$, let $\cR_J^{k,n} \colon \R^\bfd \times \Omega \to [0, \infty)$  satisfy for all $J \in \N$, $k,n \in \N_0$, $\theta \in \R^\bfd$, $\omega \in \Omega$ that  \cfadd{def:p-norm}
\begin{equation}
    \cR_J^{k,n}(\theta, \omega) = \frac{1}{J} \left[ \sum_{j=1}^J \bigl| \scrN^{\theta, (1, \ell, 1)}_{u,v}(X_j^{k,n}(\omega))-Y_j^{k,n}(\omega) \bigr| ^2 \right],
\end{equation}
let $\cG ^{k,n} \colon \R^\bfd \times \Omega \to \R^\bfd$, $k,n\in \N$, satisfy for all $ k,n \in \N$, $\omega \in \Omega$, $\theta \in \{ \vartheta \in \R^\bfd \colon (\cR^{k,n}_{J_n}(\cdot, \omega) \colon \R^\bfd \to [0, \infty) \text{ is differentiable at } \vartheta )\}$ that $\cG^{k,n}(\theta, \omega) = (\nabla_\theta \cR^{k,n}_{J_n})(\theta, \omega)$, assume for all $ k,n \in \N$ that $\Theta_{k,n} = \Theta_{k, n-1} - \gamma_n \cG^{k,n}(\Theta_{k, n-1})$, and assume for all $\omega \in \Omega$ that
\begin{equation}
    \bfk( \omega) \in \arg \min\nolimits_{(k,n) \in \{1,2, \ldots, K\} \times \fN,\, \| \Theta_{k,n} (\omega) \|_\infty \leq \beta} \cR_M^{0,0} (\Theta_{k,n}(\omega), \omega)
\end{equation}
\cfload. Then
\begin{equation}
 \begin{split}
         &\left( \E \! \left[ \left(   \int_{[a,b]} \abs{ \scrN^{\Theta_\bfk, (1, \ell, 1) }_{u,v}(x)-\cE (x) } ^2 \, \P_{X_1^{0,0}}(\dx x)  \right)^{\! \nicefrac{p}{2}} \right] \right)^{\! \nicefrac{1}{p}} \\ &
           \leq \frac{3c^2}{\ell} + \frac{4c^2 \ell \max \{p, 1 \}}{K^{[(\ell+1)^{-2}/4]}} + \frac{6 \beta c \, \ell \max \{p, \ln (eM) \}}{M^{\nicefrac{1}{4}}}.
    \end{split}
    \end{equation}
\end{cor}

\begin{proof} [Proof of \cref{cor:sgd1d}]
This is a direct consequence \cref{cor:1d} (applied with $\cR \with \cR_M^{0,0}$, $((X_j, Y_j))_{j \in \{1, \ldots, M \} } \with ((X_j^{0,0}, Y_j^{0,0}))_{j \in \{1, \ldots, M \} }$ in the notation of \cref{cor:1d}). The proof of \cref{cor:sgd1d} is thus complete.
\end{proof}

\subsection{Multidimensional strong overall error analysis} \label{subsec:multid}
\cfclear
\begin{theorem} \label{theo:main}
Let $d, \bfd, \bfL, M,K,N \in \N$, $A, p \in (0, \infty)$, $L \in [ 0, \infty)$, $a, u \in \R$, $v \in (u, \infty)$, $b \in (a, \infty)$, $c \in [ \max\{1, 2|u|, 2|v|, |a|, |b|, L \}, \infty)$, $\beta \in [c, \infty)$, $\bfl = (\bfl_0, \bfl_1, \ldots, \bfl_\bfL) \in \N^{\bfL+1}$, let $\fN \subseteq \{0,1,\ldots, N\}$ satisfy $0 \in \fN$, assume
$ \bfL \geq 1 + (\ceil{ \log_2 \left(\nicefrac{ A  }{ ( 2d ) } \right)} + 1) \indicator{  ( 6^d, \infty ) }( A )$,
$ \bfl_0 = d $,
$ \bfl_1 \geq A \indicator{ ( 6^d, \infty )}( A ) $,
$ \bfl_\bfL = 1 $,
and
$ \bfd \geq \sum_{i=1}^{\bfL} \bfl_i( \bfl_{ i - 1 } + 1 ) $,
assume for all
$ i \in \{ 2, 3, \ldots , \bfL-1\}$
that
$ \bfl_i \geq 3 \ceil{\nicefrac{A}{(2^id)}} \indicator{( 6^d, \infty )}( A ) $, let $(\Omega, \cF, \P)$ be a probability space, let $(X_j, Y_j) \colon \Omega \to [a,b]^d \times [u,v]$, $j \in \{1,2, \ldots, M\}$, be i.i.d.\ random variables, let $\cE \colon [a,b]^d \to [u,v]$ satisfy $\P$-a.s.\ that $\cE(X_1) = \E[Y_1|X_1]$, assume for all $x,y \in [a,b]^d$ that $| \cE (x)- \cE (y)| \leq L \norm{x-y}_1$, let $\Theta_{k,n} \colon \Omega \to \R^\bfd$, $k,n \in \N_0$, and $\bfk \colon \Omega \to (\N_0)^2$ be random variables, assume that $\Theta_{k,0}$, $k \in \{1,2, \ldots, K\}$, are i.i.d., assume that $\Theta_{1,0}$ is continuous uniformly distributed on $[-c,c]^\bfd$, and let $\cR \colon \R^\bfd \times \Omega \to [0, \infty)$  satisfy for all $\theta \in \R^\bfd$, $\omega \in \Omega$ that \cfadd{def:p-norm} 
\begin{equation}
    \cR(\theta, \omega) = \frac{1}{M} \left[ \sum_{j=1}^M \abs{ \scrN^{\theta, \bfl}_{u,v}(X_j(\omega))-Y_j(\omega) } ^2 \right],
\end{equation}
\begin{equation}
    \bfk( \omega) \in \arg \min\nolimits_{(k,n) \in \{1,2, \ldots, K\} \times \fN,\, \| \Theta_{k,n} (\omega) \|_\infty \leq \beta} \cR(\Theta_{k,n}(\omega), \omega)
\end{equation}
\cfload.
Then
\begin{equation} \label{eq:theo:main}
 \begin{split}
        & \left( \E \! \left[ \left(   \int_{[a,b]^d} \abs{ \scrN^{\Theta_\bfk, \bfl}_{u,v}(x)-\cE (x) }^2 \, \P_{X_1}(\dx x)  \right)^p \right] \right)^{\! \nicefrac{1}{p}} \\
        &\leq \frac{9d^2 L^2(b-a)^2}{A^{\nicefrac{2}{d}}}+
        \frac{4(v-u) \bfL ( \norm{\bfl}_\infty+1)^\bfL c^{\bfL+1} \max \{p, 1 \}}
            {K^{[\bfL^{-1}(\norm{\bfl}_\infty+1)^{-2}]}} \\
        &+ \frac{10 (  \max \{v - u, 1 \} )^2 \bfL ( \norm{\bfl}_\infty +1)^{\nicefrac{3}{2}} \max \{p, \ln (4 M \beta c)\}}{\sqrt{M}}.
    \end{split}
    \end{equation}

\end{theorem}
\begin{proof} [Proof of \cref{theo:main}]
Observe that \cref{prop:approximation_error} ensures that there exists $\vartheta \in \R^\bfd$ such that $\norm{\vartheta}_\infty \leq \max \{1, L, |a|, |b|, 2\sup\nolimits_{x \in [a,b]^d} |\cE(x)| \}$ and
\begin{equation} \label{theo:est:approx}
    \sup\nolimits _{x \in [a,b]^d} |\scrN_{u,v}^{\vartheta, \bfl} (x) - \cE(x)| \leq \frac{3dL(b-a)}{A^{\nicefrac{1}{d}}}.
\end{equation}
The assumption that $\cE ([a,b]^d) \subseteq  [u,v]$ implies that $\sup\nolimits_{x \in [a,b]^d} |\cE(x)| \leq \max \{|u|, |v|\} \leq \frac{c}{2}$ and thus $\norm{\vartheta}_\infty \leq c$. Equation \eqref{theo:est:approx} therefore shows that
\begin{equation}
    \inf\nolimits_{\theta \in [-c,c]^\bfd}\sup\nolimits_{x \in [a,b]^d} |\scrN_{u,v}^{\theta, \bfl}(x)- \cE(x)|^2  \leq \frac{9d^2L^2(b-a)^2}{A^{\nicefrac{2}{d}}}.
\end{equation}
Combining this with \cref{prop:fullerr} establishes \eqref{eq:theo:main} and thus completes the proof of \cref{theo:main}.
\end{proof}

\cfclear
\begin{cor} \label{cor:main}
Let $d, \bfd, \bfL, M,K,N \in \N$, $p \in (0, \infty)$, $ L \in [0, \infty)$, $a, u \in \R$, $v \in (u, \infty)$, $b \in (a, \infty)$, $c \in [ \max\{1, 2|u|, 2|v|, |a|, |b|, L \}, \infty)$, $\beta \in [c, \infty)$, $\bfl = (\bfl_0, \bfl_1, \ldots, \bfl_\bfL) \in \N^{\bfL+1}$, $\fN \subseteq \{0,1,\ldots, N\}$ satisfy $0 \in \fN$, 
$ \bfl_0 = d $, 
$ \bfl_\bfL = 1 $,
and
$ \bfd \geq \sum_{i=1}^{\bfL} \bfl_i( \bfl_{ i - 1 } + 1 ) $,
 let $(\Omega, \cF, \P)$ be a probability space, let $(X_j, Y_j) \colon \Omega \to [a,b]^d \times [u,v]$, $j \in \{1,2, \ldots, M\}$, be i.i.d.\ random variables, let $\cE \colon [a,b]^d \to [u,v]$ satisfy $\P$-a.s.\ that $\cE(X_1) = \E[Y_1|X_1]$, assume for all $x,y \in [a,b]^d$ that $| \cE (x)- \cE (y)| \leq L \norm{x-y}_1$, let $\Theta_{k,n} \colon \Omega \to \R^\bfd$, $k,n \in \N_0$, and $\bfk \colon \Omega \to (\N_0)^2$ be random variables, assume that $\Theta_{k,0}$, $k \in \{1,2, \ldots, K\}$, are i.i.d., assume that $\Theta_{1,0}$ is continuous uniformly distributed on $[-c,c]^\bfd$, and let $\cR \colon \R^\bfd \times \Omega \to [0, \infty)$  satisfy for all $\theta \in \R^\bfd$, $\omega \in \Omega$ that \cfadd{def:p-norm}
\begin{equation}
    \cR(\theta, \omega) = \frac{1}{M} \left[ \sum_{j=1}^M  \abs{\scrN^{\theta, \bfl}_{u,v}(X_j(\omega))-Y_j(\omega) } ^2 \right],
\end{equation}
\begin{equation}
    \bfk( \omega) \in \arg \min\nolimits_{(k,n) \in \{1,2, \ldots, K\} \times \fN,\, \| \Theta_{k,n} (\omega) \|_\infty \leq \beta} \cR(\Theta_{k,n}(\omega), \omega)
\end{equation}
\cfload.
Then
\begin{equation}
 \begin{split}
        & \left( \E \! \left[ \left(   \int_{[a,b]^d} \abs{ \scrN^{\Theta_\bfk, \bfl}_{u,v}(x)-\cE (x) }^2 \, \P_{X_1}(\dx x)  \right)^{ \! \nicefrac{p}{2}} \right] \right)^{\! \nicefrac{1}{p}} \\
        &\leq \frac{3d L(b-a)}{\left( \min\{2^{{\bfL} - 1}, \bfl_1, \ldots, \bfl_{\bfL-1} \} \right)^{ \! \nicefrac{1}{d}}}+
        \frac{2 \left[(v-u) \bfL ( \norm{\bfl}_\infty+1)^\bfL c^{\bfL+1} \max \{ \nicefrac{p}{2}, 1 \}\right]^{\nicefrac{1}{2}}}
            {K^{[(2\bfL)^{-1}(\norm{\bfl}_\infty+1)^{-2}]}} \\
        &+ \frac{ \max \{v-u, 1 \} \left[10\bfL ( \norm{\bfl}_\infty +1)^{\nicefrac{3}{2}} \max \{ \nicefrac{p}{2} , \ln (4 M \beta c)\}\right]^{\nicefrac{1}{2}}}{M^{\nicefrac{1}{4}}} \\
        &\leq \frac{6 d c^2}{\left( \min\{2^{{\bfL} - 1}, \bfl_1, \ldots, \bfl_{\bfL-1} \} \right)^{\nicefrac{1}{d}}}
        +  \frac{2 \left[ \bfL ( \norm{\bfl}_\infty+1)^\bfL c^{\bfL+2} \max \{p, 1 \}\right]^{\nicefrac{1}{2}}}
            {K^{[(2\bfL)^{-1}(\norm{\bfl}_\infty+1)^{-2}]}} \\
         &+ \frac{4c \beta^{\nicefrac{1}{2}} \bfL^{\nicefrac{1}{2}} ( \norm{\bfl}_\infty +1)^{\nicefrac{3}{4}} \max \{p, \ln (e M)\}}{M^{\nicefrac{1}{4}}}.
    \end{split}
    \end{equation}

\end{cor}
\begin{proof} [Proof of \cref{cor:main}]
Throughout this proof let $A \in \N$ be given by
\begin{equation} \label{cor:eq:defA}
    A =  \min\{2^{{\bfL} -1} , \bfl_1, \ldots, \bfl_{\bfL-1} \}.
\end{equation}
Note that the assumption that $\bfl_\bfL=1$ shows that
\begin{equation} \label{cor:eq:bfl1}
    \bfl_1 =  \bfl_\bfL \indicator{\{1 \}}(\bfL) +  \bfl_1 \indicator{[2, \infty)}(\bfL) \geq A \indicator{\{1 \}}(\bfL) +  A \indicator{[2, \infty)}(\bfL) = A \geq A \indicator{(6^d,  \infty)}(A).
\end{equation}
Moreover, observe that
\begin{equation} \label{cor:eq:bfl2}
\begin{split}
        \bfL 
        &\geq \ceil{\log _2(A)}+1 \geq \ceil{\log_2(A)- \log_2(d)-1}+2 = \ceil{\log_2 \left( \nicefrac{A}{2d} \right)}+2 \\
        &\geq 1 + (\ceil{ \log_2 \left(\nicefrac{ A  }{ ( 2d ) } \right)} + 1)  \indicator{  ( 6^d, \infty ) }( A ) 
\end{split}
\end{equation}
\cfload.
Furthermore, note that the facts that $A \in \N$ and $d \geq 1$ imply for all $i \in \{2,3, \ldots, \bfL-1 \}$ that
\begin{equation} \label{cor:eq:bfl3}
    \begin{split}
        \bfl_i &\geq A  \geq   3\tfrac{A+3}{4} \indicator{[9, \infty)}(A) +  6 \indicator{\{7,8\}}(A)   \geq 3 \ceil{\tfrac{A}{4}}  \indicator{[9, \infty)}(A) + 3 \ceil{\tfrac{A}{4}} \indicator{\{7,8\}}(A) \\
        &\geq   3 \ceil{\tfrac{A}{4}} \indicator{[7, \infty)}(A) \geq  3 \ceil{\tfrac{A}{2^i d}} \indicator{(6^d, \infty)}(A).
    \end{split}
\end{equation}
Combining \eqref{cor:eq:bfl1}--\eqref{cor:eq:bfl3}, the assumptions that $\bfl_0=d$, $\bfl_\bfL=1$, and \cref{theo:main} (applied with $p \with \nicefrac{p}{2}$ in the notation of \cref{theo:main}) demonstrates that
\begin{equation}
    \begin{split}
             & \left( \E \! \left[ \left(   \int_{[a,b]^d} \abs{ \scrN^{\Theta_\bfk, \bfl}_{u,v}(x)-\cE (x) }^2 \, \P_{X_1}(\dx x)  \right)^{ \! \nicefrac{p}{2}} \right] \right)^{\! \nicefrac{2}{p}} \\
          &\leq \frac{9d^2 L^2(b-a)^2}{A^{\nicefrac{2}{d}}}+
        \frac{4(v-u) \bfL ( \norm{\bfl}_\infty+1)^\bfL c^{\bfL+1} \max \{ \nicefrac{p}{2}, 1 \}}
            {K^{[\bfL^{-1}(\norm{\bfl}_\infty+1)^{-2}]}} \\
        &+ \frac{10 (  \max \{v - u, 1 \} )^2 \bfL ( \norm{\bfl}_\infty +1)^{\nicefrac{3}{2}} \max \{ \nicefrac{p}{2}, \ln (4M \beta c)\}}{\sqrt{M}}.
    \end{split}
\end{equation}
This, \eqref{cor:eq:defA}, the assumptions that $\beta \geq c \geq 1$, $M \geq 1$, $c \geq |u|+|v| \geq v-u$, and \cref{lem:est:exp} imply that
\begin{equation}
 \begin{split}
        & \left( \E \! \left[ \left(   \int_{[a,b]^d} \abs{ \scrN^{\Theta_\bfk, \bfl}_{u,v}(x)-\cE (x) }^2 \, \P_{X_1}(\dx x)  \right)^{\! \nicefrac{p}{2}} \right] \right)^{\! \nicefrac{1}{p}} \\
        &\leq \frac{3d L(b-a)}{\left( \min\{2^{{\bfL} - 1}, \bfl_1, \ldots, \bfl_{\bfL-1} \} \right)^{\nicefrac{1}{d}}}+
        \frac{2 \left[(v-u) \bfL ( \norm{\bfl}_\infty+1)^\bfL c^{\bfL+1} \max \{ \nicefrac{p}{2}, 1 \}\right]^{\nicefrac{1}{2}}}
            {K^{[(2\bfL)^{-1}(\norm{\bfl}_\infty+1)^{-2}]}} \\
        &+ \frac{ \max \{ v-u, 1 \} \left[10\bfL ( \norm{\bfl}_\infty +1)^{\nicefrac{3}{2}} \max \{ \nicefrac{p}{2} , \ln (4 M \beta c)\}\right]^{\nicefrac{1}{2}}}{M^{\nicefrac{1}{4}}} \\
        &\leq \frac{6 d c^2}{\left( \min\{2^{{\bfL} - 1}, \bfl_1, \ldots, \bfl_{\bfL-1} \} \right)^{\nicefrac{1}{d}}}
        +  \frac{2 \left[ \bfL ( \norm{\bfl}_\infty+1)^\bfL c^{\bfL+2} \max \{p, 1 \}\right]^{\nicefrac{1}{2}}}
            {K^{[(2\bfL)^{-1}(\norm{\bfl}_\infty+1)^{-2}]}} \\
         &+ \frac{4c \beta^{\nicefrac{1}{2}} \bfL^{\nicefrac{1}{2}} ( \norm{\bfl}_\infty +1)^{\nicefrac{3}{4}} \max \{p, \ln (e M)\}}{M^{\nicefrac{1}{4}}}.
    \end{split}
    \end{equation}
    This completes the proof of \cref{cor:main}.
\end{proof}

\cfclear
\begin{cor} \label{cor:simple}
Let $d, \bfd, \bfL, M,K, N \in \N$, $c \in [2,\infty)$, $\bfl = (\bfl_0, \bfl_1, \ldots, \bfl_\bfL) \in \N^{\bfL+1}$, $\fN \subseteq \{0,1,\ldots, N\}$ satisfy $0 \in \fN$, 
$ \bfl_0 = d $, 
$ \bfl_\bfL = 1 $,
and
$ \bfd \geq \sum_{i=1}^{\bfL} \bfl_i( \bfl_{ i - 1 } + 1 ) $,
 let $(\Omega, \cF, \P)$ be a probability space, let $(X_j, Y_j) \colon \Omega \to [0,1]^d \times [0,1]$, $j \in \{1,2, \ldots, M\}$, be i.i.d.\ random variables, let $\cE \colon [0,1]^d \to [0,1]$ satisfy $\P$-a.s.\ that $\cE(X_1) = \E[Y_1|X_1]$, assume for all $x,y \in [0,1]^d$ that $|\cE (x)- \cE (y)| \leq c \norm{x-y}_1$, let $\Theta_{k, n} \colon \Omega \to \R^\bfd$, $k, n \in \N_0$, be random variables, assume that $\Theta_{k, 0}$, $k \in \{1,2, \ldots, K\}$, are i.i.d., assume that $\Theta_{1, 0}$ is continuous uniformly distributed on $[-c,c]^\bfd$, let $\bfk \colon \Omega \to (\N_0)^2$ be a random variable, and let $\cR \colon \R^\bfd \times \Omega \to [0, \infty)$  satisfy for all $\theta \in \R^\bfd$, $\omega \in \Omega$ that \cfadd{def:p-norm}
\begin{equation}
    \cR(\theta, \omega) = \frac{1}{M} \left[ \sum_{j=1}^M \bigl| \scrN^{\theta, \bfl}_{0,1}(X_j(\omega))-Y_j(\omega) \bigr| ^2 \right],
\end{equation}
\begin{equation}
    \bfk( \omega) \in \arg \min\nolimits_{(k, n) \in \{1,2, \ldots, K\} \times \fN,\, \| \Theta_{k,n} (\omega) \|_\infty \leq c} \cR(\Theta_{k, n}(\omega), \omega)
\end{equation}
\cfload. Then
\begin{equation} \label{eq:cor:simple}
\begin{split}
          &\E \! \left[    \int_{[0,1]^d} \bigl| \scrN^{\Theta_\bfk, \bfl}_{0 , 1}(x)-\cE (x) \bigr| \, \P_{X_1}(\dx x) \right] \\
        & \leq \frac{3 d c}{\left( \min\{2^{{\bfL} - 1}, \bfl_1, \ldots, \bfl_{\bfL-1} \} \right)^{\nicefrac{1}{d}}} 
        +  \frac{ \bfL ( \norm{\bfl}_\infty+1)^\bfL c^{\bfL+1} }
            {K^{[(2\bfL)^{-1}(\norm{\bfl}_\infty+1)^{-2}]}} 
         + \frac{4c^{\nicefrac{1}{2}} \bfL ( \norm{\bfl}_\infty +1) \ln (e M)} {(2M)^{\nicefrac{1}{4}}}.
\end{split}
\end{equation}
\end{cor}
\begin{proof} [Proof of \cref{cor:simple}]
Observe that Jensen's inequality implies that
\begin{equation}
     \E \! \left[ \int_{[0,1]^d} \bigl| \scrN^{\Theta_\bfk, \bfl}_{0, 1}(x)-\cE (x) \bigr| \, \P_{X_1}(\dx x) \right] \leq  \E \! \left[\left( \int_{[0,1]^d} \bigl| \scrN^{\Theta_\bfk, \bfl}_{0, 1}(x)-\cE (x) \bigr| ^2 \, \P_{X_1}(\dx x) \right)^{\! \nicefrac{1}{2}} \right].
\end{equation}
Combining this, the facts that $\| \bfl \|_\infty \geq 1$, $c \geq 2$, $\bfL \geq 1$, and \cref{cor:main} (applied with $a \with 0$, $b \with 1$, $u \with 0$, $v \with 1$, $\beta \with c$, $L \with c$, $p \with 1$ in the notation of \cref{cor:main}) establishes \eqref{eq:cor:simple}. The proof of \cref{cor:simple} is thus complete.
\end{proof}

\subsection{Stochastic gradient descent for multidimensional target functions} \label{subsec:multisgd}

\cfclear
\begin{cor} \label{cor:sgd}
Let $d, \bfd, \bfL, M,K,N \in \N$, $p \in (0, \infty)$, $ L \in [0, \infty)$, $a, u \in \R$, $v \in (u, \infty)$, $b \in (a, \infty)$, $c \in [ \max\{1, 2|u|, 2|v|, |a|, |b|, L \}, \infty)$, $\beta \in [c, \infty)$, $(\gamma_n) _{n \in \N} \subseteq  \R$,
$(J_n)_{n \in \N} \subseteq \N$,
$\bfl = (\bfl_0, \bfl_1, \ldots, \bfl_\bfL) \in \N^{\bfL+1}$,
 $\fN \subseteq \{0,1,\ldots, N\}$ satisfy $0 \in \fN$, 
$ \bfl_0 = d $,  $ \bfl_\bfL = 1 $,
and $ \bfd \geq \sum_{i=1}^{\bfL} \bfl_i( \bfl_{ i - 1 } + 1 ) $,
 let $(\Omega, \cF, \P)$ be a probability space, let $(X_j^{k,n}, Y_j^{k,n}) \colon \Omega \to [a,b]^d \times [u,v]$, $j \in \N$, $k,n \in \N _0$, be random variables, assume that $(X_j^{0,0}, Y_j^{0,0})$, $j \in \{1,2, \ldots, M\}$, are i.i.d., 
 let $\cE \colon [a,b]^d \to [u,v]$ satisfy $\P$-a.s.\ that $\cE(X_1^{0,0}) = \E[Y_1^{0,0}|X_1^{0,0}]$, assume for all $x,y \in [a,b]^d$ that $| \cE (x)- \cE (y)| \leq L \norm{x-y}_1$,
 let $\Theta_{k,n} \colon \Omega \to \R^\bfd$, $k,n \in \N_0$, and $\bfk \colon \Omega \to (\N_0)^2$ be random variables, assume that $\Theta_{k,0}$, $k \in \{1,2, \ldots, K\}$, are i.i.d., assume that $\Theta_{1,0}$ is continuous uniformly distributed on $[-c,c]^\bfd$, 
 let $\cR^{k,n}_J \colon \R^\bfd \times \Omega \to [0, \infty)$, $J \in \N$, $k,n \in \N_0$, satisfy for all $J \in \N$, $k,n \in \N_0$, $\theta \in \R^\bfd$, $\omega \in \Omega$ that \cfadd{def:p-norm}
 \begin{equation}
    \cR^{k,n}_{J}(\theta, \omega) = \frac{1}{J} \left[ \sum_{j=1}^{J}  \bigl| \scrN^{\theta, \bfl}_{u,v}(X_j^{k,n} (\omega))-Y_j^{k,n} (\omega) \bigr| ^2 \right],
\end{equation}
let $\cG ^{k,n} \colon \R^\bfd \times \Omega \to \R^\bfd$, $k,n\in \N$, satisfy for all $ k,n \in \N$, $\omega \in \Omega$, 
    $\theta \in \{ \vartheta \in \R^\bfd \colon (\cR^{k,n}_{J_n}(\cdot, \omega) \colon  \R^\bfd \to [0, \infty) \text{ is differentiable at } \vartheta )\}$
that $\cG^{k,n}(\theta, \omega) = (\nabla_\theta \cR^{k,n}_{J_n})(\theta, \omega)$, assume for all $ k,n \in \N$ that $\Theta_{k,n} = \Theta_{k, n-1} - \gamma_n \cG^{k,n}(\Theta_{k, n-1})$, and assume for all $\omega \in \Omega$ that
\begin{equation}
    \bfk( \omega) \in \arg \min\nolimits_{(k,n) \in \{1,2, \ldots, K\} \times \fN,\, \| \Theta_{k,n} (\omega) \|_\infty \leq \beta} \cR^{0,0}_M(\Theta_{k,n}(\omega), \omega)
\end{equation}
\cfload.
Then
\begin{equation}
 \begin{split}
        & \left( \E \! \left[ \left(   \int_{[a,b]^d} \abs{ \scrN^{\Theta_\bfk, \bfl}_{u,v}(x)-\cE (x) }^2 \, \P_{X_1^{0,0}}(\dx x)  \right)^{ \! \nicefrac{p}{2}} \right] \right)^{\! \nicefrac{1}{p}} 
        \leq \frac{6 d c^2}{\left( \min\{2^{{\bfL} - 1}, \bfl_1, \ldots, \bfl_{\bfL-1} \} \right)^{\nicefrac{1}{d}}} \\
       & +  \frac{2 \left[ \bfL ( \norm{\bfl}_\infty+1)^\bfL c^{\bfL+2} \max \{p, 1 \}\right]^{\nicefrac{1}{2}}}
            {K^{[(2\bfL)^{-1}(\norm{\bfl}_\infty+1)^{-2}]}} 
         + \frac{4c \beta^{\nicefrac{1}{2}} \bfL^{\nicefrac{1}{2}} ( \norm{\bfl}_\infty +1)^{\nicefrac{3}{4}} \max \{p, \ln (e M)\}}{M^{\nicefrac{1}{4}}}.
    \end{split}
    \end{equation}

\end{cor}
\begin{proof} [Proof of \cref{cor:sgd}]
This is a direct consequence \cref{cor:main} (applied with $\cR \with \cR_M^{0,0}$, $((X_j, Y_j))_{j \in \{1, \ldots, M \} } \allowbreak \with ((X_j^{0,0}, Y_j^{0,0}))_{j \in \{1, \ldots, M \} }$ in the notation of \cref{cor:main}). The proof of \cref{cor:sgd} is thus complete.
\end{proof}

\cfclear
\begin{cor} \label{cor:sgdsimple}
Let $d, \bfd, \bfL, M, K, N \in \N$, $c \in [2,\infty)$, $\bfl = (\bfl_0, \bfl_1, \ldots, \bfl_\bfL) \in \N^{\bfL+1}$, $(\gamma_n)_{n \in \N} \subseteq \R$, $(J_n)_{n \in \N} \subseteq \N$, $\fN \subseteq \{0,1,\ldots, N\}$ satisfy $0 \in \fN$, 
$ \bfl_0 = d $, 
$ \bfl_\bfL = 1 $,
and
$ \bfd \geq \sum_{i=1}^{\bfL} \bfl_i( \bfl_{ i - 1 } + 1 ) $,
 let $(\Omega, \cF, \P)$ be a probability space, let $(X_j^{k,n}, Y^{k,n}_j) \colon \Omega \to [0,1]^d \times [0,1]$, $j \in \N$, $k,n \in \N_0$, be random variables, assume that $(X_j^{0,0}, Y_j^{0,0})$, $j \in \{1,2, \ldots, M\}$, are i.i.d., let $\cE \colon [0,1]^d \to [0,1]$ satisfy $\P$-a.s.\ that $\cE(X_1^{0,0}) = \E[Y_1^{0,0}|X_1^{0,0}]$, assume for all $x,y \in [0,1]^d$ that $|\cE (x)- \cE (y)| \leq c \norm{x-y}_1$, let $\Theta_{k, n} \colon \Omega \to \R^\bfd$, $k, n \in \N_0$, and $\bfk \colon \Omega \to (\N _0)^2$ be random variables, assume that $\Theta_{k, 0}$, $k \in \{1,2, \ldots, K\}$, are i.i.d., assume that $\Theta_{1, 0}$ is continuous uniformly distributed on $[-c,c]^\bfd$,  let $\cR^{k,n}_J \colon \R^\bfd \times \Omega \to [0, \infty)$, $J \in \N$, $k,n \in \N_0$, satisfy for all $J \in \N$, $k,n \in \N_0$, $\theta \in \R^\bfd$, $\omega \in \Omega$ that \cfadd{def:p-norm}
 \begin{equation}
    \cR^{k,n}_{J}(\theta, \omega) = \frac{1}{J} \left[ \sum_{j=1}^{J}  \bigl| \scrN^{\theta, \bfl}_{0, 1}(X_j^{k,n} (\omega))-Y_j^{k,n} (\omega) \bigr| ^2 \right],
\end{equation}
let $\cG ^{k,n} \colon \R^\bfd \times \Omega \to \R^\bfd$, $k,n\in \N$, satisfy for all $ k,n \in \N$, $\omega \in \Omega$,
    $\theta \in \{ \vartheta \in \R^\bfd \colon (\cR^{k,n}_{J_n}(\cdot, \omega) \colon  \R^\bfd  \to \allowbreak [0, \infty) \text{ is differentiable at } \vartheta )\}$
that $\cG^{k,n}(\theta, \omega) = (\nabla_\theta \cR^{k,n}_{J_n})(\theta, \omega)$, assume for all $ k,n \in \N$ that $\Theta_{k,n} = \Theta_{k, n-1} - \gamma_n \cG^{k,n}(\Theta_{k, n-1})$, and assume for all $\omega \in \Omega$ that
\begin{equation}
    \bfk( \omega) \in \arg \min\nolimits_{(k,n) \in \{1,2, \ldots, K\} \times \fN,\, \| \Theta_{k,n} (\omega) \|_\infty \leq c} \cR^{0,0}_M(\Theta_{k,n}(\omega), \omega)
\end{equation}
\cfload.
Then
\begin{equation} \label{eq:cor:sgdsimple}
    \begin{split}
          &\E \! \left[    \int_{[0,1]^d} \bigl| \scrN^{\Theta_\bfk, \bfl}_{0 , 1}(x)-\cE (x) \bigr| \, \P_{X_1}(\dx x) \right] \\ &
          \leq \frac{3 d c}{\left( \min\{2^{{\bfL} - 1}, \bfl_1, \ldots, \bfl_{\bfL-1} \} \right)^{\nicefrac{1}{d}}} 
        +  \frac{ \bfL ( \norm{\bfl}_\infty+1)^\bfL c^{\bfL+1} }
            {K^{[(2\bfL)^{-1}(\norm{\bfl}_\infty+1)^{-2}]}} 
         + \frac{4c^{\nicefrac{1}{2}} \bfL ( \norm{\bfl}_\infty +1) \ln (e M)} {(2M)^{\nicefrac{1}{4}}}.
   \end{split}
\end{equation}
\end{cor}
\begin{proof}[Proof of \cref{cor:sgdsimple}]
Observe that \cref{cor:simple} (applied with $((X_j, Y_j))_{j \in \{1, \ldots, M \} } \with \allowbreak ((X_j^{0,0}, \allowbreak Y_j^{0,0}))_{j \in \{1, \ldots, M \} }$, $\cR \with \cR_M^{0,0}$ in the notation of \cref{cor:simple}) establishes \eqref{eq:cor:sgdsimple}. The proof of \cref{cor:sgdsimple} is thus complete.
\end{proof}

\subsection*{Acknowledgements}
This work has been funded by the Deutsche Forschungsgemeinschaft (DFG, German Research Foundation) under Germany’s Excellence Strategy EXC 2044-390685587, Mathematics Münster: Dynamics-Geometry-Structure. Helpful suggestions by Pierfrancesco Beneventano, Philippe von Wurstemberger, and Joshua Lee Padgett are gratefully acknowledged.

\end{document}